\documentclass[twoside,11pt]{article}

\usepackage{blindtext}

\usepackage{jmlr2e}

\usepackage{amsmath}
\usepackage{graphics}
\usepackage{graphicx}
\usepackage{txfonts}
\usepackage{subfigure}
\usepackage{algorithm}
\usepackage{algorithmic}
\usepackage{caption}
\usepackage{float}
\usepackage{threeparttable}
\usepackage{url}
\usepackage{multirow}
\usepackage{diagbox}
\usepackage{float}
\usepackage{booktabs}

\usepackage{lastpage}
\jmlrheading{23}{2023}{1-\pageref{LastPage}}{1/21; Revised 5/22}{9/22}{21-0000}{Yiran Dong and Chuanhou Gao}

\ShortHeadings{Causal Structure Learning by Using Intersection of Markov Blankets}{Yiran and Chuanhou}
\firstpageno{1}

\begin{document}

\title{Causal Structure Learning by Using Intersection of Markov Blankets}

\author{\name Yiran Dong \email 12235035@zju.edu.cn \\
       \addr School of Mathematical Sciences\\
	   Zhejiang University\\
	   Hangzhou 310027, China.
       \AND
       \name Chuanhou Gao \email gaochou@zju.edu.cn \\
       \addr School of Mathematical Sciences \\
	  Zhejiang University\\
       Hangzhou 310027, China.}

\editor{My editor}

\maketitle

\begin{abstract}
In this paper, we introduce a novel causal structure learning algorithm called Endogenous and Exogenous Markov Blankets Intersection (EEMBI), which combines the properties of Bayesian networks and Structural Causal Models (SCM). Exogenous variables are special variables that are applied in SCM. We find that exogenous variables have some special characteristics and these characteristics are still useful under the property of the Bayesian network. EEMBI intersects the Markov blankets of exogenous variables and Markov blankets of endogenous variables, i.e. the original variables, to remove the irrelevant connections and find the true causal structure theoretically. Furthermore, we propose an extended version of EEMBI, namely EEMBI-PC, which integrates the last step of the PC algorithm into EEMBI. This modification enhances the algorithm's performance by leveraging the strengths of both approaches. Plenty of experiments are provided to prove that EEMBI and EEMBI-PC have state-of-the-art performance on both discrete and continuous datasets.
\end{abstract}

\section{Introduction}
Causal structure learning, or causal discovery (\cite{glymour2016causal}), aims to find the causal relation of features in datasets and generate a graph based on causal relations. Knowing the causal relation can increase the interpretability of data and contribute to the feature selection (\cite{li2017feature}) or feature intersection process (\cite{autocross}). More and more graphical models are proposed in all kinds of areas. In picture generating, Variational AutoEncoder (VAE) (\cite{kingma2013auto}; \cite{kingma2019introduction}) and diffusion model (\cite{sohl2015deep}; \cite{ho2020denoising}) use probability graphical model (\cite{koller2009probabilistic}) to approximate the joint distribution of data and they can generate new pictures from the joint distribution. In natural language processing, Latent Dirichlet Allocation (LDA) (\cite{blei2003latent}) and stochastic variational inference (\cite{hoffman2013stochastic}) are topic models based on the Bayesian network to gather the topic information from massive document collections. Furthermore, Graph Neural Network (GNN) (\cite{scarselli2008graph}; \cite{zhou2020graph}) uses the graph structure of the dataset as prior knowledge to construct different kinds of neural networks. With the development of graphical model, causal structure learning becomes an important question since a good graph structure can dramatically improve the generative and predictive ability of graphical models.

There are three basic types of causal structure learning algorithms at the beginning (\cite{scutari2018learns}). Constraint-based methods, like the PC algorithm (\cite{spirtes1999algorithm}; \cite{colombo2014order}), build the graph structure based on conditional independence in the dataset. Score-based methods, like the Greedy Equivalence Search (GES) (\cite{chickering2002optimal}), aim to maximize the corresponding score to find the optimal graph structure. And hybrid methods mix the property of both methods. Max-Min Hill Climbing (MMHC) (\cite{tsamardinos2006max}) is a typical hybrid causal discovery algorithm, it uses the constraint-based method to find the skeleton of the graph and uses GES to orient every edge. Recently, more and more different types of causal learning methods are proposed, like constraint functional causal models, permutation-based methods, etc. 

Linear Non-Gaussian Acyclic Model (LiNGAM) (\cite{shimizu2006linear}) is one of the constraint functional causal models. It uses Independent Component Analysis (ICA) (\cite{hyvarinen2000independent}) to identify the exogenous variables among the original features. However, in SCM (\cite{bollen1989structural}), observed variables are called endogenous variables. Exogenous variables, as latent variables, have no parent and contain all the randomicity of endogenous variables. Therefore, it is inappropriate for LiNGAM and its variants to find exogenous variables from endogenous variables. Moreover, LiNGAM and its variants can only find part of the exogenous variables. If we can find all the exogenous variables for every endogenous variable, learning the connections in endogenous variables will be much easier and more accurate by using the properties of exogenous variables.

Inspired by LiNGAM and SCM, we propose EEMBI and EEMBI-PC algorithms which are the mixtures of constraint-based and constraint functional causal models. EEMBI wants to find exogenous variables and uses them to remove the redundant edges. Different from LiNGAM and its variants, EEMBI uses ICA to generate all the exogenous variables directly under the causal sufficiency assumption. In the Markov blanket of node $T$, there is not only children and parents of $T$ but also some extra spouses nodes of $T$. It is easy to identify these spouse nodes by using the exogenous variable of $T$. Therefore, EEMBI has four phases
\begin{itemize}
    \item [(1)] Find the Markov blanket for every endogenous variable;
    \item [(2)] Generate all exogenous variables and match them with endogenous variables;
    \item [(3)] Find the parents for every endogenous variable using the Markov blanket of exogenous variables and build a Bayesian network;
    \item [(4)] Turn the Bayesian network to a Completed Partially Directed Acyclic Graph (CPDAG).
\end{itemize}
If we turn the Bayesian network in (3) as the skeleton the of graph and use the PC algorithm to orient the edges, we have EEMBI-PC algorithms.

The rest of the paper is organized as follows: Section 2 gives some background knowledge including properties of the Bayesian network, SCM, and ICA. In Section 3, we give our improved IAMB algorithms to find Markov blankets for endogenous variables. In Section 4, we give the main body of EEMBI and EEMBI-PC, i.e. algorithms that generate corresponding exogenous variables and find the parent nodes for endogenous variables. We also provide plenty of theorems to guarantee their efficiency. In Section 5, we compare EEMBI and EEMBI-PC with a number of baselines on discrete and continuous datasets. We conclude the paper in Section 6. 

\textbf{Notation:} We use italicized lowercase letters and lowercase Greek, like $x, y, \alpha, \beta$ to represent single nodes or scalars. Column vectors are set as bold lowercase letters. Decorative uppercase letters represent sets, such as $\mathcal X, \mathcal Y$. Italicized uppercase letters, like $A, B, E$, can either be sets or scalars, depending on the context. Matrices are denoted by bold uppercase letters.

\section{Preliminaries}
In this section, we give some basic information about Bayesian networks, causal structure learning, and structural causal models. Then we simply introduce the principle of Independent Component Analysis (ICA). 

\subsection{Foundation of Bayesian Network}
A graph structure, represented as $\mathcal G=(\mathcal X, \mathcal E)$ (\cite{koller2009probabilistic}), consists of the nodes $\mathcal X$ and the edges $\mathcal E$ between nodes.
If a graph structure $\mathcal G$ whose nodes are also the features of the dataset $\mathbf X$, and every edge is directed i.e. it has one source and one target, we call this graph a Bayesian network. For any edges $x \to y \in \mathcal E$, we call $x$ the parent of $y$, and $y$ is the child of $x$. We denote the parents of $x$ as $Pa_x$ and children of $x$ as $Ch_x$. If $x\to z\gets y\in\mathcal E$, then we call $y$ a spouse of $x$. The union of $Ch_x$, $Pa_x$, and all spouses of $x$ is defined as Markov blanket of $x$, denoted as $MB_x$.
Now we give some basic definitions about the Bayesian network.

\begin{definition}
For a subset of nodes $x_0, x_1, ..., x_n$ in the Bayesian network, if there are directed edges between these nodes such that $x_0\to x_1\to x_2\to ... \to x_n$, we say that $x_0, x_1, ..., x_n$ form a path. Moreover, If there is a path such that $x_0\to x_1\to ... \to x_n$ where $x_n=x_0$, we call this path a cycle.
\end{definition}

If there is no cycle in the Bayesian network $\mathcal G$, we define $\mathcal G$ as a Directed Acyclic Graph (DAG). In this and the latter section, we mainly discuss the Bayesian network in condition of DAG. Normally, we use DAG $\mathcal G$ to represent the joint distribution $P$ of the dataset $\mathbf X$, whose nodes $\mathcal X$ are random variables in $P$. According to the edges in $\mathcal E$, we can decompose the joint distribution by using a conditional probability distribution, 
\begin{align}
\label{decomposition}
P=\prod_{x\in \mathcal X} P(x\ |\ Pa_x)
\end{align}

DAG is composed of three basic structures: (1) chain: $x_i\to x_j \to x_k$; (2) fork: $x_i\gets x_j \to x_k$; (3) V-structure: $x_i\to x_j \gets x_k$. Specifically, $x_j$ in V-structure is called a collider.

Now, consider a chain $x_i\to x_j \to x_k$ as a whole DAG, we can decompose the chain according to equation (\ref{decomposition}), 
\begin{align*}
P(x_i, x_j ,x_k)&=P(x_k\ |\ x_j)P(x_j\ |\ x_i)P(x_i)\\
P(x_i, x_j, x_k)/P(x_j)&=P(x_k\ |\ x_j)P(x_j, x_i)/P(x_j)\\\
P(x_i, x_k\ |\ x_j)&=P(x_k\ |\ x_j)P(x_i\ |\ x_j)
\end{align*}
Therefore, we get $x_i$ and $x_k$ are conditional independent given $x_j$, denoted as $x_i\perp\!\!\!\perp x_k\ |\ x_j$. Following the similar procedure, we have $x_i\perp\!\!\!\perp x_k\ |\ x_j$ in fork and $x_i\perp\!\!\!\perp x_k\ |\ \emptyset$ in V-structure where $\emptyset$ is empty set.

Each of these basic structures contains a single conditional independency, large DAG may contain numerous conditional independencies. Thus the Bayesian network is one of the structural representations of joint distribution $P$. 
By learning the true structure of the Bayesian network in $\mathbf X$, we can accurately capture the true conditional independency relationships between random variables, thus approaching the true joint distribution $P$ more closely. This is one of the motivations that we need to learn the structure of the Bayesian network from the dataset.
To establish a direct connection between conditional independency and graphical structure, we introduce d-separation.
\begin{definition}
We say $x_0, x_1, ..., x_n\in \mathcal X$ form a trail if $x_0 \rightleftharpoons x_1\rightleftharpoons...\rightleftharpoons x_n$ where ``$\rightleftharpoons$'' stands for ``$\to$'' or ``$\gets$''. Moreover, for a subset $Z\subset \mathcal X$, if the trail satisfies (1) For any V-structure $x_i\to x_{i+1} \gets x_{i+2}$ in the trail, the collider $x_{i+1} \in Z$; (2) There is no other node in the trail belongs to $Z$. Then we call that the trail is active given $Z$.
\end{definition}

\begin{definition}
Let $A$, $B$, and $C$ be three disjoint subset nodes in DAG $\mathcal G$. If for any node $x\in A$ and any node $y\in B$, there is no active trail given $C$, we call $A$ and $B$ are d-separated given $C$, denoted as $A \overset{d}{\perp} B\ |\ C$.
\end{definition}

Observing the d-separation in three basic structures, we can find that basic structures have exactly the same d-separated relationship and conditional independencies. Expanding from the basic structures to the larger DAG, we have the following theorem.
\begin{theorem}
\label{d-sep}
Let $A$, $B$, and $C$ be three disjoint subset nodes in DAG $\mathcal G$. $A$ and $B$ are d-separated given $C$ if and only if $A$ and $B$ are independent given $C$. i.e. $A\perp\!\!\!\perp B\ |\ C \Longleftrightarrow A\overset{d}{\perp} B\ |\ C$.
\end{theorem}
The detailed proof of \textbf{Theorem \ref{d-sep}} can be found in \cite{pearl1988probabilistic}. \textbf{Theorem \ref{d-sep}} gives us a theoretical basis that we can construct the graph by using conditional independencies. Therefore, the key to learn the structure of DAG is to find all the conditional independencies in random variables of the dataset. However, the chain and fork structures have the same conditional independency with different structures. That is to say, different DAGs may contain exactly the same conditional independencies.

\begin{definition}
\label{I-equivalent}
Let $\mathcal G$ and $\mathcal H$ be two DAGs, we denote the set of all conditional independencies in $\mathcal G$ as $\mathcal I(\mathcal G)$ and the set of all conditional independencies in distribution $P$ as $\mathcal I(P)$. If $\mathcal I(\mathcal G)=\mathcal I(\mathcal H)$, we say $\mathcal G$ and $\mathcal H$ are I-equivalent.
\end{definition}

The aim of structure learning is not only to find one DAG that $\mathcal I(\mathcal G)=\mathcal I(P)$, but to find all the DAGs which are I-equivalent to $\mathcal G$ (\cite{koller2009probabilistic}).

\begin{theorem}
\label{equivalent}
Let $\mathcal G$ and $\mathcal H$ be two DAGs, we define the skeleton of $\mathcal G$ to be a graph structure that replaces the directed edges in $\mathcal G$ as undirected edges. $\mathcal G$ and $\mathcal H$ are I-equivalent if and only if $\mathcal G$ and $\mathcal H$ have the same skeleton and same V-structures.
\end{theorem}

Having the method to find all the I-equivalent DAGs, we still need some methods to represent the I-equivalent class.
\begin{definition}
A Partially Directed Acyclic Graph (PDAG) is an acyclic graph which contains both directed edges and undirected edges. We say a PDAG $\mathcal H$ is Completed Partially Directed Acyclic Graph (CPDAG) of $\mathcal G$, if $\mathcal H$ satisfies\\
(1) $\mathcal H$ and $\mathcal G$ have the same skeletons; \\
(2) $\mathcal H$ contains the directed edge $x_i\to x_k$ if and only if all the I-equivalent graphs of $\mathcal G$ contain this edge.
\end{definition}

\begin{figure*}[htbp]
\centering

\subfigure[]
{
    \begin{minipage}[b]{.3\linewidth}
        \centering
        \includegraphics[scale=0.5]{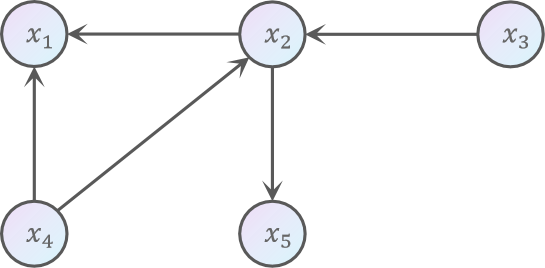}
    \end{minipage}
}
\subfigure[]
{
    \begin{minipage}[b]{.3\linewidth}
        \centering
        \includegraphics[scale=0.5]{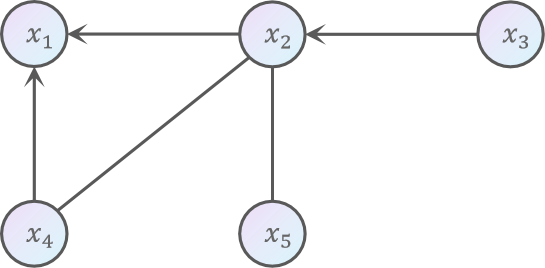}
    \end{minipage}
}
\subfigure[]
{
 	\begin{minipage}[b]{.3\linewidth}
        \centering
        \includegraphics[scale=0.5]{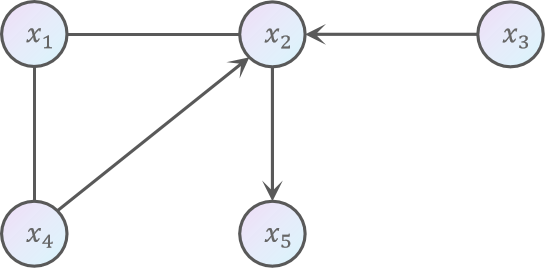}
    \end{minipage}
}
\caption{(a), (b) and (c) are examples of Bayesian networks, PDAG, and CPDAG. (a) is also a DAG since it has no cycle. (b) has both directed and undirected edges, but it is not CPDAG since it changes the V-structure of (a). (c) is the CPDAG of (a), it keeps the V-structure $x_4\to x_2\gets x_3$. (c) also keeps $x_2\to x_5$, since if we reverse the direction of this edge, it will form a new V-structure and the new graphs will not be I-equivalent to (a) anymore.}\label{DAG-PDAG-CPDAG}
\end{figure*}

The differences between DAG, PDAG and CPDAG are shown in Figure \ref{DAG-PDAG-CPDAG}.
Almost all the structure learning algorithms return a CPDAG, and many algorithms learn the structure based on \textbf{Theorem \ref{equivalent}}, i.e. they learn the skeleton of CPDAG, and then based on this skeleton to find all the V-structures. 

In addition to the theoretical basics, many of the causal structure learning algorithms, like PC, MMHC, or Structural Agnostic Modeling (SAM) (\cite{kalainathan2022structural}), require the causal sufficiency assumption of the dataset (\cite{eberhardt2007interventions}). This assumption posits that there is no latent variable that is a common cause of any pair of nodes. By assuming causal sufficiency, these algorithms aim to discover the causal relationships among observed variables without considering hidden confounding factors, which simplifies the tasks and identifies the direct causal relationships.

We take the PC algorithm as an example. The PC algorithm is one of the most commonly used constraint-based methods. We show the process of it in Figure \ref{PC}. PC algorithm starts at a fully connected undirected graph, and for every pair of nodes $x,y\in \mathcal X$, PC algorithm violently traverses all the subset $Z\subset \mathcal X \backslash\{x,y\}$. Then it uses the conditional independence test in statistics, like Randomized Conditional Independence Test (RCIT), Hilbert-Schmidt Independence Test (HSIT), Gaussian conditional independence test (G-test), to test whether $x\perp\!\!\!\perp y\ |\ Z$. If the PC algorithm finds a $Z$ such that $x\perp\!\!\!\perp y\ |\ Z$, then it makes $x,y$ disconnected. This procedure or similar methods are contained in many constraint-based structure learning algorithms. It based on the theorem that two nodes $x,y$ are connected if and only if there is no subset $Z\subset \mathcal X\backslash\{x,y\}$ such that $x,y$ are d-separated given $Z$. After this procedure, we have the skeleton of the graph (the first arrow in Figure \ref{PC}).  For V-structure, PC algorithm traverses all triples $x_i, x_j, x_k$, if for any $Z\subset \mathcal X\backslash \{x_i, x_k\}$ which satisfies $x_i \perp\!\!\!\perp x_k\ |\ Z$ and $x_j\notin Z$, then $x_i, x_j, x_k$ form a V-structure, i.e. $x_i\to x_j \gets x_k$ (the second arrow in Figure \ref{PC}). Finally, we give directions on other undirected edges as much as we can by following Meek's rules (\cite{meek2013causal}).
\begin{definition}
\label{Meek}
 Let $\mathcal G$ be a PDAG.
 \begin{itemize}
     \item [(1)] If $x_i \to x_j\in\mathcal E$ and $x_j - x_k \in \mathcal E$, then orient $x_j-x_k$ into $x_j\to x_k$;
     \item [(2)] If $x_i\to x_j\to x_k \in \mathcal E$ and $x_i - x_k\in\mathcal E$, then orient $x_i-x_k$ into $x_i\gets x_k$;
     \item [(3)] If $x_i-x_j\to x_k\in \mathcal E$, $x_i-x_l\to x_k\in\mathcal E$, $x_i-x_k\in\mathcal E$, and $x_j$, $x_l$ are disconnected, then orient $x_i-x_k$ into $x_i\to x_k$.
 \end{itemize}
 We define these three rules as Meek's rules.
\end{definition}
Meek's rules prevent the DAG which can be represented by the CPDAG from having a cycle or forming a new V-structure (the third arrow in Figure \ref{PC}). After these three steps, the PC algorithm returns a CPDAG.

Different methods in causal discovery exhibit distinct principles. GES transforms the structure learning problem into an optimization task, where it seeks to learn the CPDAG by optimizing the Bayesian Information Criterion (BIC) score. But all the causal discovery algorithms are still rely on fundamental principles \textbf{Theorem \ref{d-sep}} and \textbf{Theorem \ref{equivalent}}.

\begin{figure*}
\centering
\includegraphics[scale=0.45]{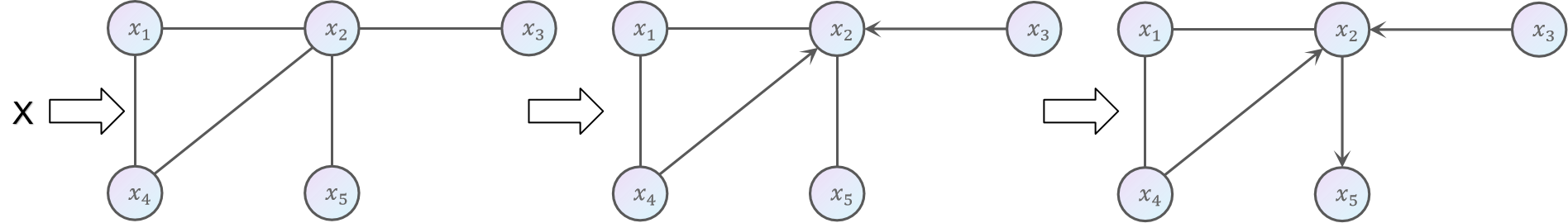}
\caption{The process of the PC algorithm.}\label{PC}
\end{figure*}

\subsection{Structural Causal Model}
Structural Causal Model (SCM) or Structural Equation Model (SEM) (\cite{bollen1989structural}) is one of the important tools in causal structure learning. It provides another way to study the Bayesian network.
\begin{definition}
We define the structural causal model as $\mathcal M=(\mathcal X, \mathcal Y, D_\mathcal X, D_\mathcal Y, \mathcal F, P_\mathcal Y)$ and define
\begin{itemize}
\item [(1)] $\mathcal X$ is the set of endogenous variables and $\mathcal Y$ is the set of exogenous variables.\\
\item [(2)] $D_\mathcal X=\prod_{x\in\mathcal X}D_x$ and $D_\mathcal Y=\prod_{e\in\mathcal Y}D_e$ where $D_x$ and $D_e$ are the codomains of endogenous variable $x$ and exogenous variable $e$. \\
\item [(3)] $\mathcal F=\{f_x, x\in\mathcal X$\} is the set of measurable functions $f_x$ which maps the codomain of $\mathcal X\cup \mathcal Y\backslash\{x\}$ to the codomain of $x\in\mathcal X$.\\
\item [(4)] $P_\mathcal Y=\underset{e\in\mathcal Y}{\prod}P_e$ is the joint distribution function of exogenous variables $\mathcal Y$.
\end{itemize}
\end{definition}

Comparing the definition of Bayesian network and SCM, we find that SCM has two sets of nodes, endogenous and exogenous nodes. Endogenous variables are the observed variables in the dataset, and exogenous variables are hidden variables. Although the nodes in the Bayesian network are the same as endogenous variables in SCM, SCM does not treat variables in the dataset as random variables, it puts the randomicity of features into exogenous variables and assumes the independence of exogenous variables. This is why SCM only has the distribution functions for exogenous variables $P_\mathcal Y=\underset{e\in\mathcal Y}{\prod}P_e$ and does not have distribution functions of endogenous variables. SCM assumes that as long as we know all the randomicity, the endogenous variable can be determined by the randomicity and other endogenous variables. Therefore, instead of putting the edges $\mathcal E$ into $\mathcal M$, SCM defines a set of maps $\mathcal F: D_\mathcal X\times D_\mathcal Y\to D_\mathcal X$. SCM not only puts the structure into $\mathcal F$ but also puts the models and parameters into $\mathcal F$.

\begin{definition}
\label{SCM-parent}
Let $\mathcal M=(\mathcal X, \mathcal Y, D_\mathcal X, D_\mathcal Y, \mathcal F, P_\mathcal Y)$ be a SCM. For any endogenous $x\in\mathcal X$ and variables $z\in\mathcal X\cup\mathcal Y$, let $\Tilde{\mathcal X}=\mathcal X\backslash\{z\}$ and $\Tilde{\mathcal Y}=\mathcal Y\backslash\{z\}$, $z$ is a parent of $x$ if and only if there is no $\Tilde f_x: D_{\Tilde{\mathcal X}\backslash\{x\}}\times D_{\Tilde{\mathcal Y}}\to D_x$ such that
\begin{align*}
    x=f_x\left(\mathcal X\backslash\{x\}, \mathcal Y\right) \Longleftrightarrow x=\tilde f_x(\Tilde{\mathcal X}\backslash\{x\}, \Tilde{\mathcal Y}).
\end{align*}
\end{definition}
Different from the nodes in CPDAG, $x$ and $y$ can not be the parent of each other, since they both have their corresponding exogenous variables as parents and the functional relationship can not be reversed.
Thus $z\in\mathcal X\cup\mathcal Y$ is a parent of $x\in\mathcal X$ if and only if $x$ is not deterministic by any transformation of $f_x$ without knowing $z$ (\cite{olkopf2016foundations}). And exogenous variables are not deterministic, so they do not have any parents.

We give a simple example of SCM. Let us consider the influencing factors of a student's final test grade. Let $s$ be the score of his final test grade, $t$ be the time this student spends on this course, $d$ be the difficulty of the final test. For simplicity, we assume the values of these three variables have no boundary. 
We put randomicity of $s,t,d$ into $e_s$, $e_t$, $e_d$. For example, $e_s$ can be the health condition of the student on the day of the final test, $e_t$ can be the family influence on student, $e_d$ can be the mood of the teacher when he writes the questions. Thus $s,t,d$ are the endogenous variables, $e_s, e_t, e_d$ are the exogenous variables. 
If we know the set of functions $\mathcal F=\{f_d,f_t,f_s\}$,
\begin{align*}
d&=f_d(e_d)=3e_d+1\\
t&=f_t(e_t)=e_t^2\\
s&=f_s(d,t,e_s)=10t-2d+3e_s+5.
\end{align*}
and by \textbf{Definition \ref{SCM-parent}}, we have the structure of SCM in Figure \ref{structure SCM}. If $e_s, e_t, e_d$ follow standard Gaussian distribution i.i.d, we have the SCM in this example
$\mathcal M=\left( \{s,t,d\}, \{e_s, e_t, e_d\}, \mathbb R^3, \mathbb R^3, \mathcal F, \left(\Phi(x)\right)^3 \right)$ where $\mathbb R$ is the set of all real numbers and $\Phi(x)$ is the distribution function of standard Gaussian distribution.

\begin{figure}
\centering
\includegraphics[scale=0.7]{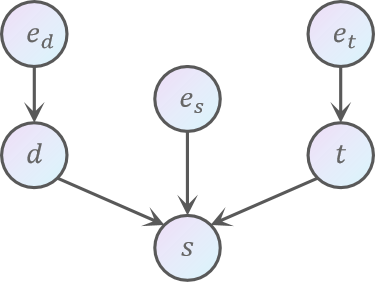}
\caption{The structure of SCM in the final test example.}\label{structure SCM}
\end{figure}

\subsection{Independent Component Analysis}
Independent Component Analysis (ICA) (\cite{hyvarinen2000independent}) aims to find the source messages from given mixed messages. Let $\mathbf x=(x_1,x_2)$ be two messages which are mixed of two independent source messages $\mathbf s=(s_1,s_2)$. Then ICA wants to construct the functions $F_1, F_2$ such that
\begin{align*}
    s_1=F_1(x_1,x_2),\\
    s_2=F_2(x_1,x_2).
\end{align*}
To achieve the independence of $s_1$ and $s_2$, ICA minimizes the mutual information of $s_1$ and $s_2$, $I(s_1, s_2)$. But the actual value of mutual information is hard to compute, especially when the dimension of source messages is more than 2, i.e. $I(s_1, s_2, ..., s_n)$ where $n>2$. 

FastICA is the most commonly used method in ICA problems (\cite{hyvarinen1999fast}), it assumes $F_1$ and $F_2$ to be linear functions. With this assumption, it defines the negentropy of $x$ as $J(x)=H(x_{gauss})-H(x)$ where $x_{gauss}$ is a random variable that follows standard Gaussian distribution, then minimization of mutual information $I(s_1, s_2, ..., s_n)$ is equivalent to maximize $\underset{i}{\sum}J(s_i)$. Furthermore, we can make the loss function easier by making approximation on negentropy
\begin{align}
\label{ICA loss}
J(x)\approx \left[E(G(x))-E(G(x_{gauss}))\right]^2,
\end{align}
where $G$ is chosen as $G(x)=\frac{1}{a}\log\cosh(ax)$ or $G(x)=-\exp(-\frac{1}{2}x^2)$. We use $g$ to denote the derivation of $G$. The whole FastICA algorithm is shown in \textbf{Algorithm 1}.
\begin{algorithm}
\renewcommand{\algorithmicrequire}{\textbf{Input:}}
\renewcommand{\algorithmicensure}{\textbf{Output:}}
\caption{FastICA}
\label{FastICA}
\begin{algorithmic}[1]
\REQUIRE $\mathbf x$, $n$
\STATE Centering $\mathbf x$: $\mathbf x=\mathbf x-E(\mathbf x)$;
\STATE $\mathbf C=E(\mathbf x\mathbf x^\top)$;
\STATE $\mathbf U, \mathbf S, \mathbf V=SVD(\mathbf C)$;
\STATE $\mathbf x=\mathbf U\mathbf S^{-\frac12}\mathbf U\mathbf x$;
\FOR{$i$ from $1$ to $n$}
\STATE Initialize $\mathbf w_i$;
\STATE $\beta=E\left[\mathbf w_i^\top\mathbf xg(\mathbf w_i^\top\mathbf x)\right]$;
\STATE $\mathbf w_i^+=\mathbf w_i-\mu\left\{E\left[\mathbf xg(\mathbf w_i^\top\mathbf x)\right]-\beta\mathbf w_i\right\}\bigg/\left\{E\left[g'(\mathbf w_i^\top \mathbf x)\right]-\beta\right\}$;
\STATE $\mathbf w_i=\mathbf w_i^+\big/ \parallel \mathbf w_i^+ \parallel$;
\IF{$i>0$}
\STATE $\mathbf w_i=\mathbf w_i-\overset{i-1}{\underset{j=1}{\sum}}\mathbf w_i^\top\mathbf w_j\mathbf w_j$;
\STATE $\mathbf w_i=\mathbf w_i\big/\parallel\mathbf w_i\parallel$;
\ENDIF
\ENDFOR
\ENSURE $\mathbf W=[\mathbf w_1, ..., \mathbf w_n]$
\end{algorithmic}
\end{algorithm}

In the input of \textbf{Algorithm 1}, $\mathbf x$ is the vector of observed variables. FastICA needs the sample of $\mathbf x$ to compute the mean and covariance of $\mathbf x$. $n$ is the dimension of source messages $\mathbf s$. Commonly, $n=\dim\mathbf s=\dim\mathbf x$. Line 2 to line 4 is called the whitening step. $SVD(\mathbf C)$ stands for the singular value decomposition of the positive definite matrix $\mathbf C$. Whitening can make the covariance matrix of $\mathbf x$ become an identity matrix, i.e. $E(\mathbf x\mathbf x^\top)=\mathbf I$. It helps reduce the dimension of parameters, reduce the noise, and prevent overfitting. Line 7 to line 9 are using Newton’s method to maximize equation (\ref{ICA loss}) under the constraint $E((\mathbf w_i^\top\mathbf x)^2)=1$ where $\beta$ is the Lagrange multiplier of this constraint. To prevent different $\mathbf w_i$ from converging to the same vector, line 10 to line 12 decorrelate $\mathbf w_i$ from $\mathbf w_1,..., \mathbf w_{i-1}$ based on Gram-Schmidt-like method (\cite{kantorovich2016functional}). FastICA algorithm returns the $n\times n$ coefficient matrix, then $\mathbf s=\mathbf W^\top \mathbf x$.

\section{Improved Incremental Association Markov Blanket}
The Incremental Association Markov Blanket (IAMB) algorithm (\cite{tsamardinos2003algorithms}) was proposed to find the Markov blanket of nodes. In this section, we improve the IAMB algorithm and give theoretical analyses of it.

We give our improved IAMB in \textbf{Algorithm 2}.
\begin{algorithm}
\renewcommand{\algorithmicrequire}{\textbf{Input:}}
\renewcommand{\algorithmicensure}{\textbf{Output:}}
\caption{Improved IAMB}
\label{Improved IAMB}
\begin{algorithmic}[1]
\REQUIRE data set $\mathbf X$, number of nodes $n$, $\alpha$
\STATE Initialize all the Markov Blanket of nodes $\{CMB_1, ..., CMB_n\}$ as empty set $\emptyset$;
\FOR{$i$ from $1$ to $n$}
\STATE $S=\{x_1, x_2, ..., x_n\}$;
\WHILE{$CMB_i$ is changed}
\STATE Find the node $x_j\in S$ $j\neq i$ that maximizes $I(x_i\ ; x_j\ |\ CMB_i)$;
\IF{$I(x_i\ ; x_j\ |\ CMB_i)>\alpha$}
\STATE Add $x_j$ in $CMB_i$;
\STATE Remove $x_j$ from $S$;
\ENDIF
\ENDWHILE
\FOR{node $x_k$ in $CMB_i$}
\IF{$I(x_i\ ; x_k\ |\ CMB_i\backslash \{x_k\})<\alpha$}
\STATE Remove $x_k$ from $CMB_i$;
\ENDIF
\ENDFOR
\ENDFOR
\FOR{$i$ from $1$ to $n$}
\FOR{node $x_j$ in $CMB_i$}
\IF{$x_i\notin CMB_j$}
\STATE Remove $x_j$ from $CMB_i$;
\ENDIF
\ENDFOR
\ENDFOR
\ENSURE $\{CMB_1, CMB_2, ..., CMB_n\}$
\end{algorithmic}
\end{algorithm}

Improved IAMB uses Conditional Mutual Information (CMI) to find all the conditional independencies. In the beginning, we use $CMB_i$ to represent the candidate of the true Markov blanket $MB_i$ and initialize $CMB_i$ as an empty set for $i=1, 2, ..., n$. Line 2 to line 10 is the forward phase, it is mainly based on the total conditioning property of Markov blankets (\cite{pellet2008using}). 
\begin{theorem}[Total Conditioning]
\label{total conditioning}
Let $x$ and $y$ be random variables in data set $\mathbf X$, then $y\in MB_x$ if and only if $x \not\! \perp \!\!\! \perp y\mid \mathcal X\backslash\{x,y\}$.
\end{theorem}
\textbf{Theorem \ref{total conditioning}} establishes that two nodes exhibit a strong dependent relationship if one node is part of the other node's Markov blanket. Therefore if a node has the highest CMI with the target node, then it is likely to be a member of the target node's Markov blanket. In the forward phase, we iterate through every node $x_i$, and find the node $x_j$ that maximizes the value of CMI given current $CMB_i$. If the CMI is big enough, i.e. it exceeds a predefined threshold $\alpha$, we will add the node into the $CMB_i$. We continue the forward phase until $CMB_i$ remains unchanged. Ultimately, $CMB_i$ contains $MB_i$ after the forward phase for $i=1, 2, ..., n$.

For the each node $x_i$ that after the forward phase, it follows the backward phase. the backward phase is based on the theorem below, 
\begin{theorem}
\label{MB-def}
Let $x$ and $y$ be two random variables. If $y\notin MB_x$, then $x\perp \!\!\! \perp y\mid MB_x$ and $x\perp \!\!\! \perp y\mid MB_y$. Moreover, for any subsets of nodes $Z\subset\mathcal X$ that $x,y\notin Z$ and $MB_x\cap Z=\emptyset$, then $x\perp\!\!\!\perp y\mid MB_x\cup Z$.
\end{theorem}
\begin{proof}
The detailed proof can be found in \textbf{Appendix A}.
\end{proof}

If $x_j$ is not in $MB_i$, then $MB_i\subset CMB_i\backslash\{x_j\}$. We can divide $CMB_i\backslash\{x_j\}$ into two parts $CMB_i\backslash\{x_j\}=MB_i\cup Z$. Then by the \textbf{Theorem \ref{MB-def}}, we have $x_i\perp\!\!\!\perp x_j\mid CMB_i\backslash\{x_j\}$. Thus in the backward phase, we pick every node in $CMB_i$ to compute the CMI $I(x_i; x_j\mid CMB_i\backslash\{x_j\})$. If the CMI is smaller than the threshold, we see $x_i$ and $x_j$ are conditional independence and remove $x_j$ from $CMB_i$.

Different from the original IAMB, the improved IAMB has the checking phase after doing the forward phase and backward phase on every node in $\mathcal X$. The checking phase is based on the simple fact that
\begin{align*}
    y\in MB_x \Longleftrightarrow x\in MB_y,
\end{align*}
i.e. the symmetry of Markov blankets. For node $x_i$, we check every node $x_j$ in the $CMB_i$ of $x_i$ that whether $x_i$ also belongs to the $CMB_j$. If not, we exclude $x_j$ from $CMB_i$. 

CMI is a powerful measure to estimate conditional independencies. However, the computation of CMI is much more complex than the computation of mutual information. In the original IAMB algorithm, CMI is computed based on the definition, which requires a large amount of data to obtain accurate estimates. Additionally, we need to do discretization before applying the original IAMB on continuous datasets. However, continuous data may lose its information after discretization, resulting in CMI estimates that may deviate significantly from the true values.

In improved IAMB, we apply $k$th nearest neighbour conditional mutual information (kNN-CMI)  which is proposed in \cite{mesner2020conditional}. Let $X, Y, Z\subset\mathcal X$ be three disjoint sets of random variables. To compute $I(X; Y\mid Z)$, kNN-CMI computes the $l_\infty$ distance $\rho_{k,i}$ of $(\mathbf x_i, \mathbf y_i, \mathbf z_i)$ to the $k$th nearest neighbor (kNN) with hyperparameter $k$ in the dataset where $\mathbf x_i, \mathbf y_i, \mathbf z_i$ is the value of $X, Y, Z$ on the $i$th instance. Then define $N_{XY,i}$ as 
\begin{align*}
    N_{XZ,i}=\bigg|\left\{(\mathbf x_j,\mathbf z_j); \parallel(\mathbf x_i,\mathbf z_i)- (\mathbf x_j, \mathbf z_j)\parallel\leq\rho_{k,i}, 1\leq j\leq N\right\}\bigg|
\end{align*}
where $\parallel\cdot\parallel$ is the $l_\infty$ norm and $N$ is the total number of instances in the dataset. We can also define $N_{YZ,i}$ and $N_{Z,i}$ in a similar way. Then we define $\tilde k_i$ as the number of instances whose distance to $(\mathbf x_i, \mathbf y_i, \mathbf z_i)$ is less or equal to $\rho_{k,i}$. The difference between $\tilde k_i$ and $k_i$ is that we also count the boundary points into $\tilde k_i$. Apparently, $\tilde k_i$ is equal to $k_i$ on continuous data, since the number of boundary points in the continuous condition is zero with probability one. Then we have the approximation of CMI on the $i$th instance:
\begin{align*}
    \xi_i=\psi(\tilde k_i)-\psi(N_{XZ,i})-\psi(N_{YZ,i})+\psi(N_{Z,i}),
\end{align*}
where $\psi$ is derivation of logarithm gamma function $\psi(x)=\frac{d}{dx}\log\Gamma(x)$.

Then kNN-CMI uses $\frac1N\overset{N}{\underset{i=1}{\sum}}\xi_i$ to approximate $I(X, Y\mid Z)$. Author in \cite{mesner2020conditional} also proves that 
\begin{align*}
    \lim_{N\to\infty}E\left(\frac1N\sum^{N}_{i=1}\xi_i\right)=I(X; Y\mid Z).
\end{align*}

kNN-CMI uses kNN to obtain the distance $\rho_{k,i}$, it does not depend on the type of data. Therefore kNN-CMI can directly compute the estimation of CMI for discrete, continuous, or even mixed datasets without losing any information. Moreover, kNN-CMI is one of the most accurate estimations of the true CMI under any size of sample size. Thus based on these properties of kNN-CMI, improved IAMB can return the candidate Markov blankets, which are close to true Markov blankets on any type of dataset.

Combing \textbf{Theorem \ref{total conditioning}}, \textbf{Theorem \ref{MB-def}}, and the analyses above, we can guarantee the effectiveness of improved IAMB. 

\begin{corollary}
\label{effectiveness of improved IAMB}
If the conditional mutual information we compute satisfies $I(X; Y\mid Z)=0$ if and only if $X\perp\!\!\!\perp Y\mid Z$, then the improved IAMB returns the true Markov blankets for any small enough $\alpha$.
\end{corollary}

\section{Endogenous and Exogenous Markov Blankets Intersection}
In this section, we introduce the main algorithms in this paper: endogenous and exogenous Markov blankets intersection (EEMBI) algorithm and EEMBI-PC. They use improved IAMB to obtain the Markov blankets of endogenous variables for the first step. Now we introduce the rest of the steps.

\subsection{Generating and Matching of Exogenous Variables}

The nodes in SCM are composed of endogenous variables $\mathcal X$ and exogenous variables $\mathcal Y$. Every endogenous variable at least needs one endogenous variable to contain its randomicity. We can simplify this definition by letting $|\mathcal X|=|\mathcal Y|$ under the causal sufficiency assumption. We can put the randomicity of $x_i$ into one exogenous variable $e_i$, and the randomicity of $x_i$ has no influence on other endogenous variables because of the causal sufficiency assumption. Then every endogenous variable has only one parent in $\mathcal Y$.

By \textbf{Definition \ref{SCM-parent}}, we know that every endogenous node $x\in\mathcal X$ can be determined by its parents in $\mathcal X$ and $\mathcal Y$. Therefore we can transform $f_i\in\mathcal F$ as 
\begin{align*}
    x_i=f_i(Pa_i, e_i)
\end{align*}
where $Pa_i$ is the parents of $x_i$ in $\mathcal X$. If we want to treat all the DAG as a SCM, we need to find the exogenous variables.

Let $\mathcal X=\{x_1, ..., x_n\}$ and $\mathcal Y=\{e_1, ..., e_n\}$ where $e_i$ is the exogenous variable corresponding to $x_i$. Using the acyclic characteristics of DAG, we can find an endogenous node $x_j$ which has no parent in $\mathcal X$ and is only determined by its exogenous variable $e_j$, i.e. $x_j=f_j(e_j)$. Then we replace $x_j$ with $f_j(e_j)$ in the function of its children. For example, if $x_i=f_i(x_j, x_k, e_i)$ is a child of $x_i$, we replace $x_j$ with $f_j$ to obtain $x_i=f_i(f_j(e_j), x_k, e_i)$. After the replacement for all the children of $x_j$, we can still find another node $x_l$ which is only determined by the exogenous variables according to the acyclic characteristics of DAG, and continue the replacement for its children. At last, we can use exogenous variables to represent all the endogenous variables $x_i$=$g_i(\mathbf e)$ where $\mathbf e=(e_1, e_2, ..., e_n)^\top$ is the vector of all exogenous variables. and $g_i$ is the combination of $f_j$, $j\in Pa_i$, and $f_i$. Then we have
\begin{align}
\label{E2E}
\mathbf x=\mathbf g(\mathbf e),
\end{align}

where $\mathbf x=(x_1, x_2, ..., x_n)^\top$ is the vector of all endogenous and $\mathbf g=(g_1, g_2, ..., g_n)^\top$. Equation (\ref{E2E}) gives us a way to generate $\mathbf e$. If we want to treat $\mathbf e$ as source messages, we need to state that $\mathbf e$ are independent with each other. 

Firstly, If we see $\mathcal X$ and $\mathcal Y$ as a whole graph, it still satisfies the \textbf{Theorem \ref{d-sep}}. Although $x_i\in\mathcal X$ is not a random variable and is determined by its parents, we can see the conditional distribution $P(x_i\mid Pa_i, e_i)$ as a Dirac distribution whose probability is one when $x_i=f_i(Pa_i, e_i)$ is zero otherwise. Thus the conditional independencies in basic structure achieve, and d-separation is equivalent to conditional independencies. 

If all the path between two exogenous variables $e_i, e_j$ contains at least one V-structure, then $e_i$ and $e_j$ will be conditionally independent given the empty set, i.e. they are independent. Let us assume there is a trail $e_i\to x_i\rightleftharpoons y_1\rightleftharpoons y_2, ..., y_m\rightleftharpoons x_j\gets e_j$ where $y_1, y_2, ..., y_m\in \mathcal X$ which has no V-structure. Then we can verify the first and the last $\rightleftharpoons$ as $\to$ and $\gets$, and the trail become $e_i\to x_i\to y_1\rightleftharpoons y_2, ..., y_m\gets x_j\gets e_j$. We can find that there are two reverse paths on this trail if we continue this process. Then there must be a V-structure in the cross of these two paths which is contradictory to our assumption. Thus we can have the conclusion that all the trails between $e_i, e_j$ are not active and all the exogenous variables are independent with each other.

With the independencies of exogenous variables, we can see $\mathbf e$ as source messages and endogenous variables $\mathbf x$ as the mixture of $\mathbf e$ in the ICA problem. Then we can use the method in ICA to recover or generate the exogenous variables. In this paper, we only use FastICA in \textbf{Section 2.3} to generate exogenous variables. 

We still face a matching problem after generating $\mathbf e'$, since the criteria in ICA are independencies and information. ICA does not have interpretability on source messages it generates. But to construct SCM, we need to find the corresponding exogenous variable for every endogenous variable and match them up. To avoid confusion, we denote the vector of the generated exogenous variables after matching as $\mathbf e'=(e'_1, e'_2, ..., e'_n)$ where $e'_i$ is the generated exogenous variable corresponding to $x_i$.





The only message about matching is that $e_i$ is the only exogenous variable that directly connects with $x_i$. According to this property, we can use mutual information to make an assessment of their relationship.

\begin{theorem}
\label{trail MI}
Let graph $\mathcal G$ be a trail $x\leftrightharpoons x_1\leftrightharpoons ... \leftrightharpoons x_n$. Then we have
\begin{align*}
    I(x\ ;x_1)>I(x\ ;x_2)>...>I(x\ ;x_n).
\end{align*}
For more general DAG $\mathcal G$, if this trail is the only trail $x$ can reach $x_1, x_2,..., x_n$ in $\mathcal G$, the formula above also achieves.
\end{theorem}

\begin{proof}
    The detailed proof can be found in \textbf{Appendix A}.
\end{proof}

In the condition of \textbf{Theorem \ref{trail MI}}, it is easy to obtain $I(x_i\ ; e_i)>I(x_j; e_i), i\neq j$. Furthermore, we can extend this idea to more complex situations. Using \textbf{Theorem \ref{trail MI}}, we have an important theorem for the matching process.
\begin{theorem}
\label{matching}
Let $\mathbf x=(x_1, x_2, ..., x_n)^\top$ be the vector of endogenous variables in graph $\mathcal G$, and $\mathbf e=(e_{i_1}, e_{i_2}, ..., e_{i_n})^\top$ be the vector of exogenous variables under some unknown arrangement $(i_1, i_2, ..., i_n)$. Then $e_{i_m}=e_m$ for all $m=1, 2, ..., n$ if and only the arrangement $(i_1, i_2, ..., i_n)$ maximizes $\overset{n}{\underset{m=1}{\sum}} I(x_m\ ; e_{i_m})$ under the constraints $I(x_m\ ;e_{i_m})\neq 0$, i.e.
\begin{align}
\label{maximize MI}
(j_1, j_2, ..., j_n)&=\mathop{\arg\max}_{(i_1, i_2, ..., i_n)}\sum^n_{m=1} I(x_m\ ; e_{i_m})
\\ with\ \ \ \ \ \ \ &I(x_m\ ;e_{i_m})\neq 0, \ \ \ m=1, 2, ..., n \label{constraints}
\end{align}
where $j_i=i$.
\end{theorem}
\begin{proof}
    The detailed proof can be found in \textbf{Appendix A}
\end{proof}

\textbf{Theorem \ref{matching}} presents a method for matching, it turns the matching problem as an optimization problem. Although for a single pair $x_i$ and $e_i$, $I(x_i\ ; e_i)>I(x_j\ ; e_i), i\neq j$ may not achieve in complex situations, the sum of the mutual information can reach the maximization under the right permutation of $\mathbf e$. Combing \textbf{Theorem \ref{matching}} and the generating process, we propose the generating and matching algorithm in \textbf{Algorithm 3}.

\begin{algorithm}
\renewcommand{\algorithmicrequire}{\textbf{Input:}}
\renewcommand{\algorithmicensure}{\textbf{Output:}}
\caption{Generating and Matching}
\label{GM}
\begin{algorithmic}[1]
\REQUIRE data set $\mathbf X$, number of nodes $n$
\STATE Apply ICA method on $\mathbf X$ to obtain $\mathbf E$;
\IF{$\mathbf X$ is discrete}
\FOR{every element $E_{ij}$ in $\mathbf E$}
\STATE $E_{ij}=1$ if $sigmoid(E_{ij})>0.5$, otherwise $E_{ij}=0$;
\ENDFOR
\ENDIF
\STATE Initialize $n\times n$ matrix $\mathbf C$ as zero matrix;
\FOR{i from $1$ to $n$}
\FOR{j from $1$ to $n$}
\STATE $C_{ij}=-I(x_i\ ; e'_j)$;
\IF{$I(x_i\ ; e'_j)=0$}
\STATE $C_{ij}=+\infty$;
\ENDIF
\ENDFOR
\ENDFOR
\STATE Use modified Jonker-Volgenant algorithm on $\mathbf C$ to obtain a permutation $(j_1, j_2, ..., j_n)$;
\STATE Rearrange columns of $\mathbf E$ according to $(j_1, j_2, ..., j_n)$;
\ENSURE $\mathbf E$
\end{algorithmic}
\end{algorithm}

In \textbf{Algorithm 3}, the data set $\mathbf X$ is still the instances set of endogenous vector $\mathbf x$. We only consider discrete and continuous data set. We use FastICA on $\mathbf X$ to obtain the instances of exogenous variables $\mathbf E$ in line 1. The $\mathbf E$ computed by FastICA is continuous. To compute the mutual information for endogenous and exogenous variables, we need to discretize $\mathbf E$ for the condition that $\mathbf X$ is discrete. In line 2 to line 6, we use sigmoid function on every element of $\mathbf E$, and change the value of elements to $0$ or $1$ depends on the threshold $0.5$. Then we have $\mathbf e'$ as binary variables. So far, we complete the generating process. 

In line 13, Modified Jonker-Volgenant (\cite{crouse2016implementing}) aims to find the solution for minimizing the assignment cost:
\begin{align*}
    \min\sum^n_{i=1}\sum^n_{j=1} C_{ij}M_{ij},
\end{align*}
where $\mathbf C$ is the cost matrix, $C_{ij}$ represents the cost if we assign $j$ to $i$, and $M_{ij}=1$ if we assign $j$ to $i$ otherwise $M_{ij}=0$. Modified Jonker-Volgenant algorithm find the $\mathbf M$ to minimize the cost, and outputs the indices $(j_1, j_2, ..., j_n)$ which $M_{i,j_i}=1$. In matching process, we set the element in cost matrix $C_{ij}$ as the minus mutual information of $x_i$ and $e_j$, and we set $C_{ij}$ as infinite if mutual information is zero. Then minimizing the assignment cost is equivalent to maximize the equation (\ref{maximize MI}) under constraints (\ref{constraints}). After rearranging columns of $\mathbf E$, $e'_i$ is the exogenous variable correspond to $x_i$ for $i=1, 2, ..., n$. 

Although we use FastICA, which use linear function to separate observed messages, to generate exogenous variables, we do not have to assume the mixed function to be linear function. Then we have 
\begin{align*}
\mathbf e'=\mathbf P\mathbf W^\top\mathbf x=\mathbf P\mathbf W^\top\mathbf g(\mathbf e),
\end{align*}
where $\mathbf W$ is the output of FastICA in \textbf{Algorithm \ref{GM}}, line 1, and $\mathbf P$ is the permutation matrix which is constructed according to the permutation in line 13. Thus $\mathbf e'$ can be determined by $\mathbf e$ which illustrates that the exogenous vector we generated in \textbf{Algorithm 3} only contains part of the information of true exogenous vector. Apparently, the exogenous vector $\mathbf e'$ we generate is equal to true exogenous vector $\mathbf e'$ if and only if $\mathbf g(\mathbf e)$ is a linear functions. In other cases, $\mathbf e$ and $\mathbf e'$ are very different. However, we only prove \textbf{Theorem \ref{matching}} on $\mathbf e$ and we apply \textbf{Theorem \ref{matching}} on $\mathbf e'$. Therefore, we still need to fill this gap.


\begin{theorem}
\label{MI equivalence}
Let $\mathbf e$ be the true exogenous vector of CPDAG $\mathcal G$, and $\mathbf e'$ be an another exogenous vector that can be determined by $\mathbf e$. i.e. $e'_i$ is the exogenous variable of $x_i$ and there is a $\mathbf h$ such that
\begin{align*}
    \mathbf e'=\mathbf h(\mathbf e).
\end{align*}
Then for element of $\mathbf h$ $h_i$, there is a function $\tilde h_i$ such that $e'_i=h_i(\mathbf e)=\tilde h_i(e_i)$. Moreover, if $\mathbf h$ in equation (\ref{E2E}) is invertible, then $I(x_i\ ; e'_i)=I(x_i\ ; e_i)$ and $I(x_j\ ; e_i)=I(x_j\ ; e'_i)$ achieve for any $j\neq i$.
\end{theorem}
\begin{proof}
The detailed proof can be found in \textbf{Appendix A}.
\end{proof}

The dimension of $\mathbf e$ and $\mathbf x$ are the same according to the assumption we made at the beginning, therefore invertibility of $\mathbf g$ is easy to achieve. If $\mathbf e'$ is generated from \textbf{Algorithm \ref{GM}}, the assumption of \textbf{Theorem \ref{MI equivalence}} achieves. Then $I(x_j\ ;e_i)=I(x_j\ ;e'_i)$ for any $i,j$. Therefore the $(1, 2, ..., n)$ is also the permutation that can maximize 
\begin{align*}
    \sum_{i=1}^n I(x_i\ ;e'_{i_m})
\end{align*}
under the same constraints in \textbf{Theorem \ref{matching}}. \textbf{Theorem \ref{MI equivalence}} guarantees the effectiveness of \textbf{Algorithm \ref{GM}}.
Although we can not generate the true exogenous vector, depending on the similarity between $\mathbf e'$ and $\mathbf e$, we can still learn the true graph structure without knowing the true exogenous vector in next subsection.

\subsection{Markov Blankets Intersection}

\begin{figure*}[htbp]
\centering

\subfigure[]
{
    \begin{minipage}[b]{.3\linewidth}
        \centering
        \includegraphics[scale=0.8]{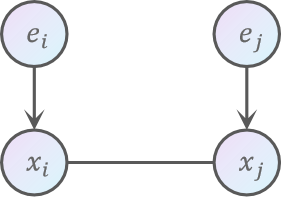}
    \end{minipage}
}
\subfigure[]
{
    \begin{minipage}[b]{.3\linewidth}
        \centering
        \includegraphics[scale=0.8]{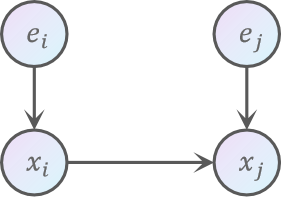}
    \end{minipage}
}
\subfigure[]
{
 	\begin{minipage}[b]{.3\linewidth}
        \centering
        \includegraphics[scale=0.8]{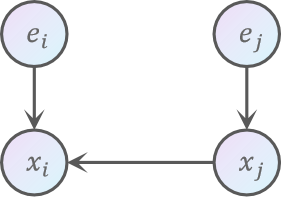}
    \end{minipage}
}
\caption{(a) is CPDAG $\mathcal G$ with exogenous variables. (b) and (c) are augmented graphs $\mathcal G^a_1$ and $\mathcal G^a_2$. Although $\mathcal G_1$ and $\mathcal G_2$ are I-equivalent, their augmented graphs $\mathcal G^a_1$, $\mathcal G^a_2$ are no longer I-equivalent. Since they have different conditional independencies between endogenous and exogenous variables, and only one of them is correct.}\label{paradox}
\end{figure*}

First of all, we give the definition of endogenous Markov blanket and exogenous Markov blanket.
\begin{definition}
Let $\mathcal X$ be the set of endogenous variables and $\mathcal Y$ be the set of exogenous variables. Let $x_i\in\mathcal X, e_i\in\mathcal Y$. We define the endogenous Markov blanket of $x_i$ as the Markov blanket in $\mathcal X$, and the exogenous Markov blanket of $e_i$ as the Markov blanket in $\mathcal X$. We denote the exogenous Markov blanket of $e_i$ as $MB^e_i$.
\end{definition}

The endogenous Markov blankets are the normal Markov blankets in \textbf{Section 2.1}. The exogenous Markov blankets are the Markov blanket of exogenous variables. They are both a subset of $\mathcal X$, since we only care about the structure of endogenous variables.

Although we know the true exogenous variables are the parents of endogenous variables, there are still some differences between $\mathbf e'$ and $\mathbf e$ as we analyze in the last subsection. To find the exogenous Markov blankets, we need to make sure that the exogenous variables we generate also satisfy this relation.
\begin{theorem}
\label{exo parent}
Let $\mathbf e'=(e'_1, e'_2, ..., e'_n)^\top$ is the exogenous vector we generate in \textbf{Algorithm 3}. If $(e'_1, e'_2, ..., e'_n)^\top$ are independent with each other, then $e'_i$ is a parent of $x_i$ and $e'_i$ has no parent.
\end{theorem}
\begin{proof}
    The detailed proof can be found in \textbf{Appendix A}.
\end{proof}
It seems to be contradictory to the analyses in the last subsection that $e'_i$ can be determined only by $e_i$ and $e'_i\gets e_i\to x_i$ forms a fork structure. Actually, $e_i$, as a hidden variable, is a variable we may never know. Without giving $e_i$, this trail $e'_i\gets e_i\to x_i$ is always active. If we omit the $e_i$ the edge between $e'_i$ and $x_i$ can be any direction. But \textbf{Theorem \ref{exo parent}} states that only $e'_i\to x_i$ achieves.

\textbf{Theorem \ref{exo parent}} indicates that $e'_i$ and $e_i$ have the same exogenous Markov blanket $MB^e_i$. We need to combine exogenous variables and endogenous variables into one graph according to \textbf{Theorem \ref{exo parent}}. 
\begin{definition}
Let $\mathcal G=(\mathcal X, \mathcal E)$ be a DAG, we define the augmented graph $\mathcal G^a=(\mathcal X\cup\mathcal Y, \mathcal E\cup\mathcal E')$ where $\mathcal Y=\{e_1, e_2, ..., e_n\}$ is the set of exogenous variables and $\mathcal E'=\bigcup\{e_i\to x_i\}$.
\end{definition}

\begin{figure}
\centering
\includegraphics[scale=0.5]{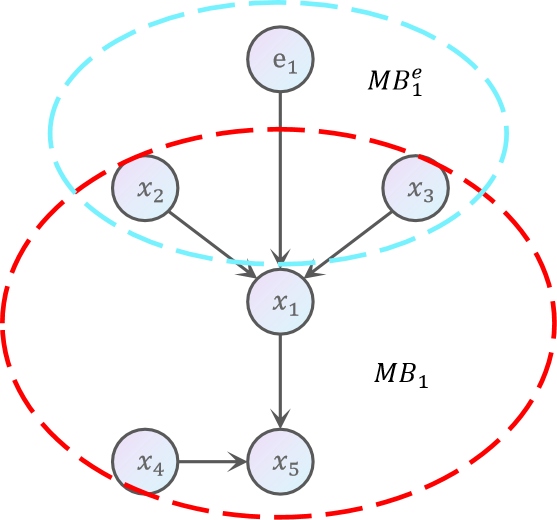}
\caption{An example of Markov blankets intersection on a augmented graph. The intersection of Markov blanket of $x_1$ and $e_1$ is the parents of $x_1$.}\label{MB-intersection}
\end{figure}

After the generating and matching process, we add the exogenous variables and obtain the augmented graph $\mathcal G^a$ of DAG $\mathcal G$.
Different from SCM, augmented graphs do not define the functional relationship or prior distributions, and keep the set of edges $\mathcal E'$, since we only care about the structure of the graph. 

After knowing the relation between endogenous variables and exogenous variables, we can study the exogenous Markov blankets on the augmented graph. Since $e'_i$ is only connected with $x_i$, the relation between $e'_i$ and $x_j$ depends on the relation between $x_i$ and $x_j$. 

The child of $x_i$ is not in $MB^e_i$ since $e_i\to x_i\to x_j$ forms a chain structure. If $x_j$ is a parent of $x_i$, $e'_i\to x_i\gets x_j$ forms a V-structure, then $x_j\in MB^e_i$. However, there are some undirected edges in CPDAG. As we discussed in \textbf{Section 2.1}, if CPDAG $\mathcal G$ contains the undirected edges $x_i-x_j$, then there must be two I-equivalents DAG $\mathcal G_1=(\mathcal X_1, \mathcal E_1)$ and $\mathcal G_2=(\mathcal X_2, \mathcal E_2)$, such that $x_i\to x_j\in\mathcal E_1$ and $x_i\gets x_j\in\mathcal E_2$. But after adding the exogenous variables, $\mathcal G^a_1$, $\mathcal G^a_2$ may have different V-structures. It seems to be contradictory to the I-equivalence since the augmented graph is only a definition, it can not change the conditional independencies in the graph. $\mathcal G_1$ and $\mathcal G_2$ are equivalent under the property of conditional independence, but they are not equivalent in augment graphs.

For example, let $e_1, e_2$ be two exogenous variables that follow the standard Gaussian distributions, $x_1=2e_1$ and $x_2=x_1+e_2^2$ (Figure \ref{paradox}). Using the definition of I-equivalence, $x_1, x_2$ has no conditional independency, therefore $x_1\to x_2$ and $x_1\gets x_2$ are I-equivalent. If given $x_1=c$ as a constant, $x_2=e^2_2+c$ is independent with $e_1$. If given an empty set,  $x_2=x_1+e^2_2=2e_1+e^2_2$ is not independent with $e_1$. Therefore we have the facts that $x_2\perp\!\!\!\perp e_1\mid x_1$ and $x_2\not\!\perp\!\!\!\perp e_1\mid\emptyset$. For $e_2$, we can also infer $x_1\perp\!\!\!\perp e_2\mid\emptyset$ and $x_1\not\!\perp\!\!\!\perp e_2\mid x_2$ in a similar way. Thus $e_1\to x_1\to x_2$ forms a chain and $x_1\to x_2\gets e_2$ form a V-structure which is exactly the $\mathcal G^a_1$ (Figure \ref{paradox} (b)), and $\mathcal G^a_1$ is not I-equivalent with $\mathcal G^a_2$.

The reason for this paradox is that under the SCM, only $x_1\to x_2$ is correct, and $x_1\gets x_2$ can not be achieved under \textbf{Definition \ref{SCM-parent}} since $x_1$ is not determined by $x_2$. Using the conclusion of \textbf{Theorem \ref{exo parent}}, we may add the $x_2$ into the exogenous Markov blanket of $e_1$ because $x_2$ is spouse node of $e_1$ in other I-equivalence DAG, but $x_2\notin MB^e_1$ since it does not satisfy \textbf{Theorem \ref{total conditioning}}, i.e. $x_2\perp\!\!\!\perp e_1\mid x_1$. Thus we have the following theorem:
\begin{theorem}
\label{exo MB}
Let $e_i$ be the exogenous variables with respect to the endogenous variables $x_i$. Then $x_j\in MB^e_i$ if and only if $x_j$ is a parent of $x_i$ under \textbf{Definition \ref{SCM-parent}}.
\end{theorem}
\begin{proof}
    The detailed proof can be found in \textbf{Appendix A}.
\end{proof}
Combing \textbf{Theorem \ref{exo parent}}, we can conclude that the theorem above also achieves for $\mathbf e'$.

These analyses also indicate that we can not find the undirected edges of CPDAG by finding the exogenous Markov blankets, but we can learn a DAG by intersecting endogenous Markov blankets and exogenous Markov blankets (Figure \ref{MB-intersection}). We give the intersection algorithm in \textbf{Algorithm 4}.

\begin{algorithm}
\renewcommand{\algorithmicrequire}{\textbf{Input:}}
\renewcommand{\algorithmicensure}{\textbf{Output:}}
\caption{Intersection Algorithm}
\label{Intersection}
\begin{algorithmic}[1]
\REQUIRE $\mathcal{MB}=\{MB_1, ..., MB_n\}$, $\mathbf X$, $\mathbf E$, $n$, $\beta$
\FOR{$i$ from $1$ to $n$}
\STATE Initialize $MB^e_i=\emptyset$;
\STATE $S=MB_i$;
\WHILE{$MB^e_i$ is changed}
\STATE Find the node $x_j\in S$ that maximizes $I(e_i\ ; x_j\mid MB^e_i)$;
\IF{$I(e_i\ ; x_j\mid MB^e_i)>\beta$}
\STATE Add $x_j$ in $MB^e_i$;
\STATE Remove $x_j$ from $S$;
\ENDIF
\ENDWHILE
\FOR{node $x_k$ in $MB^e_i$}
\IF{$x_k\neq x_i$ and $I(e_i\ ; x_k\mid MB^e_i\backslash \{x_k\})<\beta$}
\STATE Remove $x_k$ from $MB^e_i$;
\ENDIF
\ENDFOR
\FOR{node $x_l$, $l\neq i$, in $MB^e_i$}
\STATE Add the edge $x_l \to x_i$ in $\mathcal G$;
\ENDFOR
\ENDFOR
\ENSURE DAG $\mathcal G$.
\end{algorithmic}
\end{algorithm}

\begin{figure*}[htbp]
\centering

\subfigure[]
{
    \begin{minipage}[b]{.3\linewidth}
        \centering
        \includegraphics[scale=0.55]{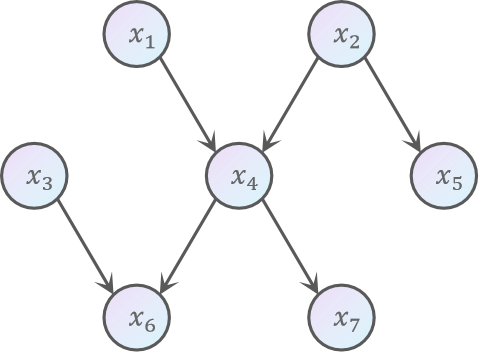}
    \end{minipage}
}
\subfigure[]
{
    \begin{minipage}[b]{.3\linewidth}
        \centering
        \includegraphics[scale=0.55]{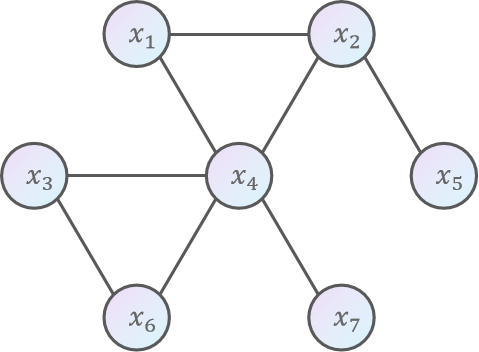}
    \end{minipage}
}
\subfigure[]
{
 	\begin{minipage}[b]{.3\linewidth}
        \centering
        \includegraphics[scale=0.55]{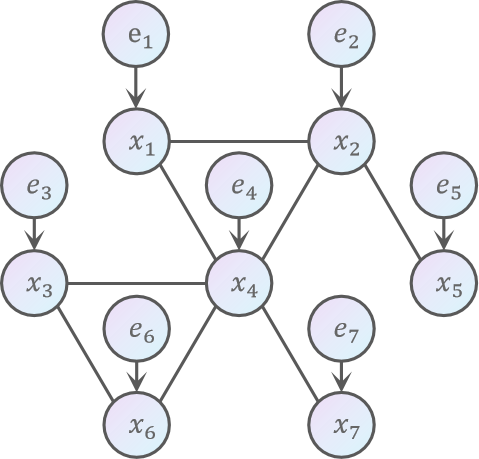}
    \end{minipage}
}

\subfigure[]
{
    \begin{minipage}[b]{.4\linewidth}
        \centering
        \includegraphics[scale=0.55]{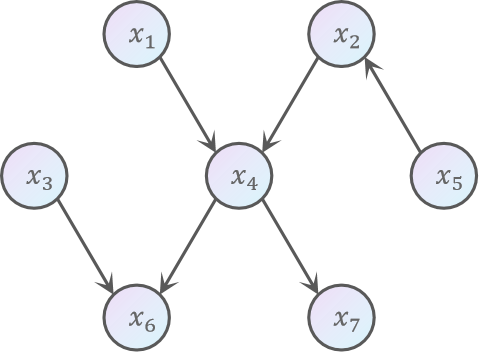}
    \end{minipage}
}
\subfigure[]
{
 	\begin{minipage}[b]{.4\linewidth}
        \centering
        \includegraphics[scale=0.55]{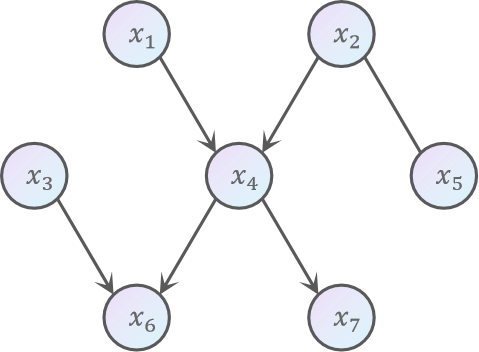}
    \end{minipage}
}

\subfigure[]
{
 	\begin{minipage}[b]{.3\linewidth}
        \centering
        \includegraphics[scale=0.55]{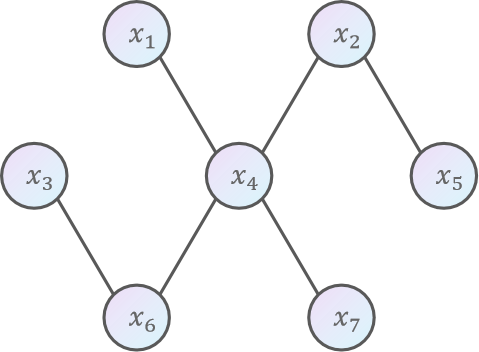}
    \end{minipage}
}
\subfigure[]
{
    \begin{minipage}[b]{.3\linewidth}
        \centering
        \includegraphics[scale=0.55]{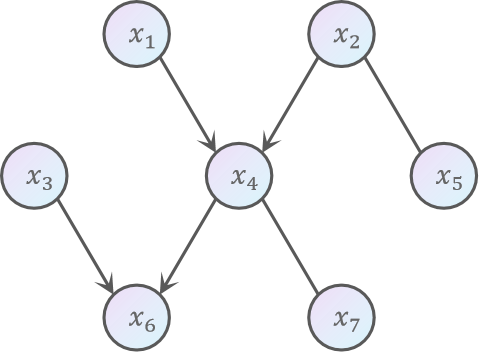}
    \end{minipage}
}
\subfigure[]
{
 	\begin{minipage}[b]{.3\linewidth}
        \centering
        \includegraphics[scale=0.55]{EEMBI-PC6.png}
    \end{minipage}
}
\caption{(a) is original DAG. EEMBI and EEMBI-PC aim to find the CPDAG of (a). (a)$\to$(b)$\to$(c)$\to$(d)$\to$(e) is the process of EEMBI and (a)$\to$(b)$\to$(c)$\to$(f)$\to$(g)$\to$(h) is the process of EEMBI-PC. After learning the Markov blankets, we can connect the node to all the nodes in its Markov blanket to obtain (b). Then generating and matching algorithm assigns the exogenous variable for every $x_i$ to obtain the augmented graph (c). After finding the exogenous Markov blankets for $e_i$, we can exclude the edges between nodes and their spouses $x_1-x_2$, $x_3-x_4$, and return the DAG (d). Finally, we turn the DAG into the CPDAG (e). EEMBI-PC turns the DAG into a skeleton (f) and uses the PC algorithm to find all the V-structures $x_3\to x_6\gets x_4$, $x_1\to x_4\gets x_2$ in (g). Applying Meek's rules, we have the CPDAG of (h).}\label{EEMBI and EEMBI-PC}
\end{figure*}

The Intersection algorithm starts at an augmented graph constructed by the Markov blankets. We assume that we already know the Markov blankets of all nodes $\mathcal{MB}$, and use $\mathcal{MB}$ as one of the inputs. $\mathbf X$, and $\mathbf E$ are the instances of endogenous variables and exogenous variables, they have the same dimension $n$ and the same number of instances. $\beta$ is a threshold similar to $\alpha$ in \textbf{Algorithm \ref{Improved IAMB}}.

Line 1 to line 15 is the modification of the forward phase and backward phase in improved IAMB. By \textbf{Theorem \ref{exo MB}}, we have $MB^e_i\subset MB_i$, therefore we can totally shrink the searching area from all the nodes $S=\{1, 2, ..., n\}$ to the endogenous Markov blanket $MB_i$. In the forward and backward phase, we also use kNN-CMI to estimate the CMI. We can find the parents of $x_i$ efficiently by assessing every conditional independency of $e_i$ and $x_j$ given $MB^e_i$ in the forward phase. 
Although some children or spouses of $x_i$ may be added in $MB^e_i$ accidentally, they can be removed in the backward phase according to \textbf{Theorem \ref{MB-def}}. After the forward and backward phases, we connect the nodes in $MB^e_i$ with $x_i$ in line 17. For node $x_i$, we only connect $x_i$ with its parents and leave its children behind. However, when we consider $x_j\in Ch_i$, $x_i$, as a parent of $x_j$, can be connected under the same process. In this way, the connection of the $x_i$ with its spouses is excluded from the Markov blanket. Therefore, according to \textbf{Theorem \ref{effectiveness of improved IAMB}} and \textbf{Theorem \ref{exo parent}}, \textbf{Algorithm \ref{Intersection}} returns the DAG whose connections follow \textbf{Definition \ref{SCM-parent}}.

Combing all the algorithms above, we can give the main algorithms in this paper in \textbf{Algorithm \ref{EEMBI}} and \textbf{Algorithm \ref{EEMBI-PC}}.

\begin{algorithm}
\renewcommand{\algorithmicrequire}{\textbf{Input:}}
\renewcommand{\algorithmicensure}{\textbf{Output:}}
\caption{EEMBI}
\label{EEMBI}
\begin{algorithmic}[1]
\REQUIRE $\mathbf X$, $\alpha$, $\beta$
\STATE Use \textbf{improved IAMB} on $\mathbf X$ based on the threshold $\alpha$ to obtain Markov blankets set $\mathcal{MB}$;
\STATE Compute the instances of exogenous variables $\mathbf E$ by using \textbf{Generating and Matching Algorithm} and construct the augmented graph;
\STATE Build the DAG $\mathcal G$ using \textbf{Intersection Algorithm};
\STATE Turn the DAG into the CPDAG by Meek's rules;
\ENSURE CPDAG $\mathcal G$.
\end{algorithmic}
\end{algorithm}

\begin{algorithm}
\renewcommand{\algorithmicrequire}{\textbf{Input:}}
\renewcommand{\algorithmicensure}{\textbf{Output:}}
\caption{EEMBI-PC}
\label{EEMBI-PC}
\begin{algorithmic}[1]
\REQUIRE $\mathbf X$, $\alpha$, $\beta$
\STATE Use \textbf{improved IAMB} on $\mathbf X$ based on the threshold $\alpha$ to obtain Markov blankets set $\mathcal{MB}$;
\STATE Compute the instances of exogenous variables $\mathbf E$ by using \textbf{Generating and Matching Algorithm} and construct the augmented graph;
\STATE Learn the DAG from \textbf{Intersection Algorithm} and turn all the directed edges into undirected edges to obtain the skeleton of graph;
\STATE Find all the V-structure by PC algorithm based on the skeleton of graph;
\STATE Orient the rest of the edges using Meek's rules;
\ENSURE CPDAG $\mathcal G$.
\end{algorithmic}
\end{algorithm}

In the last step of the EEMBI, we keep the skeleton and V-structure of the DAG and turn other edges into undirected edges. Then we apply Meek's rules to it. In this way, we successively turn a DAG into a CPDAG.

EEMBI-PC follows the standard steps of constraint-based methods: Learn the skeleton; Find the V-structures; Orient the rest of the edges. The difference of EEMBI-PC starts at line 3. EEMBI-PC uses \textbf{Algorithm \ref{Intersection}} to learn the skeleton and uses the PC algorithm to learn V-structures. Having the skeleton of the graph as prior knowledge, the PC algorithm does not need to traverse all the subsets of $\mathcal X$ to find V-structure. For any possible V-structure $x_i-x_j-x_k$, PC algorithm only traverses all the subsets of $Pa_i\cup Ch_i\cup Pa_k\cup Ch_k$, which is more efficient than original PC algorithm. The visualization of the EEMBI and EEMBI-PC is shown in Figure \ref{EEMBI and EEMBI-PC}. It is easy to conclude that 

\begin{corollary}
EEMBI and EEMBI-PC return the CPDAG of the data set $\mathbf X$.
\end{corollary}

Finally, we discuss the complexity of EEMBI and EEMBI-PC.

In the forward and backward phases of improved IAMB, we need to compute the CMI for $n^2$ times for one node $x_i$ in the worst case. Thus finishing the first two phases for all nodes needs $O(n^3)$ operations. And the checking phase only needs $O(n^2)$ operations, and the computational complexity of improved IAMB is $O(n^3)$. In \textbf{Algorithm \ref{GM}}, it needs $O(n^2)$ time to compute the cost matrix $\mathbf C$ and $O(n^3)$ to apply the Jonker-Volgenant algorithm to solve equation (\ref{maximize MI}). Thus \textbf{Algorithm \ref{GM}} costs $O(n^3)$ operations. In the intersection algorithm, since we shrink the area to Markov blanket, finding the parents and children for $x_i$ only needs $O(|MB|^2)$ operations where $|MB|=max_i(|MB_i|)$, and complexity of the \textbf{Algorithm \ref{Intersection}} is $O(n\times|MB|^2)$. Therefore the computational complexity of EEMBI is $O(n^3)+O(n^3)+O(n\times|MB|^2)=O(n^3)$. 

For EEMBI-PC, it has the same complexity as EEMBI in the first three steps. The PC algorithm needs to test every subset of $Pa_i\cup Ch_i\cup Pa_k\cup Ch_k$ for every possible V-structure $x_i-x_j-x_k$, finding V-structure process needs $O(c\times 2^{2|Pa\cup Ch|})$ time where $c$ is the number of $x_i-x_j-x_k$ in skeleton, and $|Pa\cup Ch|=max_i(|Pa_i\cup Ch_i|)$. In conclusion, the computation complexity of the whole EEMBI-PC algorithm is $O(n^3+c\times 2^{2|Pa\cup Ch|})$.

\section{Experiment}

In this section, we provide several experiments to state the effectiveness of our proposed algorithms. Firstly, we introduce the setup of experiments. Then we evaluate the EEMBI and EEMBI-PC on discrete and continuous datasets. Finally, we study the influence of hyperparameters and give the results of ablation studies. We do all the experiments on CPU i7-12700H with 24G RAM. The code of the proposed algorithms can be found in \href{https://github.com/ronedong/EEMBI}{https://github.com/ronedong/EEMBI}

\begin{table*}[htbp]
\centering
\fontsize{8}{12}\selectfont
\caption{Basic information about datasets}\label{datasets}
\setlength{\tabcolsep}{3mm}{
\begin{tabular}{|c|c|c|c|c|c|c|c|}
\hline
\multirow{2}{*}{Data set} & Sample   & Number of  & \multirow{2}{*}{$\beta$} & \multirow{2}{*}{Data set} & Sample   & Number of  & \multirow{2}{*}{$\beta$} \cr  
                          & Size   & Nodes      &                           &                         & Size   & Nodes      &
\cr\hline 
\hline
ALARM              & 3000    & 37     & 0.01  & BARLEY               & 3000    & 48       & 0.01   \cr\hline
CHILD              & 3000    & 20     & 0.01  & INSURANCE            & 3000    & 27       & 0.01   \cr\hline
MILDEW             & 3000    & 35     & 0.01  & HailFinder           & 3000    & 56       & 0.01   \cr\hline
SACHS              & 7466    & 11     & 0.05  & DREAM3-Ecoli 1,2     & 483     & 50       & 0.05   \cr\hline
DREAM3-Yeast 1,2,3 & 483     & 50     & 0.05  & Education 1,2,3,4,5  & 1000    & 50       & 0.05   \cr\hline
\end{tabular}}
\end{table*}

\subsection{Experimental Setup}
We evaluate the performance of EEMBI and EEMBI-PC on six discrete datasets: ALARM, BARLEY, CHILD, INSURANCE, MILDEW, and HailFinder (\cite{bnlearn}); Additionally, we conduct experiments on eleven continuous datasets: SACHS (\cite{sachs2005causal}), five dream3 datasets (Dream3-Ecoli 1, Dream3-Ecoli2, Dream3-Yeats 1, Dream3-Yeast 2, Dream3-Yeast 3) (\cite{dream3}), as well as five education datasets (Education Net 1,2,3,4,5). An overview of the basic information for these seventeen datasets is provided in Table \ref{datasets}. We compare EEMBI and EEMBI-PC against seven baselines on the discrete datasets: 
\begin{itemize}
    \item PC, Fast Causal Inference (FCI) (\cite{10.5555/2074158.2074215}), Grow-Shrink (GS) (\cite{margaritis2003learning}), and Constraint-based causal Discovery from NOnstationary Data (CDNOD) (\cite{huang2020causal}): they are all constraint-based models. PC algorithm uses G-test as the conditional independent test score;
    \item Greedy Interventional Equivalence Search (GIES) (\cite{hauser2012characterization}): it is a score-based methods, and is the modification of GES;
    \item MMHC (\cite{tsamardinos2006max}): mixture methods mentioned in Section 1;
    \item Greedy Relaxation of the Sparsest Permutation (GRaSP) (\cite{lam2022greedy}): a permutation-based method. 
\end{itemize} 
In addition to the seven baseline methods used on discrete datasets, we include three additional algorithms that are specifically designed for continuous data as baselines for the proposed methods:
\begin{itemize}
    \item Direct Linear Non-Gaussian Acyclic Model (DirectLiNGAM) (\cite{shimizu2011directlingam}) and Causal Additive models (CAM) (\cite{buhlmann2014cam}): they are constraint functional causal models, and DirectLiNGAM is the improvement of LiNGAM;
    \item Non-combinatorial Optimization via Trace Exponential and Augmented lagRangian for Structure learning (NOTEARS) (\cite{zheng2018dags}): NOTEARS is a score-based method, and it is also one of the state-of-the-art causal structure learning algorithms. 
\end{itemize} 
PC, GIES, and GS are achieved by the Python package \texttt{causal discovery box} (\cite{cdt}). FCI, GRaSP, and CDNOD are achieved by the package \texttt{causal-learn} (\cite{causal-learn}). DirectLiNGAM is achieved with the package \texttt{LiNGAM} (\cite{LiNGAM}). We implement MMHC and NOTEARS by using the code provided in their original papers. 

We use the adjacency matrix $\mathbf A$ to represent the CPDAG of dataset $\mathbf X$. The adjacency matrix is constructed as follows:
\begin{itemize}
\item If $x_i\to x_j\in\mathcal E$, then $A_{ij}=1$ and $A_{ji}=0$;
\item If $x_i-x_j\in\mathcal E$, then $A_{ij}=A_{ji}=1$;
\item If $x_i$ and $x_j$ are not connected, $A_{ij}=A_{ji}=0$.
\end{itemize}

All these causal structure learning algorithms return the adjacency matrices of the dataset. We use Structural Hamming Distance (SHD) and Area Under the Precision Recall curve (AUPR) to measure the difference between predicted and true adjacency matrices. They are widely used metrics in causal structure learning. SHD counts the number of different edges between predicted and true adjacency matrices, 
\begin{align*}
    SHD(\mathbf A, \mathbf B)=\sum_i\sum_j \left| A_{ij}-B_{ij}\right|,
\end{align*}
where $\mathbf A$ and $\mathbf B$ are two adjacency matrices. AUPR computes the area under the curve which is constructed by precision: $\frac{TP}{TP+FP}$ and recall: $\frac{TP}{TP+FN}$ with different causation thresholds where $TP, FP, FN$ are short for True Positive, False Positive, and False Negative. All the causal structure learning algorithms aim to learn adjacency matrices that have lower SHDs and higher AUPRs with true adjacency matrix. 

\textbf{Data processing:} For discrete datasets, we encode the categorical values to integer values based on their orders. For continuous datasets, we apply min-max normalization and turn all the values of features in [0,1]. Since we only compare CPDAG in all experiments, we turn the true adjacency matrices of DAG into CPDAG using the function \texttt{dag2cpdag} in pcalg package. 

\textbf{Hyperparameters:} EEMBI and EEMBI-PC have only two hyperparameters: $\alpha$ in \textbf{Algorithm \ref{Improved IAMB}} and $\beta$ in \textbf{Algorithm \ref{Intersection}}. We fix the $\alpha=0.01$, and we set $\beta=0.01$ on discrete datasets and $\beta=0.05$ on continuous datasets. We use all the instances of SACHS and dream3 datasets. However, restricted by the computational complexity and memory of CPU, we only sample part of instances to do the causal structure learning in all discrete datasets and Education datasets. The numbers of instances we sample are shown in Sample Size of Table \ref{datasets}. 

For every structure learning algorithm, we sample from each dataset three times and feed the instances to algorithms to obtain an adjacency matrix for every sampling. After computing the SHD and AUPR metrics for every sampling, we combine all the results and compute the mean and standard deviation for every algorithm on every dataset. We show the SHD results in Table \ref{discrete SHD}, Table \ref{Education SHD} and Table \ref{Dream3 SHD} in the form of $mean\ (standard\  deviation)$. However, we use all the instances in SAHCS and Dream3 datasets, there is no randomicity in these experiments. Therefore, we only run methods once on SACHS and Dream3 and show the results without standard deviations. We highlight the lowest SHD result in each dataset for emphasis. The AUPR results are shown in bar graphs in Figure \ref{discrete AUPR}, Figure \ref{Education AUPR}, and Figure \ref{Dream3 AUPR}. The detailed mean and standard deviation values of AUPR can be found in \textbf{Appendix B}.

\subsection{Discrete datasets}

\begin{table*}[htpb]
\centering
\fontsize{10}{22}\selectfont
\caption{SHD on discrete datasets}\label{discrete SHD}
\setlength{\tabcolsep}{2mm}{
\begin{tabular}{ccccccc}
\bottomrule[1.5pt]
\hline
\multirow{2}{*}{Algorithms} & \multirow{2}{*}{ALARM} & \multirow{2}{*}{BARLEY} & \multirow{2}{*}{CHILD} & \multirow{2}{*}{INSURANCE} & \multirow{2}{*}{MILDEW} & \multirow{2}{*}{HailFinder} \cr
 & \cr
\bottomrule
PC      & 43.7(1.25) & 105.0(0.82) & 34.7(2.87) & 85.0(2.94) & 46.7(2.62) & 123.7(2.49)  \cr
FCI     & 67.0(6.98) & 182.3(3.30) & 46.0(2.16) & 103.7(3.68) & 87.7(3.30)  & 205.3(15.33) \cr
GIES    & 60.0(7.79) & 168.0(11.22) & 44.3(3.30) & 107.7(3.30) & 96.3(3.86) & 140.3(6.02)  \cr
MMHC    & 45.7(2.36) & 119.3(4.50) & 40.3(1.25) & 83.7(3.68) & 47.3(2.87) & 115.3(3.40) \cr
GS      & 57.7(1.70) & 151.3(5.73) & 42.7(4.92) & 98.0(6.48) & 48.7(1.25) & 146.3(6.60) \cr
GRaSP   & 55.0(4.97) & 145.7(11.84) & 41.3(1.70) & 85.0(9.20) & 83.0(14.24) & 133.7(13.10) \cr
CDNOD   & 56.0(1.41) & 157.3(3.40) & 45.0(0.82) & 97.0(3.56) & 55.0(2.16) & 161.7(2.49) \cr
EEMBI   & 47.0(1.63) & 118.3(3.77) & 29.3(0.47) & 68.7(1.89) & 49.7(2.49) & \textbf{92.3(6.65)} \cr
EEMBI-PC& \textbf{38.7(1.25)} & \textbf{95.7(4.99)} & \textbf{29.0(0.82)} & \textbf{58.0(2.94)} & \textbf{46.3(6.65)} & 93.0(0.82) \cr
 \bottomrule[1.5pt]
\end{tabular}}
\end{table*}

\begin{figure*}[htpb]
\centering
\includegraphics[scale=0.25]{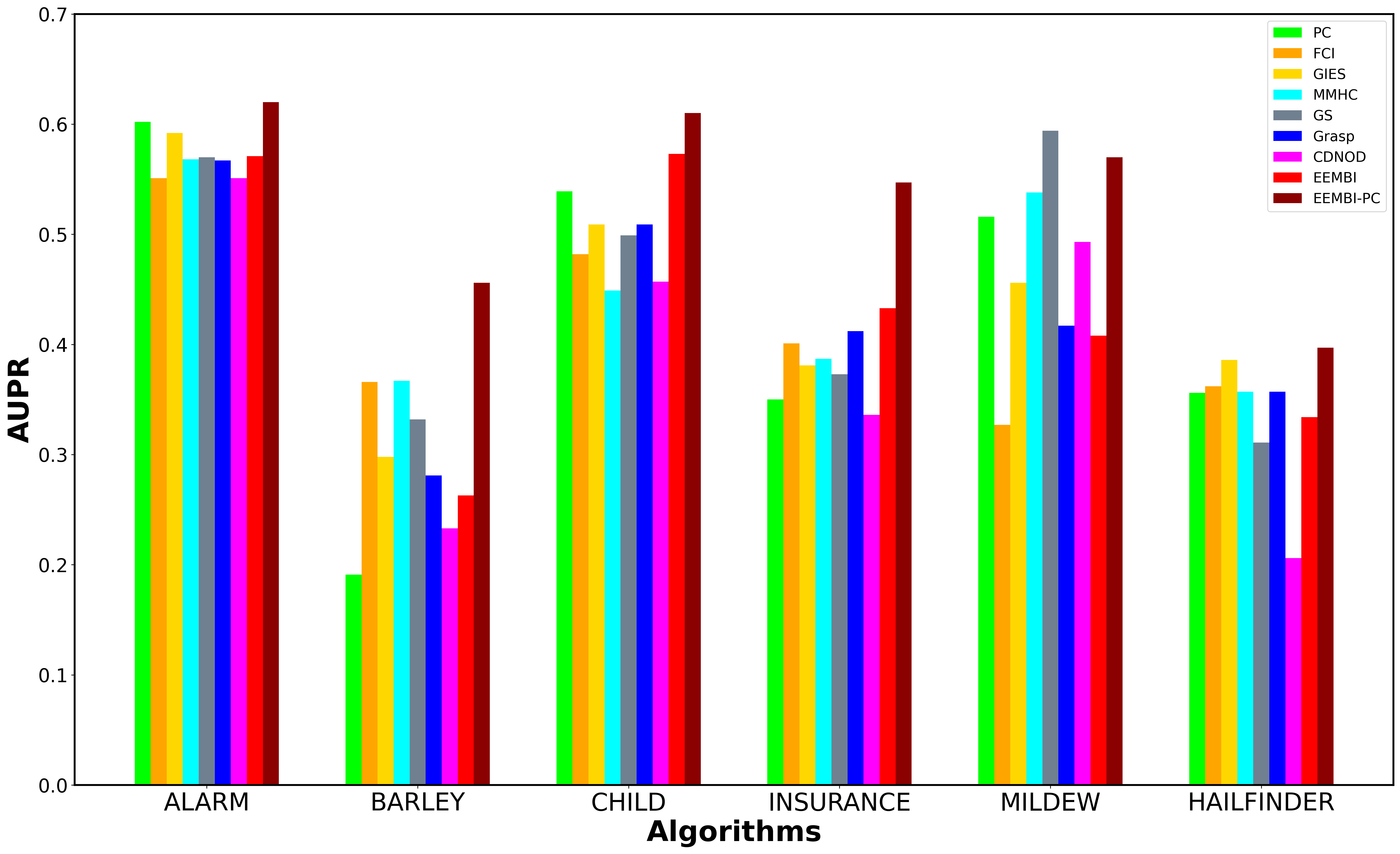}
\caption{AUPR results of causal structure learning methods on discrete datasets}\label{discrete AUPR}
\end{figure*}

On outcomes of discrete datasets in Table \ref{discrete SHD} and Figure \ref{discrete AUPR}, except for the MILDEW and HailFinder datasets, EEMBI-PC has the best performance on the other four datasets, i.e. it has the lowest SHD and highest AUPR. Although EEMBI has the lowest SHD on HailFinder and GS has the highest AUPR on MILDEW, EEMBI-PC shows very close outcomes to them and has the lowest SHD on MILDEW and highest AUPR on HailFinder. Therefore, EEMBI-PC has the best performance of all causal structure learning algorithms on discrete datasets. EEMBI also has SHDs which are close to EEMBI-PC on ALARM, CHILD, MILDEW, and HailFinder datasets. For instance, the SHD of EEMBI (29.3) is only 0.3 higher than the SHD of EEMBI-PC (29.0). But EEMBI has ordinary results on AUPR, it exceeds the baseline algorithms only on CHILD and INSURANCE datasets, and has much lower AUPRs than most of the baselines on BARLEY and MILDEW. Furthermore, the proposed methods outperform the baselines dramatically on some datasets. For example, on HailFinder EEMBI and EEMBI-PC are the only two algorithms that have SHDs lower than 100, 92.3 and 93.0. The best baseline, MMHC, only has 115.3 SHD, and the worst baseline, FCI, has 205.3 SHD which is more than twice of SHD of EEMBI and EEMBI-PC.

In addition to these numerical results, we also present selected parts of the CPDAG structure learned from different methods in Figure \ref{child} and Figure \ref{ins}. We pick eight features from the nodes and include all the edges connecting these eight nodes from the original CPDAG. The direction of the edges remains unchanged. Then these eight nodes and edges form a subgraph. We label the name of the features on the nodes. 

\begin{figure*}[htbp]
\centering

\subfigure[]
{
    \begin{minipage}[b]{.48\linewidth}
        \centering
        \includegraphics[scale=0.35]{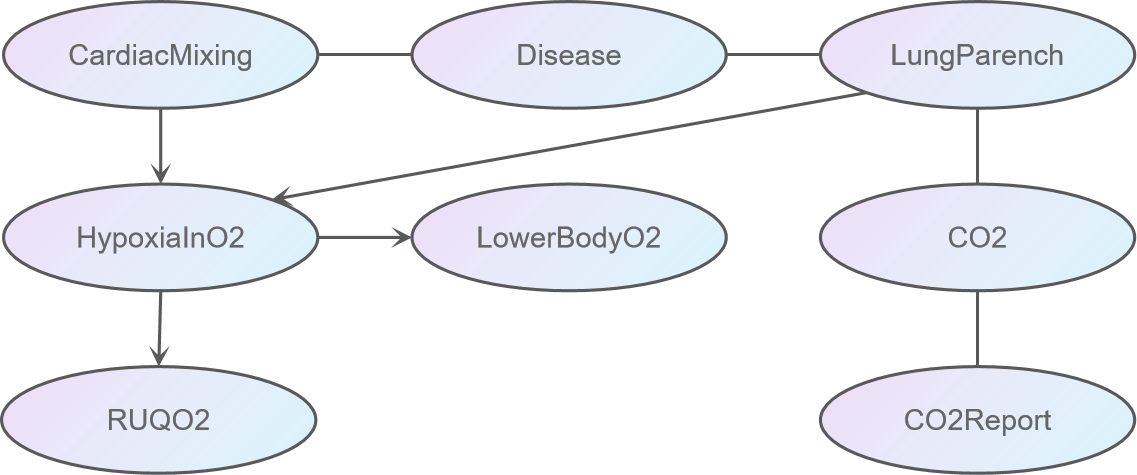}
    \end{minipage}
}
\subfigure[]
{
    \begin{minipage}[b]{.48\linewidth}
        \centering
        \includegraphics[scale=0.35]{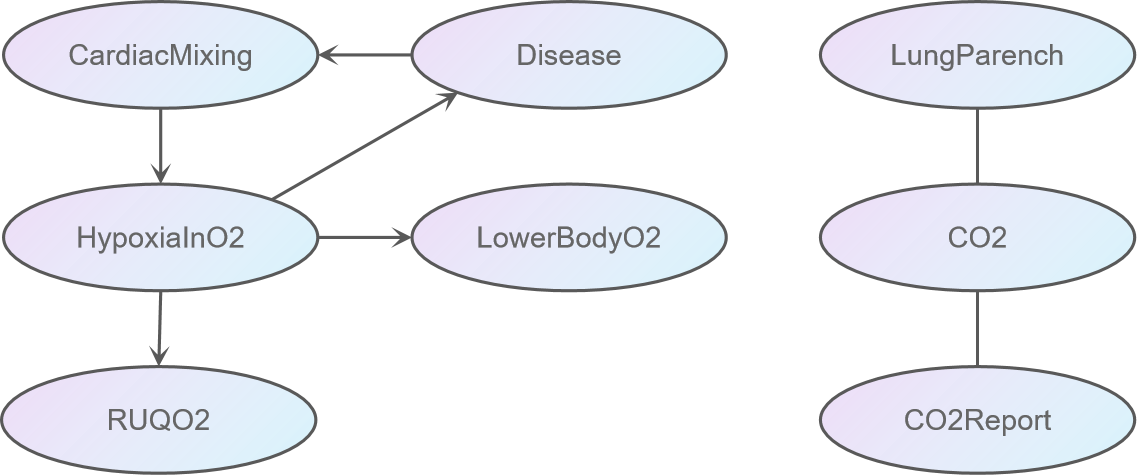}
    \end{minipage}
}

\subfigure[]
{
    \begin{minipage}[b]{.48\linewidth}
        \centering
        \includegraphics[scale=0.35]{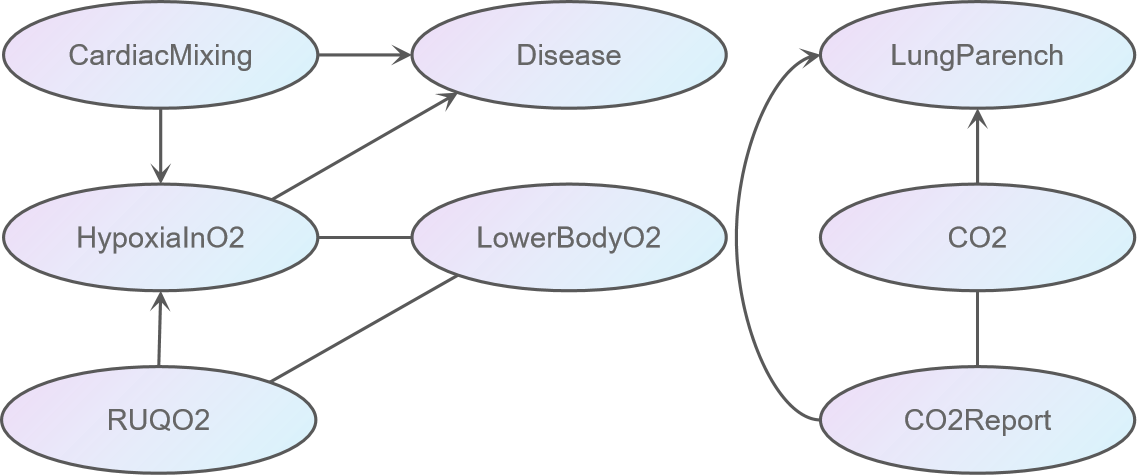}
    \end{minipage}
}
\subfigure[]
{
    \begin{minipage}[b]{.48\linewidth}
        \centering
        \includegraphics[scale=0.35]{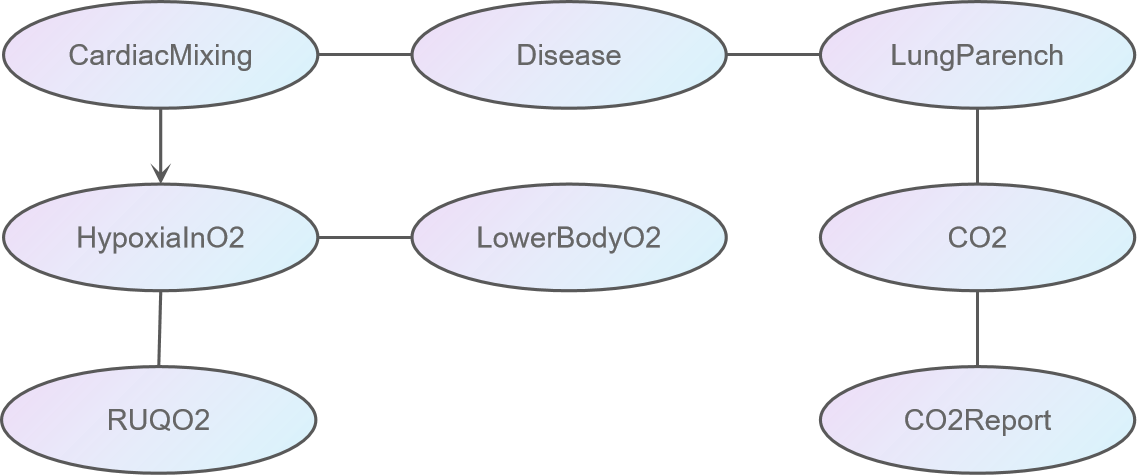}
    \end{minipage}
}

\caption{Parts of the CPDAG structure of CHILD dataset. (a) is the true CPDAG of CHILD dataset. (b) (c) (d) are the CPDAGs learned from PC, CDNOD, and EEMBI-PC algorithms.}\label{child}
\end{figure*}

In Figure \ref{child}, CPDAGs (b) (c) learned from PC and CDNOD both fail to capture the connection between $HypoxianlnO2$ and $LungParench$, as well as the connection between $Disease$ and $LungParench$. (b) incorrectly connects $Disease$ and $HypoxianlnO2$. (c) mistakenly adds 3 edges: $Disease\to HypoxianlnO2$, $RUQO2\to LowerBodyO2$ and $CO2Report\to LungParench$. Furthermore, (c) incorrectly determines the directions of 4 edges among these right connections like $RUQO2\to HypoxianlnO2$ and $CO2\to LungParench$. (d) learned from EEMBI-PC only misses one connection between $HypoxianlnO2$ and $LungParench$, and it only has two connections with wrong directions.

\begin{figure*}[htbp]
\centering

\subfigure[]
{
    \begin{minipage}[b]{.48\linewidth}
        \centering
        \includegraphics[scale=0.35]{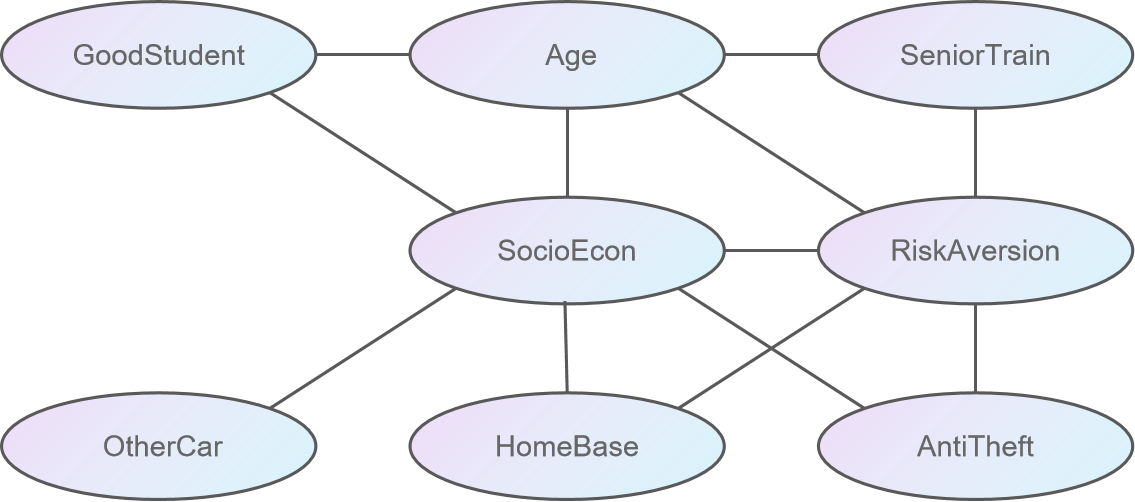}
    \end{minipage}
}
\subfigure[]
{
    \begin{minipage}[b]{.48\linewidth}
        \centering
        \includegraphics[scale=0.35]{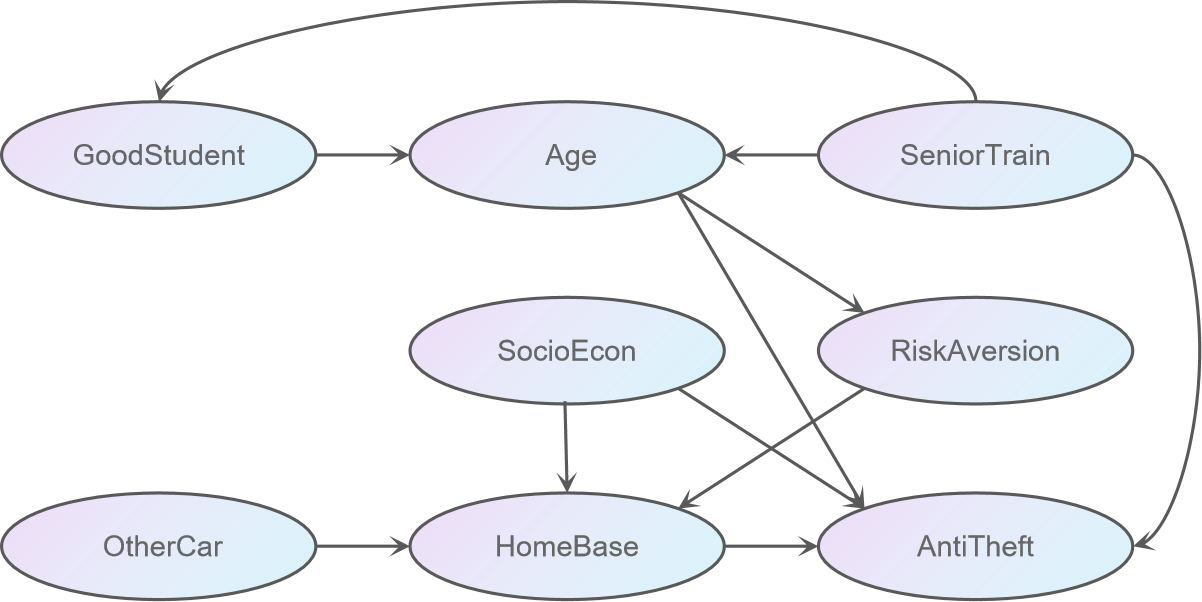}
    \end{minipage}
}

\subfigure[]
{
    \begin{minipage}[b]{.48\linewidth}
        \centering
        \includegraphics[scale=0.35]{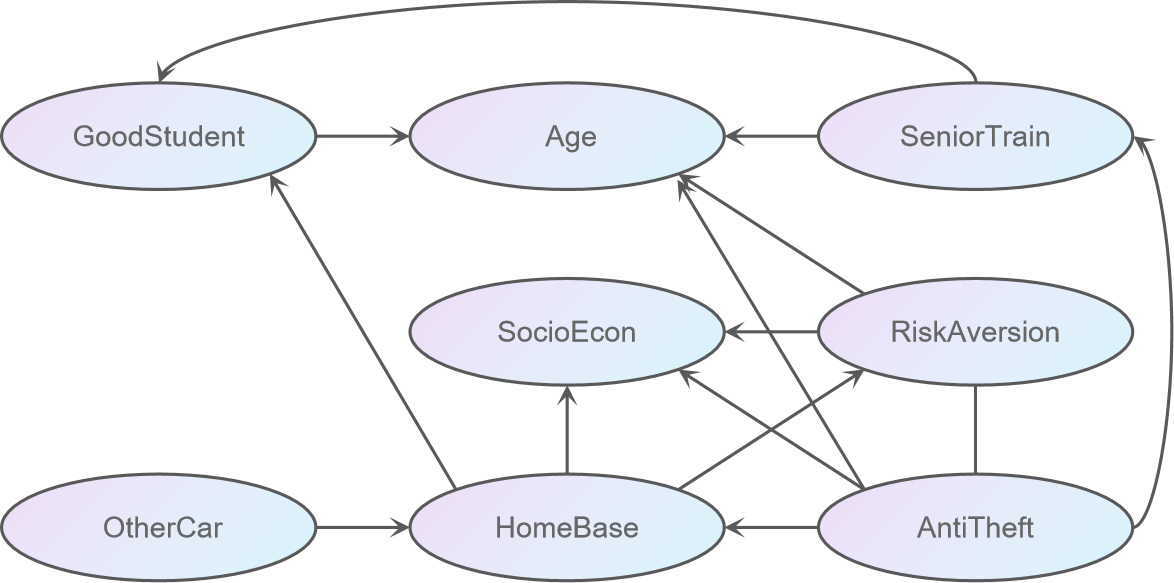}
    \end{minipage}
}
\subfigure[]
{
    \begin{minipage}[b]{.48\linewidth}
        \centering
        \includegraphics[scale=0.35]{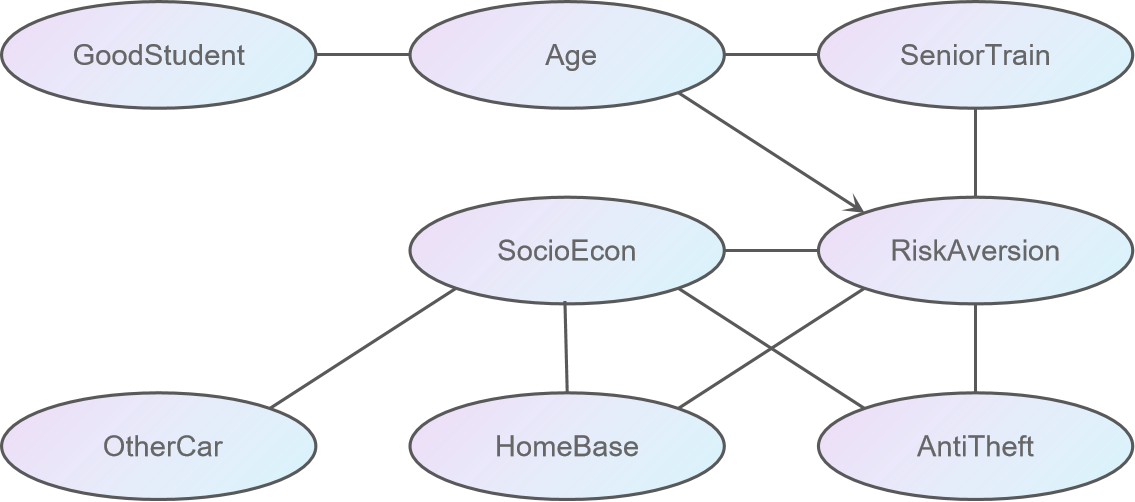}
    \end{minipage}
}

\caption{Parts of the CPDAG structure of INSURANCE dataset. (a) is the true CPDAG of CHILD dataset. (b) (c) (d) are the CPDAGs learned from PC, CDNOD, and EEMBI-PC algorithms}\label{ins}
\end{figure*}

In Figure \ref{ins}, (b) additionally connects 5 edges, like $SeniorTrain\to GoodStudent$, misses 6 edges, such as $GoodStudent- SocioEcon$, and have the wrong directions on all the edges. (c) connects 6 edges and omits 4 edges mistakenly. And it only has the right direction on $RiskAversion-AntiTheft$. (d) does not have any additional edge and only miss two edges $Age- SocioEcon$ and $GoodStudent-SocioEcon$. More surprisingly, (d) learn the right direction of all connections except for $Age\to RsikAversion$.

\subsection{Continuous datasets}
Table \ref{Education SHD}, \ref{Dream3 SHD}, and Figure \ref{Education AUPR}, \ref{Dream3 AUPR} shows the comparison results on continuous datasets. On SACHS, EEMBI-PC has the lowest SHD and reaches the top on AUPR. Although EEMBI and NOTEARS are close to EEMBI-PC on SHD, they are exceeded by many baselines like PC, FCI, and GIES on AUPR.

\begin{table*}[htpb]
\centering
\fontsize{10}{22}\selectfont
\caption{SHD on SAHCS and Education datasets}\label{Education SHD}
\setlength{\tabcolsep}{2mm}{
\begin{tabular}{cccccccc}
\bottomrule[1.5pt]
\hline
\multirow{2}{*}{Algorithms} & \multirow{2}{*}{SACHS} & \multirow{2}{*}{Education 1} & \multirow{2}{*}{Education 2} & \multirow{2}{*}{Education 3} & \multirow{2}{*}{Education 4} & \multirow{2}{*}{Education 5} & \cr
 & \cr
\bottomrule
PC           & 32.0 & 636.0(5.72) & 660.0(11.58) & 628.3(14.06) & 621.0(2.83) & 633.0(9.42) \cr
FCI          & 43.0 & 650.3(7.13) & 671.3(10.40) & 632.0(3.74) & 641.3(3.68) & 669.7(3.68) \cr
GIES         & 38.0 & 692.3(31.86) & 700.7(35.37) & 665.3(23.16) & 677.7(9.29) & 685.7(32.74) \cr
MMHC         & 37.0 & 662.3(7.72) & 689.7(5.56) & 639.7(4.19) & 650.0(2.94) & 671.7(7.41) \cr
GS           & 35.0 & 649.3(9.39) & 667.3(3.40) & 608.7(4.78) & 636.7(15.43) & 646.3(9.46) \cr
GRaSP        & 36.0 & 686.3(3.09) & 677.3(9.46) & 646.7(13.60) & 646.3(8.34) & 672.0(18.06) \cr
CDNOD        & 37.0 & 650.7(8.34) & 664.7(4.03) & 617.0(2.94) & 619.0(5.89) & 649.3(6.02) \cr
DirectLiNGAM & 51.0 & 658.7(7.54) & 629.7(11.73) & 742.7(91.00) & 781.7(114.44) & 757.3(93.83) \cr
CAM          & 37.0 & 659.3(6.55) & 679.0(10.42) & 643.3(9.74) & 671.0(15.77) & 672.7(33.32) \cr
NOTEARS      & 28.0 & 628.0(4.08) & 644.7(5.44) & 593.3(5.31) & 612.3(2.49) & 630.3(3.30) \cr
EEMBI        & 30.0 & \textbf{618.0(5.10)} & \textbf{627.0(2.16)} & \textbf{582.3(6.13)} & \textbf{580.3(2.49)} & \textbf{610.0(5.10)} \cr
EEMBI-PC     & \textbf{27.0} & 644.7(6.94) & 654.0(4.32) & 608.0(1.41) & 619.0(5.72) & 634.7(6.65) \cr
 \bottomrule[1.5pt]
\end{tabular}}
\end{table*}

\begin{figure*}
\centering
\includegraphics[scale=0.25]{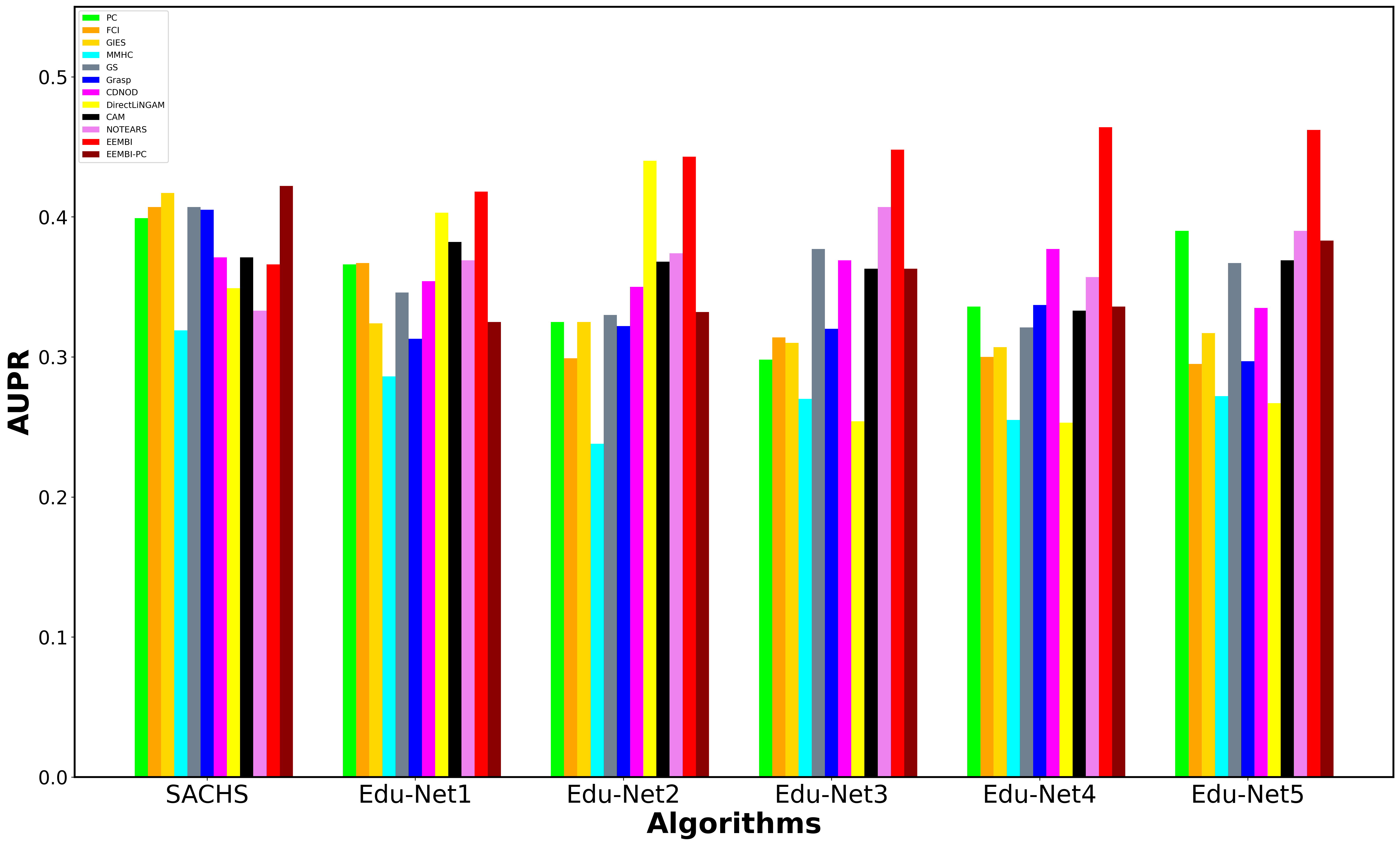}
\caption{AUPR on SACHS and Education datasets}\label{Education AUPR}
\end{figure*}

On Education datasets, EEMBI shows its dominance. It reaches the lowest SHDs and the highest AUPRs of these five datasets. EEMBI is the only method whose SHD is lower than 600 on Edu-Net 4, and it is one of the two methods whose SHD is lower than 600 on Edu-Net 3, together with NOTEARS. Moreover, EEMBI has very small standard deviations. It has the smallest deviations on Edu-Net 2 (2.16) and Education 4 (2.39) and has the third smallest deviations on Edu-Net 1 (5.10) and Education 5 (5.10). Small standard deviations indicate that EEMBI is much more robust to noisy and randomicity of data. EEMBI-PC has poor performance on these five datasets. Although EEMBI-PC outperforms baselines except for the PC algorithm and NOTEARS on SHD, it still has big gaps with EEMBI. For AUPR, EEMBI-PC is exceeded by CDNOD and NOTEARS on Edu-Net 1, 2, 3 and basically reaches the bottom on Edu-Net 1. NOTEATRS has better performance than other baseline algorithms and EEMBI-PC on SHD, but is still beaten by EEMBI on every Education dataset.

\begin{table*}[htpb]
\centering
\fontsize{10}{18}\selectfont
\caption{SHD on Dream3 datasets}\label{Dream3 SHD}
\setlength{\tabcolsep}{2mm}{
\begin{tabular}{cccccc}
\bottomrule[1.5pt]
\hline
\multirow{2}{*}{Algorithms} & \multirow{2}{*}{Ecoli 1} & \multirow{2}{*}{Ecoli 2} & \multirow{2}{*}{Yeast 1} & \multirow{2}{*}{Yeast 2} & \multirow{2}{*}{Yeast 3}\cr
 & \cr
\bottomrule

PC           & 212.0 & 216.0 & 216.0 & 257.0 & 293.0 \cr
FCI          & 289.0 & 290.0 & 286.0 & 350.0 & 361.0 \cr
GIES         & 262.0 & 280.0 & 261.0 & 296.0 & 298.0 \cr
MMHC         & 205.0 & 227.0 & 281.0 & 281.0 & 301.0 \cr
GS           & 164.0 & 192.0 & 192.0 & 213.0 & 233.0 \cr
GRaSP        & 204.0 & 426.0 & 200.0 & 298.0 & 234.0 \cr
CDNOD        & 230.0 & 249.0 & 236.0 & 315.0 & 313.0 \cr
DirectLiNGAM & 219.0 & 247.0 & 219.0 & 289.0 & 289.0 \cr
CAM          & 360.0 & 346.0 & 354.0 & 390.0 & 392.0 \cr
NOTEARS      & 157.0 & 181.0 & 195.0 & 230.0 & 253.0 \cr
EEMBI        & \textbf{124.0} & 164.0 & \textbf{143.0} & \textbf{215.0} & \textbf{226.0} \cr
EEMBI-PC     & 144.0 & \textbf{158.0} & 158.0 & 220.0 & 236.0 \cr
 \bottomrule[1.5pt]
\end{tabular}}
\end{table*}

\begin{figure*}
\centering
\includegraphics[scale=0.25]{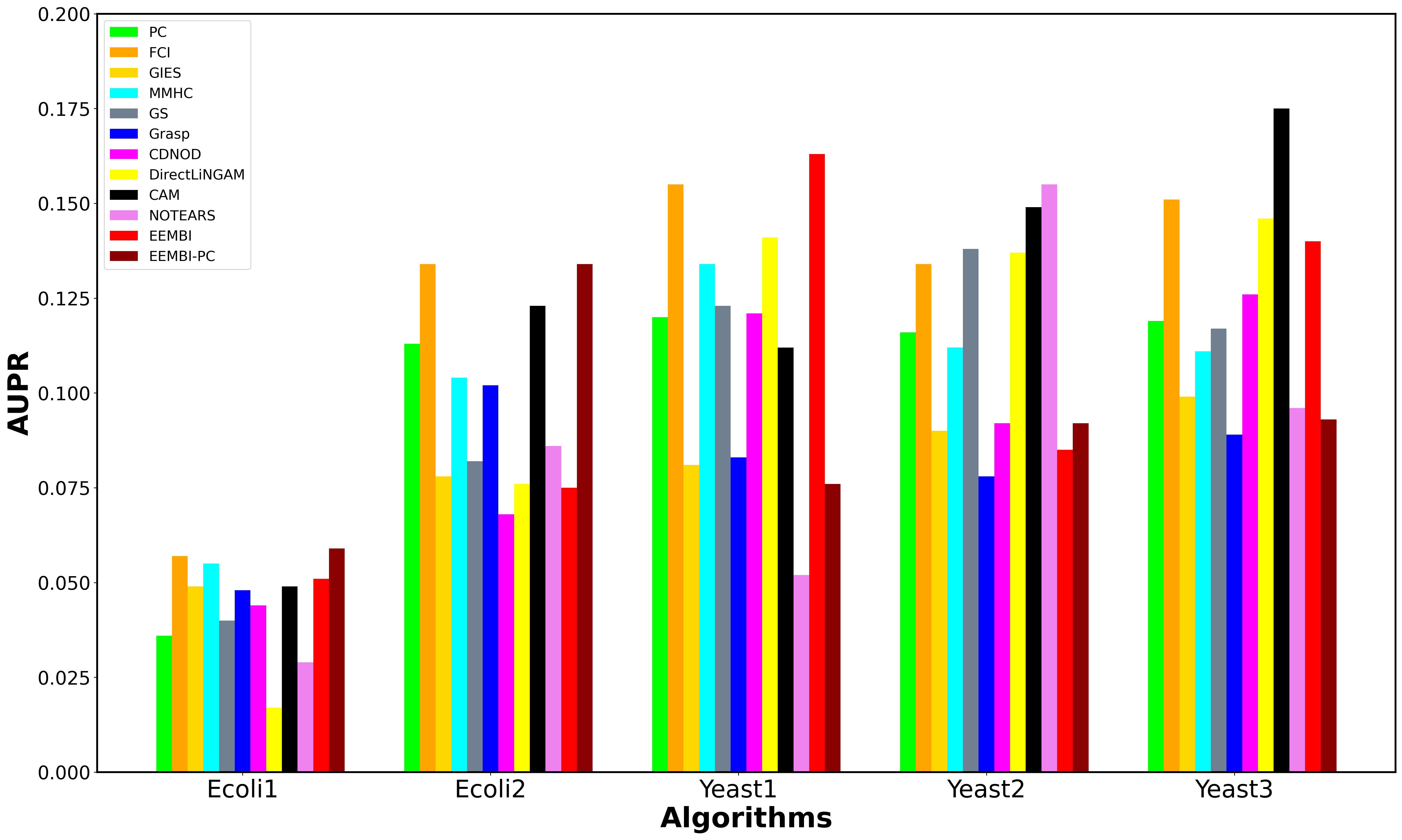}
\caption{AUPR on Dream3 datasets}\label{Dream3 AUPR}
\end{figure*}

For Dream3 datasets, EEMBI achieves the lowest SHD among all methods, except for the Ecoli 2 dataset where it is close to the best-performing method, EEMBI-PC. but it only has the highest AUPR on Yeast 1. On the other hand, EEMBI-PC performs exceptionally well on the Ecoli 1 and Ecoli 2 datasets, achieving the highest AUPR. Similar to discrete datasets, EEMBI and EEMBI-PC have the lowest two SHDs on every dream dataset. However, they both show poorer performance in terms of AUPR on Yeast 2, 3. In contrast, NOTEARS and CAM achieve the tops on them, and NOTEARS has the closest SHD to proposed methods among the baseline algorithms.

EEMBI outperforms baselines and EEMBI-PC on continuous datasets overall.

\subsection{Sensitiveness and Ablation Study}

\begin{figure*}[htbp]
\centering

\subfigure[]
{
    \begin{minipage}[b]{.48\linewidth}
        \centering
        \includegraphics[scale=0.14]{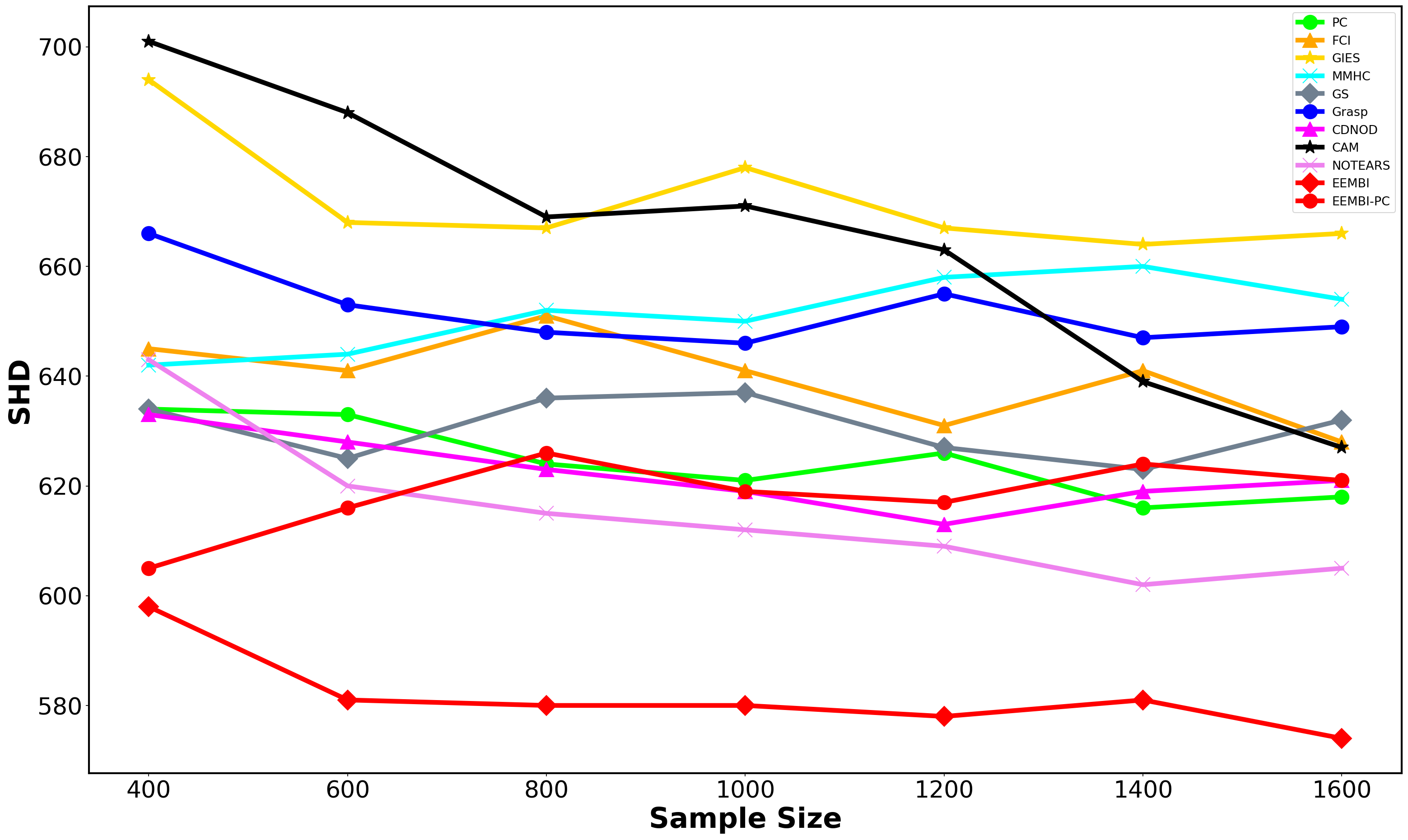}
    \end{minipage}
}
\subfigure[]
{
    \begin{minipage}[b]{.48\linewidth}
        \centering
        \includegraphics[scale=0.14]{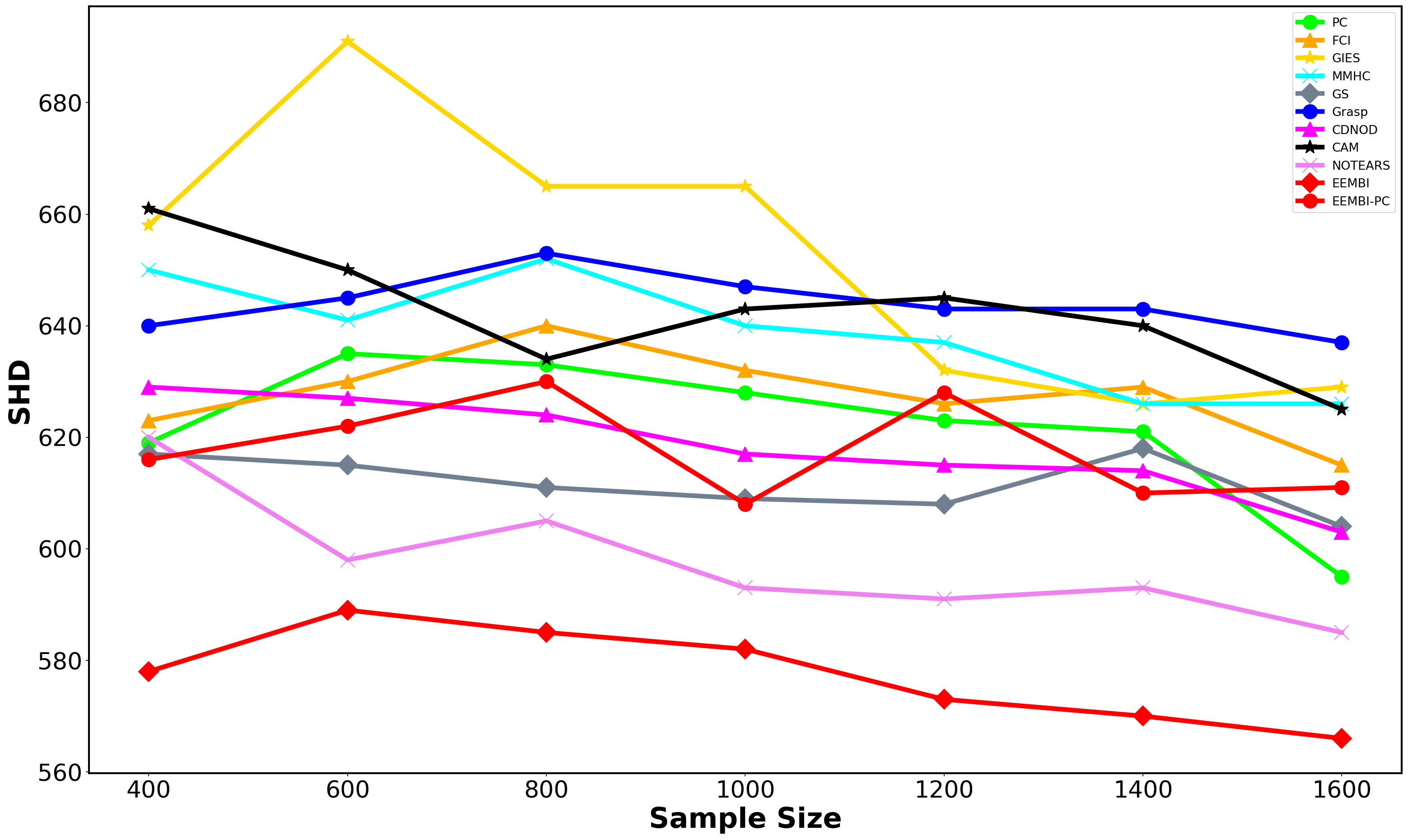}
    \end{minipage}
}

\subfigure[]
{
    \begin{minipage}[b]{.48\linewidth}
        \centering
        \includegraphics[scale=0.14]{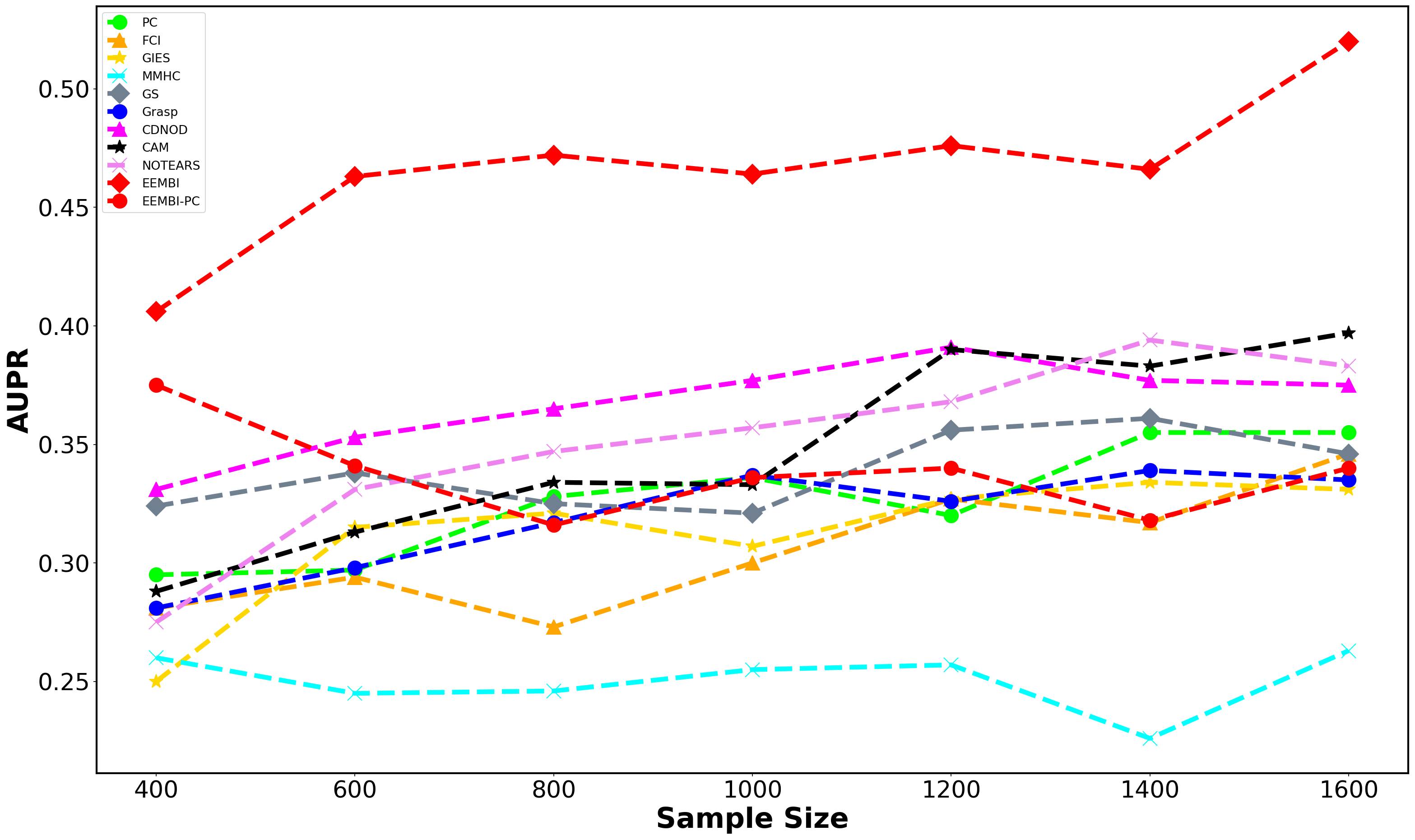}
    \end{minipage}
}
\subfigure[]
{
    \begin{minipage}[b]{.48\linewidth}
        \centering
        \includegraphics[scale=0.14]{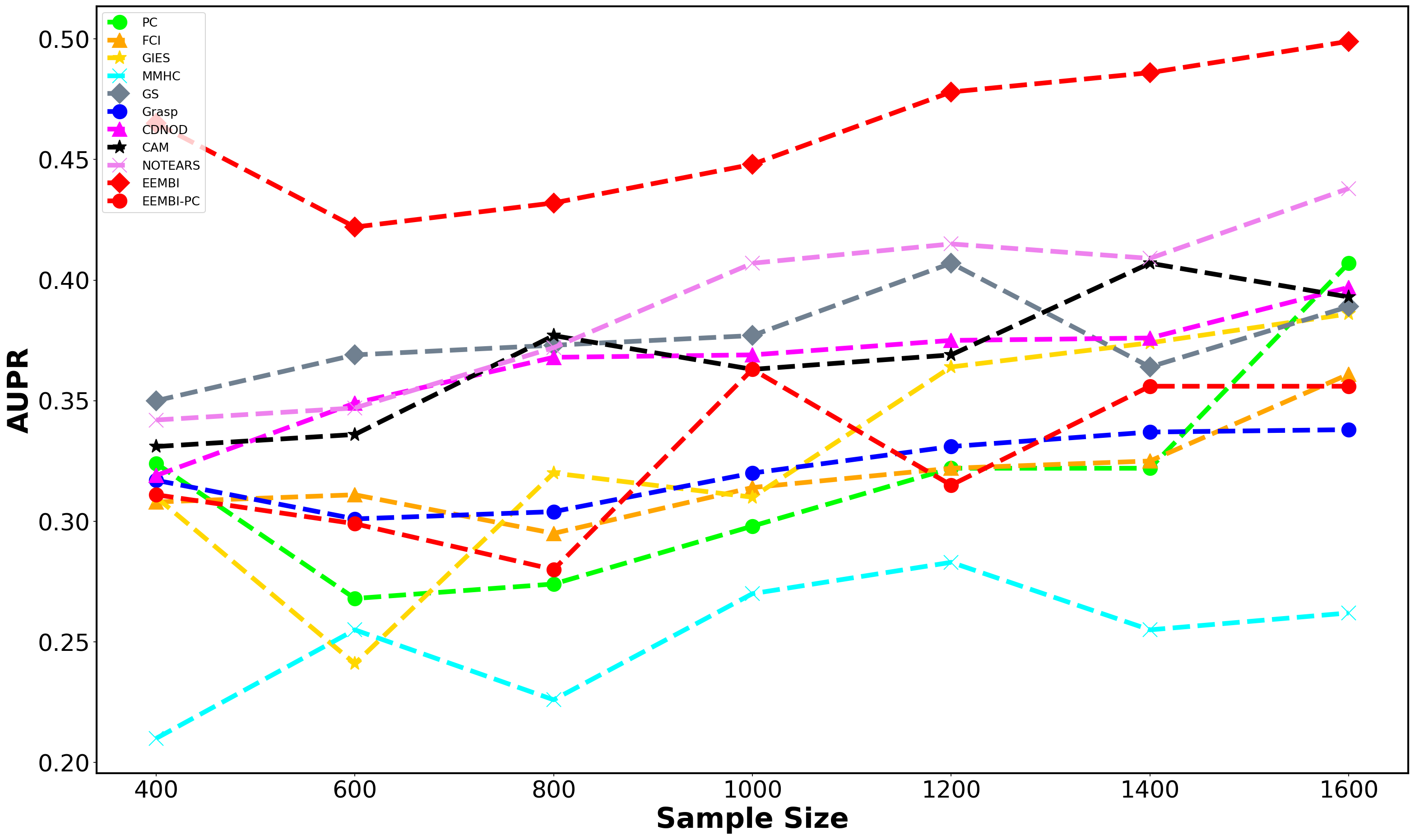}
    \end{minipage}
}

\caption{The change of SHD and AUPR of causal structure learning algorithms with respect to sample size on Education 3, 4. (a) (b) show the SHD with respect to sample size and (c) (d) show the AUPR with respect to sample size.}\label{sample size}
\end{figure*}

In Figure \ref{sample size}, we show the performance of baselines and proposed methods with respect to the sample size on Edu-Net 3, 4 datasets. We consider the performance of causal learning algorithms on 400, 600, 800, 1000, 1200, 1400, and 1600 sample sizes. Similarly, we apply every algorithm three times with different instances on every sample size and show the mean values on graphs. We remove the DirectLiNGAM since its poor results affect the presentation of other algorithms. In (a) (b), The SHDs have decreasing tendency with the increase of sample size. EEMBI has the smallest SHD on any sample size, it only decreases at the beginning in (a) and is stable after 600 sample size. In (b), EEMBI begins to decrease at the 600 sample size and gradually becomes steady. EEMBI-PC fluctuates intensely in (a) (b). And NOTEARS has the second lowest SHDs except on small sample sizes. Most of the methods, including EEMBI and NOTEARS, decrease slightly with respect to sample size. However, GIES in (a) and CAM in (b) decrease dramatically, which indicates that GIES and CAM are sensitive to sample size. For (c) (d), all methods have increasing tendencies. EEMBI reaches the top of the graph all the time, and it has a completely opposite behavior compared to its performance in (a) (b). EEMBI-PC still has fluctuation corresponding to EEMBI-PC in (a) (b). Surprisingly, MMHC has the lowest AUPR on most of the sample sizes and has big gaps to other baselines.

\begin{figure*}[htbp]
\centering

\subfigure[]
{
    \begin{minipage}[b]{.23\linewidth}
        \centering
        \includegraphics[scale=0.15]{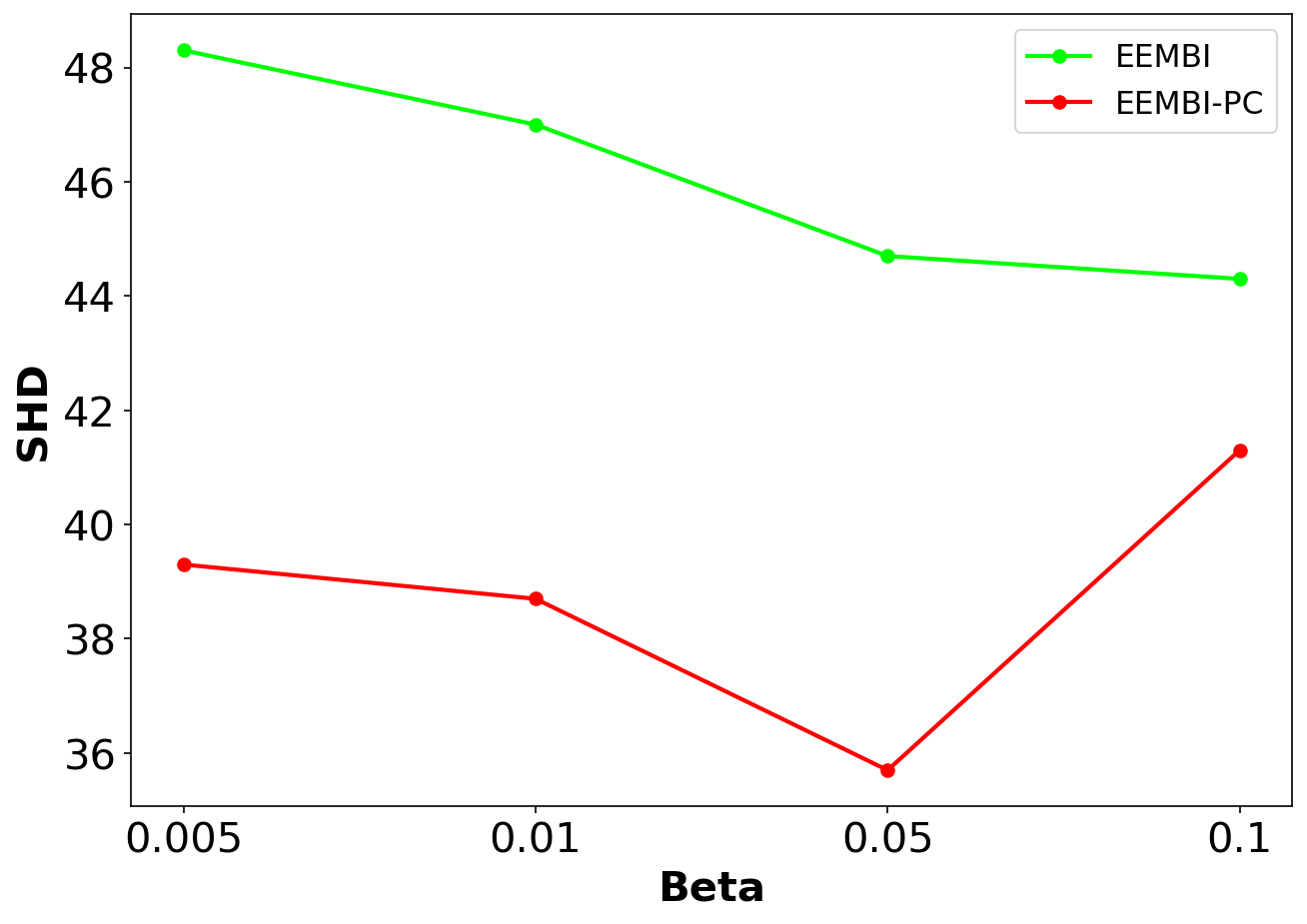}
    \end{minipage}
}
\subfigure[]
{
    \begin{minipage}[b]{.23\linewidth}
        \centering
        \includegraphics[scale=0.15]{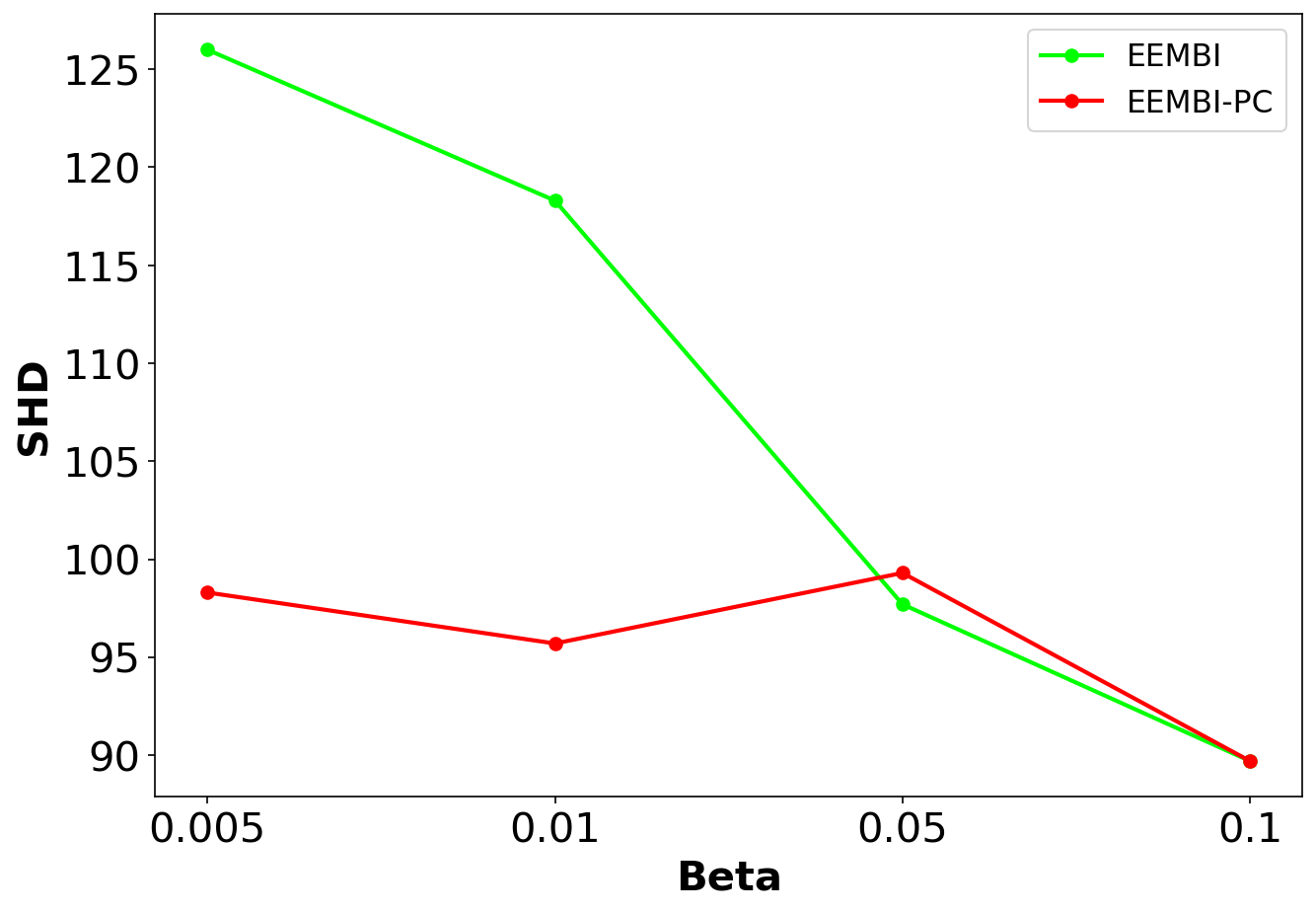}
    \end{minipage}
}
\subfigure[]
{
 	\begin{minipage}[b]{.23\linewidth}
        \centering
        \includegraphics[scale=0.15]{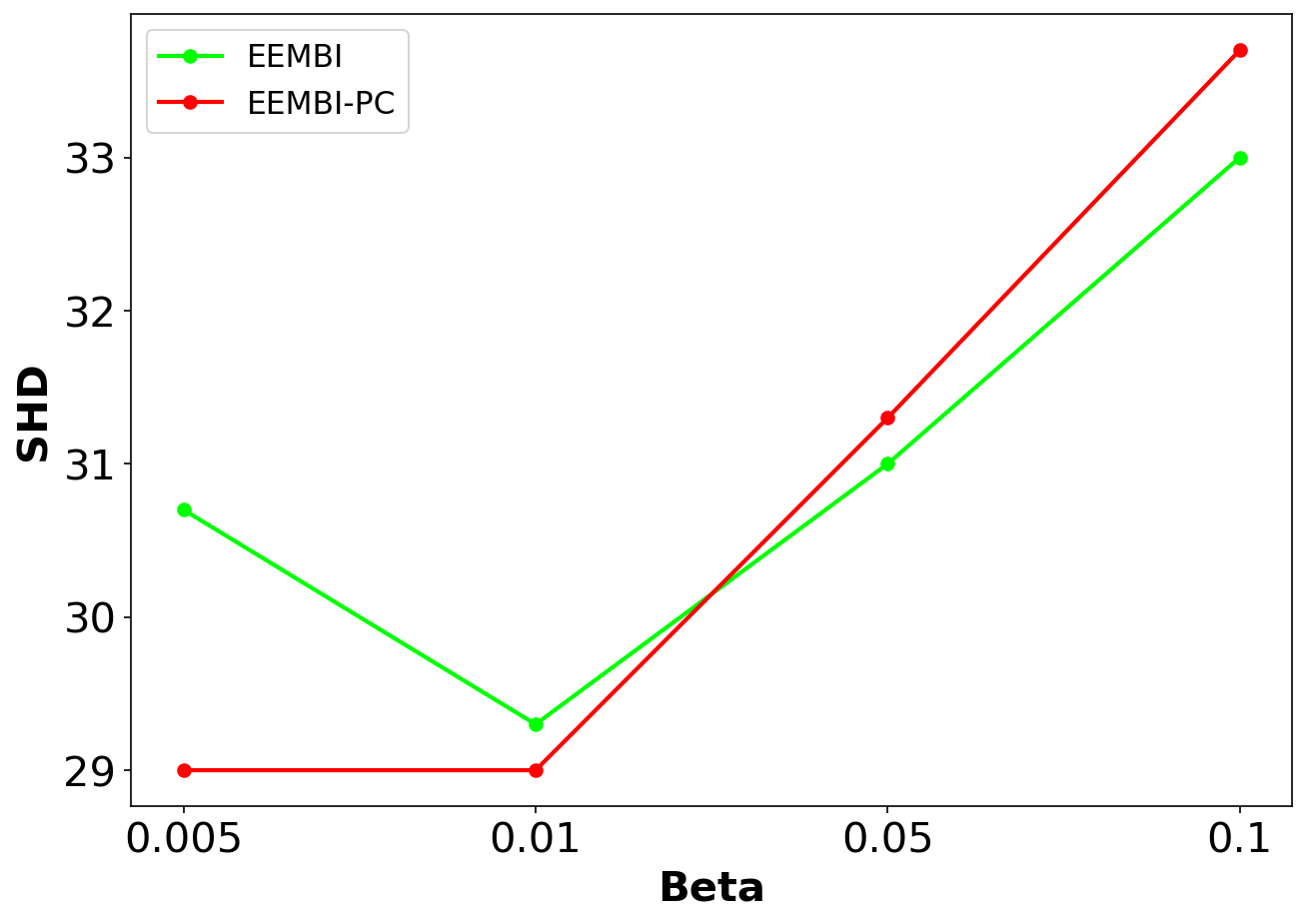}
    \end{minipage}
}
\subfigure[]
{
 	\begin{minipage}[b]{.23\linewidth}
        \centering
        \includegraphics[scale=0.15]{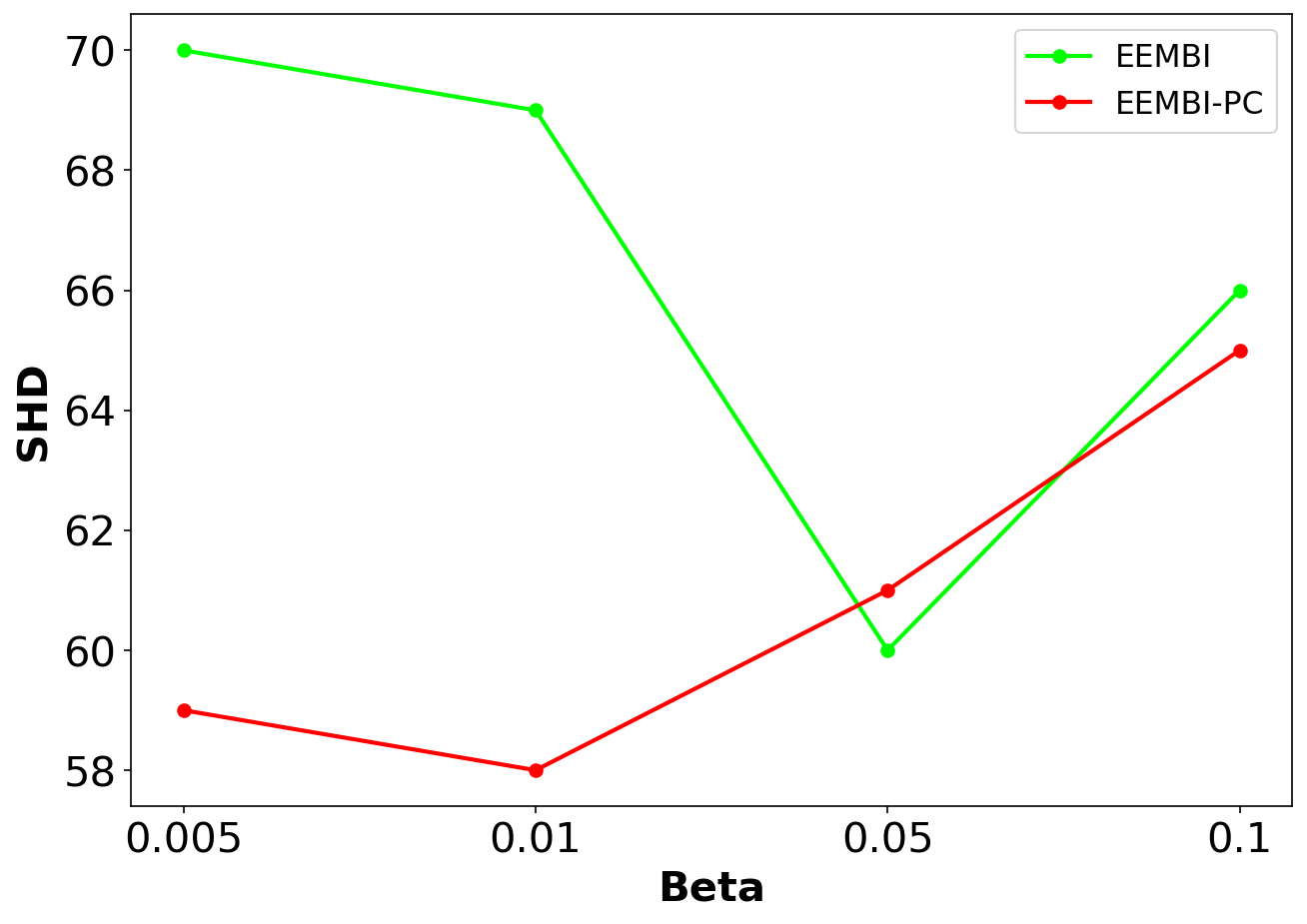}
    \end{minipage}
}

\subfigure[]
{
    \begin{minipage}[b]{.23\linewidth}
        \centering
        \includegraphics[scale=0.15]{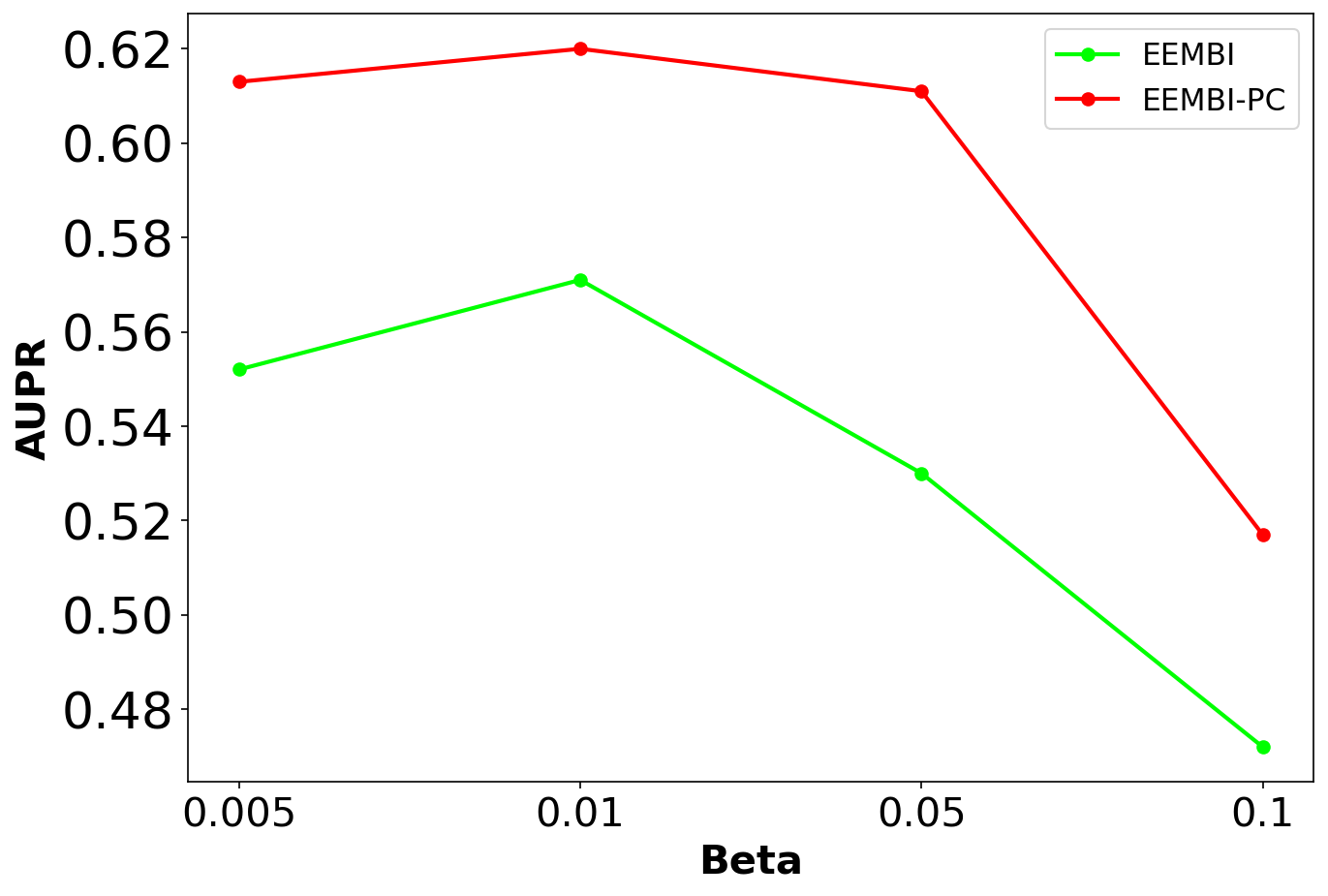}
    \end{minipage}
}
\subfigure[]
{
    \begin{minipage}[b]{.23\linewidth}
        \centering
        \includegraphics[scale=0.15]{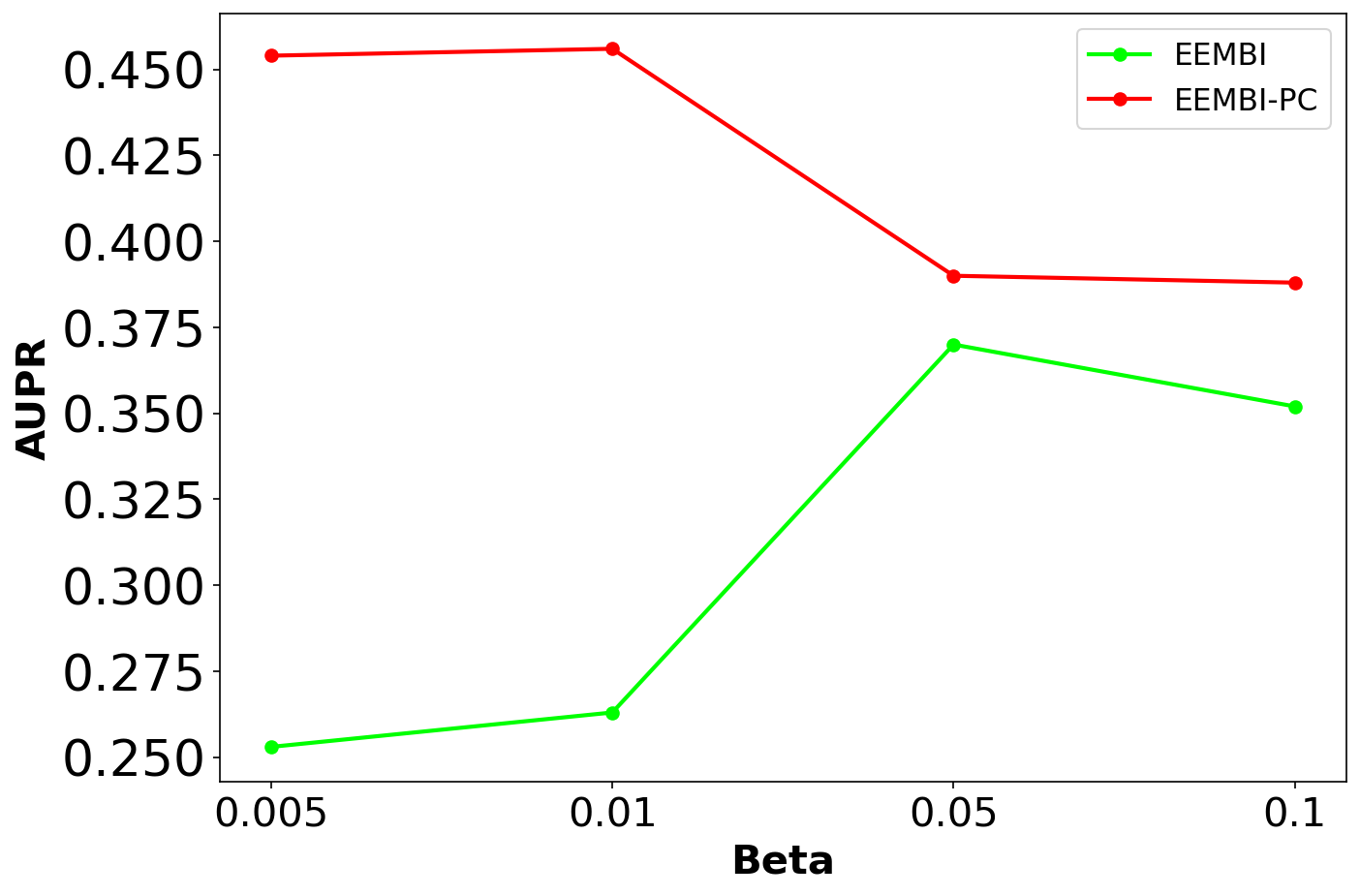}
    \end{minipage}
}
\subfigure[]
{
 	\begin{minipage}[b]{.23\linewidth}
        \centering
        \includegraphics[scale=0.15]{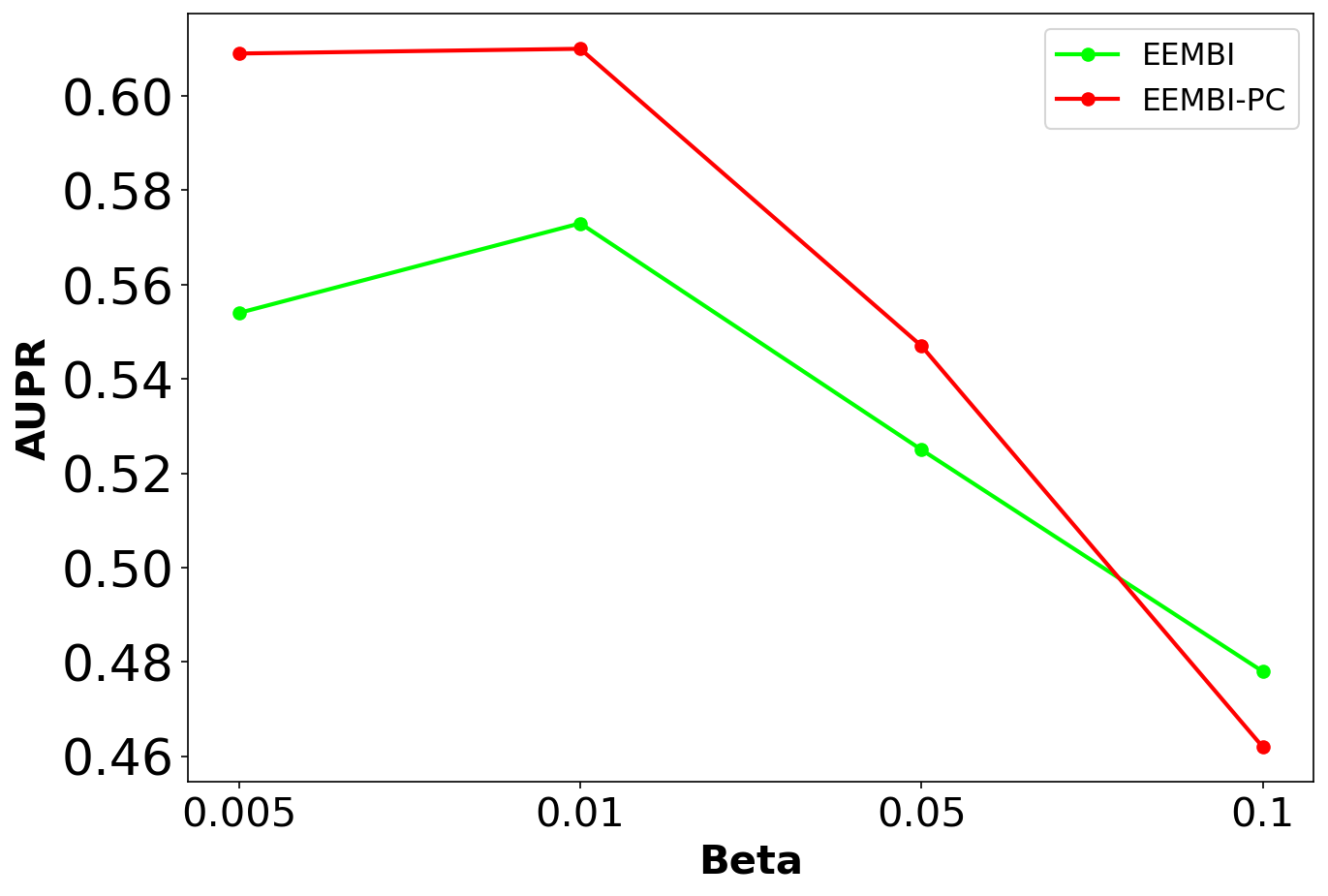}
    \end{minipage}
}
\subfigure[]
{
 	\begin{minipage}[b]{.23\linewidth}
        \centering
        \includegraphics[scale=0.15]{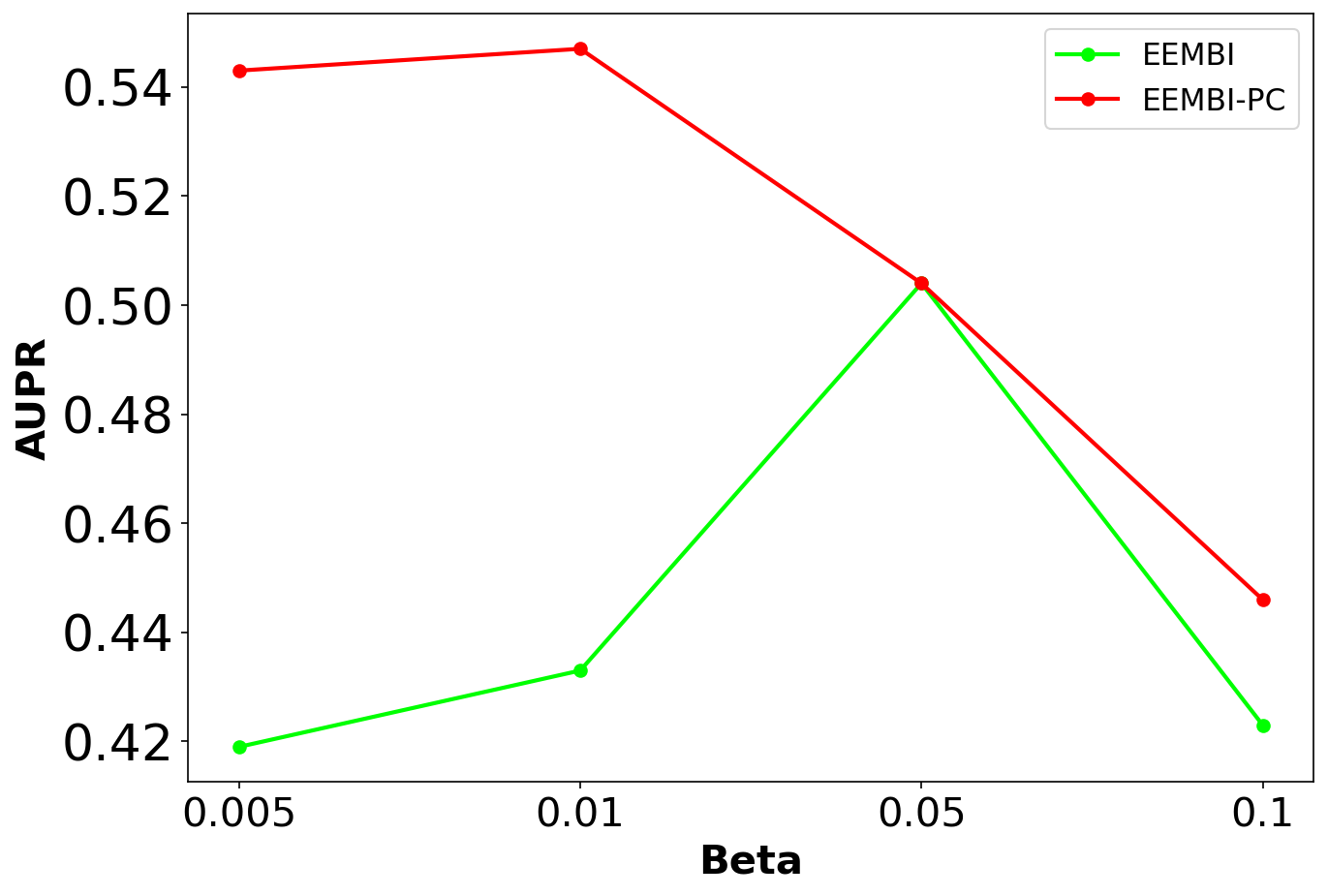}
    \end{minipage}
}

\caption{The SHD and AUPR of EEMBI and EEMBI-PC with respect to $\beta$ on ALARM ((a),(e)), BARLEY ((b), (f)), CHILD ((c), (g)), and INSURANCE ((d), (h))}\label{beta}
\end{figure*}

Since we fix the hyperparameter $\alpha$ and use different $\beta$ on discrete and continuous datasets. We are only interested in the influence of $\beta$. We study the performances of proposed methods, EEMBI and EEMBI-PC, with respect to $\beta$ on four discrete datasets in Figure \ref{beta}. EEMBI-PC outperforms EEMBI on every dataset except for a few points. EEMBI-PC with $\beta=0.01$ shows the lowest SHDs on CHILD and INSURANCE, and the highest AUPR on all four datasets. However, the regularity of EEMBI is much more complicated. EEMBI has its best performance under $\beta=0.01$ on CHILD but with $\beta=0.05$ on INSURANCE. Moreover, it reaches its minimum SHD under $\beta=0.1$ but reaches its maximum AUPR under $\beta=0.01$ on ALARM. In conclusion, EEMBI-PC with $\beta=0.05$ is the best method for discrete datasets.

To prove every step in EEMBI and EEMBI-PC is useful, an ablation study is needed. After getting the Markov blankets from improved IAMB, we link every node with the nodes in its Markov blanket using undirected edges, and the undirected graph is represented by an adjacency matrix and is compared with the true CPDAG of datasets. Then we get the SHD and AUPR of improved IAMB on datasets. We also remove the matching phase in EEMBI and EEMBI-PC, i.e. delete the line 7 to line 14 in \textbf{Algorithm \ref{GM}} and directly use the exogenous variables generated by FastICA as the input of \textbf{Algorithm \ref{Intersection}}. We denote the EEMBI and EEMBI-PC without matching phase as EEMBI (WM) and EEMBI-PC (WM) where WM is short for ``Without Matching''. We run improved IAMB, EEMBI (WM), and EEMBI-PC (WM) three times, and compare them with original EEMBI and EEMBI-PC in Table \ref{ablation}. The results of AUPR are shown in \textbf{Appendix B}. We may compare improved IAMB with EEMBI and EEBI-PC to demonstrate the effectiveness of the combination of \textbf{Algorithm \ref{GM}} and \textbf{Algorithm \ref{Intersection}}. We also may compare EEMBI with EEMBI (WM), or compare EEMBI-PC with EEMBI-PC (WM) to state the effectiveness of \textbf{Algorithm \ref{GM}} solo. We can find that EEMBI or EEMBI-PC has a much smaller SHD than EEMBI (WM) or EEMBI-PC (WM), and improved IAMB. For example, on Edu-Net 1, the SHD of EEMBI (618.0) is smaller than the SHD of EEMBI (WM) (647.3) and the SHD of improved IAMB (676.0). Although EEMBI (WM) has lower SHD than EEMBI-PC on Edu-Net 3 and Edu-Net 4, this comparison is meaningless since they have different steps besides the matching phases.

\begin{table*}[htpb]
\centering
\fontsize{10}{22}\selectfont
\caption{SHD on Dream3 datasets}\label{ablation}
\setlength{\tabcolsep}{2mm}{
\begin{tabular}{cccccc}
\bottomrule[1.5pt]
\hline
\multirow{2}{*}{Algorithms} & \multirow{2}{*}{Education 1} & \multirow{2}{*}{Education 2} & \multirow{2}{*}{Education 3} & \multirow{2}{*}{Education 4} & \multirow{2}{*}{Education 5}\cr
 & \cr
\bottomrule

Improved IAMB  & 676.0(7.48)  & 682.7(6.18) & 639.7(2.49) & 651.7(4.99)  & 679.7(5.25) \cr\hline
EEMBI (WM)     & 647.3(3.68)  & 660.0(2.16) & 603.7(3.40) & 613.0(12.75) & 641.3(6.94) \cr
EEMBI-PC (WM)  & 665.3(12.26) & 670.0(1.41) & 630.0(0.82) & 636.3(4.19)  & 657.0(6.38) \cr\hline
EEMBI          & 618.0(5.10)  & 627.0(2.16) & 582.3(6.13) & 580.3(2.49)  & 610.0(5.10) \cr
EEMBI-PC       & 644.7(6.94)  & 654.0(4.32) & 608.0(1.41) & 619.0(5.72)  & 634.7(6.65) \cr
 \bottomrule[1.5pt]
\end{tabular}}
\end{table*}

\section{Conclusion}

This paper proposes a pair of new causal structure learning algorithms: EEMBI and EEMBI-PC. They use improved IAMB to learn the Markov blankets for all nodes. Different from original IAMB, improved IAMB has an extra phase to guarantee its accuracy, and it uses kNN-CMI as estimation of CMI so that it can have more accurate conditional independencies and can work on both discrete and continuous datasets directly. FastICA is implemented to generate exogenous variables. To match every exogenous variable with its endogenous variable, we turn the problem into an optimization problem in equation (\ref{maximize MI}) and propose generating and matching algorithm in \textbf{Algorithm \ref{GM}}. Using the properties of exogenous variables, we prove that the parents of endogenous variables belong to the Markov blankets of the corresponding exogenous variables. So we intersect endogenous Markov blankets and exogenous Markov blankets to find the parents of nodes in \textbf{Algorithm \ref{Intersection}}. We have different algorithms by using different strategies to orient the edges in the final phases. For EEMBI, we directly orient the edges from the parents node found in \textbf{Algorithm \ref{Intersection}} to the target node and obtain a DAG. And we turn the DAG into CPDAG by using Meek's rules. For EEMBI-PC, we use undirected edges to link the parents with the target node to obtain the skeleton of the DAG. And we use the PC algorithm and Meek's rules to obtain the CPDAG. The experiments give empirical evidence that EEMBI-PC has the best performance on discrete datasets and EEMBI is the state-of-the-art method on continuous datasets.

Although EEMBI only needs polynomial complexity $O(n^3)$, the proposed methods still require much more computational time compare to baseline algorithms in experiments, especially under large sample sizes. This is primarily due to the complex computation of kNN-CMI and improved IAMB costs most of the time. In our future studies, we aim to explore a more efficient estimation for CMI to enhance the efficiency of improved IAMB. Furthermore, we are interested in finding an ICA method that can accurately recover the exogenous variables from observed data so that EEMBI and EEMBI-PC can be much more accurate.


\acks{This work was funded by the National Nature Science Foundation of China under Grant 12071428 and 11671418, and Zhejiang Provincial Natural Science Foundation of China under Grant No. LZ20A010002.}


\appendix



\section{Proof}
We give proofs of theorems in this appendix.

\begin{proof}{\textbf{of Theorem \ref{MB-def}}}

Let $x, y\in\mathcal X$ be two endogenous variables, and $x,y$ are connected by some trails. Since any trails that connect $x$ and $y$ must go through the Markov blanket of $x$, we have three conditions.

\textbf{condition 1}\\
If trail $x\rightleftharpoons x_1\rightleftharpoons x_2 ... \rightleftharpoons y$ goes through one of the parents of $x$, i.e. $x\gets x_1$. Since $x\gets x_1\rightleftharpoons x_1$ forms either a fork structure or a chain structure. Therefore given $x_1\in MB_x$ can make this trail not active. 

If trail $x\rightleftharpoons x_1\rightleftharpoons x_2\rightleftharpoons x_3 ... \rightleftharpoons y$ goes through one of the children of $x$ $x_1$. Since $x_1\in Ch_x$, we have $x\to x_1$. Then the other two conditions follow.

\textbf{condition 2}\\
If $x\to x_1\to x_2$ forms a chain structure, then given $x_1\in MB_x$ unactive this trail.

\textbf{condition 3}\\
If $x\to x_1\gets x_2$ forms a V-structure, so $x_2$ is one of the spouses of $x$. Then $x_1\gets x_2\rightleftharpoons x_3$ forms a fork structure or chain structure, and given $x_2\in MB_x$ can make this trail not active.

Combing these three conditions, we know any trails between $x$ and $y$ are not active given $MB_x$, thus $x\perp\!\!\!\perp y\mid MB_x$. The proof for $x\perp\!\!\!\perp y\mid MB_y$ can be achieved in a similar way.

Let $Z\subset\mathcal X$, $y\notin Z$ and $Z\cap MB_x=\emptyset$.
\begin{align*}
P(x\mid Z,MB_x)&=\sum_{y} P(x\mid y, Z, MB_x)P(y\mid Z, MB_x)\\
&=\sum_y P(x\mid MB_x)P(y\mid Z, MB_x)\\
&=P(x\mid MB_x)
\end{align*}
Since $x\perp\!\!\!\perp Z,y\mid MB_x$, we have
\begin{align*}
P(x\mid y, Z, MB_x)=P(x\mid MB_x)=P(x\mid Z, MB_x).
\end{align*}
Therefore $x\perp\!\!\!\perp y\mid Z\cup MB_x$,
\end{proof}

\begin{proof}{\textbf{of Theorem \ref{trail MI}}}

Consider a chain structure $x\to y\to z$. By the definition of mutual information $I(x\ ;y)=H(x)-H(x\mid y)$ and conditional mutual information $I(x\ ;y\mid z)=H(x\mid z)-H(x\mid y, z)$, we have 
\begin{align*}
I(x\ ;y)+I(x\ ;z\mid y)&=H(x)-H(x\mid y)+H(x\mid y)-H(x\mid y,z)\\
&=H(x)-H(x\mid y,z)\\
&=I(x\ ;y,z)
\end{align*}
According to the symmetry, we have
\begin{align*}
I(x\ ;y,z)&=I(x\ ;y)+I(x\ ;z\mid y)\\
&=I(x\ ;z)+I(x\ ;y\mid z)
\end{align*}
By the property of chain structure, $I(x\ ;z\mid y)=0$. Since $x$ is not independent with $y$ given $z$, i.e. $I(x\ ;y\mid z)>0$, we have $I(x\ ;y)>I(x\ ;z)$. We have the same conclusion for the fork structure $x\gets y\to z$ since it has the same conditional independency as the chain structure. For V-structure $x\to y\gets z$, it is easy to see $I(x\ ;y)>I(x\ ;z)=0$.

Now we expand the structure from the basic structure to a trail $x\rightleftharpoons x_1 \rightleftharpoons x_2 ... \rightleftharpoons x_n$. We can assume there is no V-structure in this trail. Otherwise, we can find the first collider node $x_m$ in this trail $x\rightleftharpoons x_1 \rightleftharpoons x_2 ...\to x_m \gets x_{m+1} ... \rightleftharpoons x_n$. Then for any $j>m$, this trail is not active between $x$ and $x_j$, and 
\begin{align*}
I(x\ ;x_{m+1})=I(x\ ;x_{m+2})=...=I(x\ ;x_n)=0
\end{align*}
by the d-separation theorem. We can cut this trail from $x_m$ and reconsider this new trail $x\rightleftharpoons x_1 \rightleftharpoons x_2 ... \rightleftharpoons x_m$. 

Since $x\rightleftharpoons x_1 \rightleftharpoons x_2 ... \rightleftharpoons x_n$ has no V-structure, we have $x\perp\!\!\!\perp x_{n}\mid x_{n-1}$. Similar to the analyses at first, we have 
\begin{align*}
I(x\ ;x_{n-1},x_n)&=I(x\ ;x_{n-1})+I(x\ ;x_n\mid x_{n-1})\\
&=I(x\ ;x_n)+I(x\ ;x_{n-1}\mid x_n).
\end{align*}
Thus we also can provide $I(x\ ;x_{n-1})>I(x\ ;x_n)$. Similarly, we also have $x\perp\!\!\!\perp x_{n-2}\mid x_{n-1}$ and $I(x\ ;x_{n-2})>I(x\ ;x_{n-1})$. Continuing this process, we have the conclusion that for any trail $x\rightleftharpoons x_1 \rightleftharpoons x_2 ... \rightleftharpoons x_n$,
\begin{align*}
I(x\ ;x_1)>I(x\ ;x_2)>...>I(x\ ;x_n).
\end{align*}

If the graph is not a trail, but $x\rightleftharpoons x_1 \rightleftharpoons x_2 ... \rightleftharpoons x_n$ is the only trail by which $x$ can reach $x_1$, $x_2$, ..., and $x_n$. Then the conditional independencies are the same as the analyses we discussed above. Therefore the conclusion
\begin{align*}
I(x\ ;x_1)>I(x\ ;x_2)>...>I(x\ ;x_n).
\end{align*}
can be achieved in this situation.

\end{proof}

\begin{proof}{\textbf{of Theorem \ref{matching}}}

Since $e_i$ is only connected with $x_i$, for $x_j\notin MB^e_i$, the trail that connects $x_j$ and $e_i$ must go through $x_i$. This trail must contain at least one of the spouses $e_i$ or children of $x_i$ because $e_i$ has no parent. 

If $x_j\notin MB^e_i$ is not descendent of $x_i$, i.e. there is no path connecting $e_i$ and $x_j$. Then all the trails that connect $x_j$ and $e_j$ must go through one of the parents of $x_i$. For any one of these trails, it must contain a V-structure $e_i\to x_i \gets x_k\rightleftharpoons...$ and $x_i$ is a collider where $x_k$ is the parent of $x_i$ that this trail contains. Therefore all these trails are unactive given the empty set, and $I(x_j\ ; e_i)=0$.

Thus for any node $x_j\notin MB^e_i$, $I(x_j\ ;e_i)\neq 0$ if and only if $x_j$ is one of the descendent of $x_i$.

Let $e_1, e_2, ..., e_n$ be the exogenous variables that correspond to endogenous variables $x_1, x_2, ..., x_n$. We need to prove 
\begin{align}\label{1}
\sum^n_{m=1} I(x_m\ ;e_m)=\max_{j_1, j_2, ..., j_n}\sum^n_{m=1}I(x_m\ ;e_{j_m})\\
with\ \ \ I(x_m\ ;e_{j_m})\neq 0,\ \ \ m=1,2,..., n .\label{2}
\end{align}

Since $e_i$ and $x_i$ are connected, $I(x_i\ ;e_i)>0$. For node $x_j$ is not the descendent of $e_i$, $e_i$ is independent with $x_j$ as we discussed before. Then $I(x_i\ ;e_i)\geq I(x_j\ ;e_i)=0$. Otherwise, there is a path $e_i\to x_i\to ... \to x_j$ which starts from $e_i$ and reaches $x_j$. If this path satisfies the assumption in \textbf{Theorem \ref{trail MI}}, i.e. it is the only trail that connect $x_i$ and $x_j$, we also have $I(x_i\ ;e_i)>I(x_j\ ;e_i)$. 

However, if there are more than one trails that connect $e_i$ and $x_j$, since $e_i$ has no parent, the trail is active given the empty set if and only if this trail is a path from $e_i$ to $x_j$ $e_i\to x_i\to ... x_k\to x_j$. If we add the exogenous variable $e_j$ to $x_j$, $e_j$ with this trail forms a V-structure $x_k\to x_j\gets e_j$. So this trail is unactive between $x_i$ and $e_j$ and $I(x_i\ ;e_j)=0$. Although we can not guarantee $I(x_i\ ;e_i)\geq I(x_j\ ;e_i)$ achieves in this condition, but if we assign $e_i$ to $x_j$ and assign $e_j$ to $x_i$, we break the constraint $I(x_m\ ;e_{j_m})\neq 0$ in equation (\ref{1})

Therefore, for given node $x_i$ and any other node $x_j\in\mathcal X$, we have $I(x_i\ ;e_i)\geq I(x_j\ ;e_i)$ or $I(x_i\ ;e_j)=0$. If we change the exogenous variables correspond to them, we either can not reach the maximum of $\sum^n_{m=1}I(x_m\ ;e_{j_m})$ or break the constraints.

For any three nodes $x_i, x_j, x_k\in\mathcal X$, if we assign $e_k, e_i, e_j$ to $x_i, x_j, x_k$ and $I(x_i\ ;e_k)\neq 0$, $I(x_j\ ;e_i)\neq 0$, $I(x_k\ ;e_j)\neq 0$. By the analyses, $x_i, x_j$, and $x_k$ are the descendants of $x_k, x_i, x_j$. There are three paths $\mathcal P_1$, which starts from $x_k$ to $x_i$, $\mathcal P_2$, which starts from $x_i$ to $x_j$, and $\mathcal P_3$, which starts from $x_j$ to $x_k$. Then $\mathcal P_1, \mathcal P_2$ and $\mathcal P_3$ form a cycle that is contradictory to the assumption that $\mathcal G$ is DAG. Therefore, $(k,i,j)$ is not in the solution of equation (\ref{1}) and equation (\ref{2}). More complex situations can be proved in a similar way. So $(1,2,...,n)$ is the solution of equation (\ref{1}) under the constraints (\ref{2})

On the other hand, if we have the solution of equation (\ref{1}) under the constraints (\ref{2}) as $(i_1, i_2,..., i_n)$. We need to state that $(i_1, i_2, ..., i_n)$ exists and is unique. We know $I(x_m\ ;e_m)\neq 0$ for all $m=1,2,...,n$, so there exist permutations that satisfy all the constraints in equation (\ref{2}). Since the number of arrangements $\mathcal A$ for finite number $(1,2,...,n)$ is also finite. We can pick permutations that follow the constraints, and denote them as $\mathcal A'$. Since $\mathcal A'$ is finite, we can find a unique permutation $\mathbf r$ that maximizes $\sum^n_{m=1}I(x_m\ ;e_{j_m})$. Then this $\mathbf r$ is the solution of equation (\ref{1}) under constraints (\ref{2}).

Now let this solution $\mathbf r=(i_1, i_2,..., i_n)$ and $(i_1, i_2,..., i_n)\neq (1,2,...,n)$. Without loss of generality, we assume $i_m\neq m$ for some $m\leq l$ and $i_m=m$ otherwise. Since $I(x_m\ ;e_{i_m})\geq 0$, $x_m$ is the descendent of $x_{i_m}$, $m=1,2,...,l$. Because $i_m=m$ for $m\geq l$, $1\leq i_m\leq l$ for $m\leq l$. Then the paths from $x_{i_m}$ to $x_m$ $m=1,2,...,l$ forms a cycle, which is contradictory to the property of DAG.

Combing all the analyses, we have that $(e_{i_1},...,e_{i_n})$ are the exogenous variables which correspond to $(x_1, x_2,...,x_n)$ if and only if $(i_1, i_2, ..., i_n)$ is the solution of equation (\ref{1}) under the constraints (\ref{2}).

\end{proof}

\begin{proof}{\textbf{of Theorem \ref{MI equivalence}}}

Since $\mathbf e'$ is another exogenous vector that determined by $\mathbf e$, i.e.
\begin{align*}
\mathbf e'=\mathbf h(\mathbf e),
\end{align*}
for every element of $\mathbf h$ $h_i$, $e'_i$ and $e_i$ are both exogenous variables of $x_i$, then $e_i$ must be one of the inputs of the determined function $h_i$. 

But $(e'_1, ..., e'_n)$ are independent with each other and $(e_1, e_2, ..., e_n)$ are independent with each other. If the determined function $h_i$ has more than one input, we assume $e_i$ and $e_j$ are both in its inputs. And for any transformation of the determined function $\tilde h_i$, we can not remove $e_j$ from its input. By the definition of parents in SCM, we $e_j$ is a parent of $e'_i$ and $e'_j$. Then $e'_i\gets e_j\to e'_j$ forms a fork structure and is active given the empty set. So $I(e'_i\ ;e'_j)>0$ which is contradictory to the independence assumption. Therefore every determined function $h_i$ has an equivalent transformation $\tilde h_i$ such that $e'_i=\tilde h_i(e_i)$ for $i=1,2,...,n$.

Let $\tilde{\mathbf h}=(\tilde h_1, \tilde h_2, ..., \tilde h_n)$. According to the analyses above, we have $\mathbf e'=\tilde{\mathbf h}(\mathbf e)$. Since $\tilde{\mathbf h}$ is the transformation of $\mathbf h$ and $\mathbf h$ is invertible, we can conclude that $\tilde{\mathbf h}$ is invertible and every element $\tilde h_i$ is invertible.

Since $e'_i$ is determined by $e_i$, $x_i\gets e_i \to e'_i$ forms a fork structure, and
\begin{align*}
I(x_i\ ;e_i,e'_i)&=I(x\ ;e_i)+I(x_i\ ;e'_i\mid e_i)\\
&=I(x_i\ ;e'_i)+I(x_i\ ;e_i\mid e'_i).
\end{align*}
In \textbf{Theorem \ref{trail MI}}, we have $I(x\ ;x_{n-1}\mid x_n)>0$ because $x_{n-1}$, $x_{n}$ have other parents $e_{n-1}$, $e_n$ and are not only determined by each other. Different from \textbf{Theorem \ref{trail MI}}, we only have $I(x_i\ ;e_i\mid e'_i)\geq 0$ in this condition.
We have $I(x_i\ ;e_i)\geq I(x_i\ ;e'_i)$ because of $I(x_i\ ;e'_i\mid e_i)=0$ and $I(x_i\ ;e_i\mid e'_i)\geq 0$. Therefore for any function $\tilde h_i$, $I(x_i\ ;e_i)\geq I(x_i\ ;\tilde h_i(e'_i))$. Since $\tilde h_i$ is invertible, we can get 
\begin{align*}
I(x_i\ ;e'_i)\geq I(x_i\ ;(\tilde h_i)^{-1}(e_i))=I(x_i\ ;e_i).
\end{align*}
Therefore $I(x_i\ ;e_i)= I(x_i\ ;e'_i)$. For any other node $x_j$, we can find a set of nodes $Z$ such that $x_j\perp\!\!\!\perp e_i\mid Z$. We can simply write the structure as $x_j\leftrightharpoons Z\gets e_i\to e'_i$, and we also have $I(x_j\ ;e'_i\mid e_i)=0$ and $I(x_j\ ;e_i\mid e'_i)\geq 0$. Therefore we have $I(x_j\ ;e_i)= I(x_j\ ;e'_i)$ in a similar way.

\end{proof}

\begin{proof}{\textbf{of Theorem \ref{exo parent}}}

Firstly, we state that $e'_i$ has no exogenous parent and endogenous parent except for $x_i$. If $e'_i$ has a parent, its parent can not belong to exogenous variables because $\mathbf e'$ are independent with each other and do not have a connection with each other. Let us assume $x_j$ is the parent of $e'_i$. The trail $e'_i\gets x_j\leftrightharpoons e'_j$ is always active no matter the direction of $\leftrightharpoons$, and $e'_i$ and $e'_j$ are not independent given the empty set which is contradictory to the independence assumption. Therefore, all the exogenous variables we generate do not have a parent as $e_j\in\mathbf Y$ or $x_j\in\mathcal X$ $j\neq i$.

Now we need to prove $x_i$ is not the parent of $e'_i$. We discuss the relation between $e'_i$ and $x_i$ in two conditions.

\textbf{condition 1}

Let us assume $x_i$ has at least one endogenous parent $x_j\in Pa_i$. If $e'_i$ is not a parent of $x_i$, then $e'_i\gets x_i$. The trail $e'_j\leftrightharpoons x_j\to x_j\to e'_i$ is active given the empty set regardless of the relation between $x_j$ and $e'_j$. We have $I(e'_j\ ;e'_i)>0$ which is contradictory to the independence assumption.

\textbf{condition 2}

If $x_i$ has no endogenous parent and we assume $x_i$ is a parent of $e'_i$ $e'_i\gets x_i$. Since $e'_i$ does not have other parents, $x_i$ is the only parent of $e'_i$, and $e'_i$ is determined by $x_i$. By \textbf{Algorithm 3}, we have
\begin{align*}
\mathbf e'=\mathbf P\mathbf W^\top\mathbf x,
\end{align*}
where $\mathbf W$ is the matrix trained by FastICA algorithm and $\mathbf P$ is the permutation matrix that rearranges $\mathbf e'$ according to the \textbf{Theorem \ref{matching}}. If $x_i$ is the only parent of $e'_i$, there exits a constant $c$ such that $e'_i=cx_i$. Therefore $x_i=\frac1ce'_i$, and $e'_i$ is also a parent of $x_i$, i.e. $x_i\to e'_i$ and $x_i\gets e'_i$ are equivalent. Without loss of generality, we treat $e'_i$ as a parent of $x_i$ in this special condition.

Combing the analyses at the beginning, we can conclude that $\mathbf e'=(e'_1, ..., e'_n)$ have no parent, and $e'_i$ is a parent of $x_i$.
\end{proof}

\begin{proof}{\textbf{of Theorem \ref{exo MB}}}

Let us assume $x_j$ is a node in exogenous Markov blanket of $e_i$, $x_j\in MB^e_i$ where $i\neq j$. By the definition of Markov blanket, $x_j$ must be a parent of $x_i$ because $e_i$ only has one child and has no parent. But $x_j$ can be a child of $x_i$ in other I-equivalence DAG. 

Let us assume $x_j$ is a parent of $x_i$ in some I-equivalence DAG but does not satisfy \textbf{Definition \ref{SCM-parent}}. Since $x_i$ and $x_j$ must be connected no matter what conditions, then the determined function $f_j$ takes $x_i$ as an input, we can write the determined function as:
\begin{align*}
x_j=f_j(Pa_j, e_j),
\end{align*}
where $x_i\in Pa_j$. For any spouses of $e_i$, $e_i$ is not conditionally independent with spouses given $Z$ as long as $x_i\in Z$ according to \textbf{Theorem \ref{d-sep}}. Let $Z=Pa_j$ which satisfies $x_i\in Z$. We want to prove $e_i\perp\!\!\!\perp e_j\mid Pa_j$, we state it in three conditions:

\textbf{condition 1}

If there is no trail from $x_j$ to $x_k$ where $x_k\in Pa_i$. Then all the trails from $e_i$ to $x_j$ have same structure
\begin{align}\label{3}
e_i\to x_i\to x_{i_1}\leftrightharpoons ...\leftrightharpoons x_{i_l}\leftrightharpoons x_j.
\end{align}
$x_{i_1}$ can not be a parent of $x_i$ by the assumption of \textbf{condition 1}. Therefore $e_i\to x_i\to x_{i_1}$ forms a chain structure, and $e_i\perp\!\!\!\perp e_j\mid x_i$ which indicates $e_i\perp\!\!\!\perp e_j\mid Z$.

\textbf{condition 2}

If there are some trails that connect $x_j$ and any parent of $x_i$ $x_k$ through the child of $x_j$, and they can be all represented as:
\begin{align}\label{4}
e_i\to x_i\gets x_k \leftrightharpoons x_{i_1}... \leftrightharpoons x_{i_l}\gets x_j.
\end{align}
If this trail $x_k$ and $x_j$ is a path, $x_k\gets x_{i_1}\gets ... \gets x_{i_l}\gets x_j$, but $x_i\in Pa_j$. Then $x_i\to x_j$ and $x_i\gets x_k\gets x_{i_1}\gets ... \gets x_{i_l}\gets x_j$ form a cycle, which is contradictory to the property of undirected edge $x_i-x_j$ in CPDAG. Thus $x_k \leftrightharpoons x_{i_1}... \leftrightharpoons x_{i_l}\gets x_j$ can not be a path.

Then $x_k \leftrightharpoons x_{i_1}... \leftrightharpoons x_{i_l}\gets x_j$ must contain at least one V-structure and at least one collider is not belong to $Pa_j$. Otherwise, if all the colliders $x_k \leftrightharpoons x_{i_1}... \leftrightharpoons x_{i_l}\gets x_j$ are parents of $x_j$, let assume the first collider is $x_{i_a}$ where $1\leq a\leq l$, and $x_{i_a}\in Pa_j$. Then $x_{i_a}\gets ... x_{i_l}\gets x_j$ is a path. Since $x_{i_a}\to x_j$, these two paths form a cycle.


Therefore the trail (\ref{4}) either is not active given $Pa_j$ or contains a collider not belongs to $Pa_j$ and is not active given the empty set. 

\textbf{condition 3}

If there are some trails that connect $x_j$ and any parent of $x_i$ $x_k$ through the parent of $x_j$, we have the representation of these trails as
\begin{align}\label{5}
e_i\to x_i\gets x_k \leftrightharpoons x_{i_1}... \leftrightharpoons x_{i_l}\to x_j.
\end{align}
Then we notice $x_{i_l}\in Pa_j$ and trails like (\ref{5}) are not active given $x_{i_l}$ or given $Pa_j$.

Combing these three conditions, we have $e_i\perp\!\!\!\perp e_j\mid Pa_j$. Since $x_j$ is only determined by $e_j$ given $Pa_j=\mathbf c$, and $x_j=f_j(\mathbf c, e_j)$, we obtain $e_i\perp\!\!\!\perp x_j\mid Pa_j$ which is contradictory to the character of spouses of $x_i$. Therefore $x_j$ is not in the exogenous Markov blanket of $e_i$. And if $x_j\in MB^e_i$, $x_j$ must be a parent of $x_i$ in \textbf{Definition \ref{SCM-parent}}.

On the other hand, let $x_j$ is a parent of $x_i$ in \textbf{Definition \ref{SCM-parent}}, i.e.
\begin{align*}
x_i=f_i(Pa_i,e_i),
\end{align*}
where $x_j\in Pa_i$. It is easy to notice the fact that structure $x-y-z$ is a V-structure if and only if for any subset of nodes $y\in Z\subset\mathcal X$, $x$ are $z$ are not conditionally independent given $Z$, $x\not\!\perp\!\!\!\perp y\mid Z$. Notice that trail $e_i\to x_i\gets x_j\gets e_j$ is active if given $x_i$, then for any $Z\in\mathcal X$, as long as $x_i\in Z$, $e_i$ and $x_j$ are not conditionally independent given $Z$ because $e_i\not\!\perp\!\!\!\perp e_j\mid x_i\longrightarrow e_i\not\!\perp\!\!\!\perp e_j\mid Z$ and the determined function $f_j$ takes $e_j$ as input. Therefore $x_j$ is a spouse of $x_i$ and $x_j\in MB^e_i$.

\end{proof}

\section{AUPR results}

\begin{table*}[htpb]
\centering
\fontsize{10}{20}\selectfont
\caption{AUPR on discrete datasets}\label{discrete AUPR}
\setlength{\tabcolsep}{2mm}{
\begin{tabular}{ccccccc}
\bottomrule[1.5pt]
\hline
\multirow{2}{*}{Algorithms} & \multirow{2}{*}{ALARM} & \multirow{2}{*}{BARLEY} & \multirow{2}{*}{CHILD} & \multirow{2}{*}{INSURANCE} & \multirow{2}{*}{MILDEW} & \multirow{2}{*}{HailFinder} \cr
 & \cr
\bottomrule
PC         & 0.602(0.01) & 0.191(0.02) & 0.539(0.04) & 0.350(0.02) & 0.516(0.03) & 0.356(0.01) \cr
FCI        & 0.551(0.02) & 0.366(0.01) & 0.482(0.02) & 0.401(0.03) & 0.327(0.03) & 0.362(0.01) \cr
GIES       & 0.592(0.05) & 0.298(0.03) & 0.509(0.01) & 0.381(0.02) & 0.456(0.01) & 0.386(0.01) \cr
MMHC       & 0.568(0.02) & 0.367(0.02) & 0.449(0.01) & 0.387(0.03) & 0.538(0.03) & 0.357(0.02) \cr
GS         & 0.570(0.01) & 0.332(0.01) & 0.499(0.04) & 0.373(0.03) & \textbf{0.594(0.01)} & 0.311(0.01) \cr
GRaSP      & 0.567(0.01) & 0.281(0.07) & 0.509(0.01) & 0.412(0.07) & 0.417(0.09) & 0.357(0.04) \cr
CDNOD      & 0.551(0.03) & 0.233(0.02) & 0.457(0.02) & 0.336(0.02) & 0.493(0.01) & 0.206(0.02) \cr
EEMBI      & 0.571(0.02) & 0.263(0.02) & 0.573(0.01) & 0.433(0.03) & 0.408(0.06) & 0.334(0.01) \cr
EEMBI-PC   & \textbf{0.620(0.01)} & \textbf{0.456(0.04)} & \textbf{0.610(0.01)} & \textbf{0.547(0.03)} & 0.570(0.04) & \textbf{0.397(0.01)} \cr
 \bottomrule[1.5pt]
\end{tabular}}
\end{table*}

\begin{table*}[htpb]
\centering
\fontsize{10}{18}\selectfont
\caption{AUPR on SACHS and Education datasets}\label{Education AUPR}
\setlength{\tabcolsep}{2mm}{
\begin{tabular}{ccccccc}
\bottomrule[1.5pt]
\hline
\multirow{2}{*}{Algorithms} & \multirow{2}{*}{SACHS} & \multirow{2}{*}{Education 1} & \multirow{2}{*}{Education 2} & \multirow{2}{*}{Education 3} & \multirow{2}{*}{Education 4} & \multirow{2}{*}{Education 5} \cr
 & \cr
\bottomrule
PC           & 0.399 & 0.366(0.01) & 0.325(0.03) & 0.298(0.04) & 0.336(0.01) & 0.390(0.03) \cr
FCI          & 0.407 & 0.367(0.02) & 0.299(0.05) & 0.314(0.01) & 0.300(0.01) & 0.295(0.02) \cr
GIES         & 0.417 & 0.324(0.05) & 0.325(0.05) & 0.310(0.04) & 0.307(0.02) & 0.317(0.06) \cr
MMHC         & 0.319 & 0.286(0.02) & 0.238(0.02) & 0.270(0.02) & 0.255(0.01) & 0.272(0.02) \cr
GS           & 0.407 & 0.346(0.02) & 0.330(0.01) & 0.377(0.01) & 0.321(0.04) & 0.367(0.02) \cr
GRaSP        & 0.405 & 0.313(0.01) & 0.322(0.01) & 0.320(0.03) & 0.337(0.01) & 0.297(0.04) \cr
CDNOD        & 0.371 & 0.354(0.02) & 0.350(0.01) & 0.369(0.01) & 0.377(0.01) & 0.335(0.05) \cr
DirectLiNGAM & 0.349 & 0.403(0.01) & 0.440(0.01) & 0.254(0.13) & 0.253(0.14) & 0.267(0.13) \cr
CAM          & 0.371 & 0.382(0.01) & 0.368(0.01) & 0.363(0.01) & 0.333(0.02) & 0.369(0.05) \cr
NOTEARS      & 0.333 & 0.369(0.02) & 0.374(0.02) & 0.407(0.02) & 0.357(0.01) & 0.390(0.01) \cr
EEMBI        & 0.366 & \textbf{0.418(0.02)} & \textbf{0.443(0.01)} & \textbf{0.448(0.02)} & \textbf{0.464(0.01)} & \textbf{0.462(0.02)} \cr
EEMBI-PC     & \textbf{0.422} & 0.325(0.03) & 0.332(0.02) & 0.363(0.01) & 0.336(0.02) & 0.383(0.02) \cr
 \bottomrule[1.5pt]
\end{tabular}}
\end{table*}

\begin{table*}[htpb]
\centering
\fontsize{10}{18}\selectfont
\caption{AUPR on Dream3 datasets}\label{Dream3 AUPR}
\setlength{\tabcolsep}{2mm}{
\begin{tabular}{cccccc}
\bottomrule[1.5pt]
\hline
\multirow{2}{*}{Algorithms}  & \multirow{2}{*}{Ecoli 1} & \multirow{2}{*}{Ecoli 2} & \multirow{2}{*}{Yeast 1} & \multirow{2}{*}{Yeast 2} & \multirow{2}{*}{Yeast 3} \cr
 & \cr
\bottomrule
PC             & 0.036 & 0.113 & 0.120 & 0.116 & 0.119 \cr
FCI            & 0.057 & 0.134 & 0.155 & 0.134 & 0.151 \cr
GIES           & 0.049 & 0.078 & 0.081 & 0.090 & 0.099 \cr
MMHC           & 0.055 & 0.104 & 0.134 & 0.112 & 0.111 \cr
GS             & 0.040 & 0.082 & 0.123 & 0.138 & 0.117 \cr
GRaSP          & 0.048 & 0.102 & 0.083 & 0.078 & 0.089 \cr
CDNOD          & 0.044 & 0.068 & 0.121 & 0.092 & 0.126 \cr
DirectLiNGAM   & 0.017 & 0.076 & 0.141 & 0.137 & 0.146 \cr
CAM            & 0.049 & 0.123 & 0.112 & 0.149 & \textbf{0.175} \cr
NOTEARS        & 0.029 & 0.086 & 0.052 & \textbf{0.155} & 0.096 \cr
EEMBI          & 0.051 & 0.075 & \textbf{0.163} & 0.085 & 0.140 \cr
EEMBI-PC       & \textbf{0.059} & \textbf{0.134} & 0.076 & 0.092 & 0.093 \cr
 \bottomrule[1.5pt]
\end{tabular}}
\end{table*}

\begin{table*}[htpb]
\centering
\fontsize{10}{22}\selectfont
\caption{AUPR on Dream3 datasets}\label{ablation AUPR}
\setlength{\tabcolsep}{2mm}{
\begin{tabular}{cccccc}
\bottomrule[1.5pt]
\hline
\multirow{2}{*}{Algorithms} & \multirow{2}{*}{Education 1} & \multirow{2}{*}{Education 2} & \multirow{2}{*}{Education 3} & \multirow{2}{*}{Education 4} & \multirow{2}{*}{Education 5}\cr
 & \cr
\bottomrule

Improved IAMB & 0.317(0.02) & 0.324(0.01) & 0.348(0.01) & 0.325(0.01) & 0.334(0.01) \cr\hline
EEMBI (WM)    & 0.338(0.02) & 0.356(0.01) & 0.409(0.01) & 0.372(0.04) & 0.377(0.02) \cr
EEMBI-PC (WM) & 0.252(0.05) & 0.309(0.01) & 0.327(0.01) & 0.300(0.02) & 0.337(0.02) \cr\hline
EEMBI         & 0.418(0.02) & 0.443(0.01) & 0.448(0.02) & 0.464(0.01) & 0.462(0.02) \cr
EEMBI-PC      & 0.325(0.03) & 0.332(0.02) & 0.363(0.01) & 0.336(0.02) & 0.383(0.02) \cr
 \bottomrule[1.5pt]
\end{tabular}}
\end{table*}

We show the AUPR of proposed methods and baselines on all the datasets in Table \ref{discrete AUPR}, \ref{Education AUPR}, and \ref{Dream3 AUPR}, including the mean and standard deviation of every algorithm. The mean values of these algorithms are the same as the values in bar graphs. Similar to the SHD in \textbf{Section 5}, we thicken the highest AUPR of causal discovery algorithms on every dataset.

Table shows the AUPR results on the ablation study corresponding to the SHD results in Table \ref{ablation AUPR}. The outcomes are very similar to the outcomes in Table. The improved IAMB has the smallest AUPR on every Education dataset. The AUPR of EEMBI or EEMBI-PC is always higher than EEMBI (WM) or EEMBI-PC (WM).


\newpage
\vskip 0.2in
\bibliography{references}

\end{document}